\documentclass{article}

\usepackage{arxiv}

\usepackage[utf8]{inputenc} % allow utf-8 input
\usepackage[T1]{fontenc}    % use 8-bit T1 fonts
\usepackage{hyperref}       % hyperlinks
\usepackage{url}            % simple URL typesetting
\usepackage{booktabs}       % professional-quality tables
\usepackage{latexsym}
\usepackage{amsfonts}       % blackboard math symbols
\usepackage{amsmath}
\usepackage{amssymb}
\usepackage{mathrsfs}
\usepackage{nicefrac}       % compact symbols for 1/2, etc.
\usepackage{microtype}      % microtypography
\usepackage{graphicx}
\usepackage{subfigure}
\usepackage{cite}
\usepackage{doi}
\usepackage{indentfirst}
\usepackage{multirow}
\usepackage{float}
\usepackage{appendix}
\usepackage{color}
\usepackage[dvipsnames]{xcolor}
\usepackage{amsthm}
\usepackage{mathrsfs}
\usepackage{soul}
\usepackage{bm}
\soulregister\cite7
\soulregister\eqref7
\soulregister\ref7
\soulregister\inlinecite7

\usepackage{hyperref}
\hypersetup{hidelinks,
	colorlinks=true,
	allcolors=blue,
	pdfstartview=Fit,
	breaklinks=true}

\usepackage{tcolorbox}
\usepackage{colortbl}
\definecolor{mycyan}{cmyk}{.3,0,0,0}
\definecolor{mypink}{rgb}{.99,.91,.95}
\definecolor{mygreen}{RGB}{124,252,0}
\theoremstyle{plain}
\newtheorem{theorem}{Theorem}[section]
\newtheorem{proposition}{Proposition}[section]

\newtheorem{lemma}{Lemma}[section]

\theoremstyle{remark}

\renewcommand{\shorttitle}{Neural network SDE models with applications to financial data forecasting}
\theoremstyle{definition}
\newtheorem{definition}[theorem]{Definition}
\newtheorem{assumption}{Assumption}
\numberwithin{equation}{section} % 设置equation的编号格式

\begin{document}
\title{ Neural network stochastic differential equation models with applications to financial data forecasting}
 %LDE-Net: L\'evy Induced Stochastic Differential Equation Equipped with Neural Network for Time Series Forecasting
\author{\bf\normalsize{
Luxuan Yang$^{1,}$\footnotemark[2],
Ting Gao$^{1,}$\footnotemark[1],
Yubin Lu$^{1,}$\footnotemark[3],
Jinqiao Duan$^{2,}$\footnotemark[4]
and Tao Liu $^{3,}$\footnotemark[5]
}\\[10pt]
\footnotesize{$^1$School of Mathematics and Statistics \& Center for Mathematical Sciences,Huazhong University of Science and Technology,} \\
\footnotesize{ Wuhan 430074, China.} \\[5pt]
\footnotesize{$^2$Department of Applied Mathematics, College of Computing, Illinois Institute of Technology, Chicago, IL 60616, USA} \\[5pt]
\footnotesize{$^3$ Securities Finance Department, China Securities Co., Ltd}
}

\footnotetext[2]{Email: \texttt{luxuan\_yang@hust.edu.cn}}
\footnotetext[1]{Email: \texttt{tgao0716@hust.edu.cn}}
\footnotetext[3]{Email: \texttt{yubin\_lu@hust.edu.cn}}
\footnotetext[4]{Email: \texttt{duan@iit.edu}}
\footnotetext[5]{Email:\texttt{liutao@csc.com.cn}}
\footnotetext[1]{is the corresponding author}

\date{}	% Here you can change the date presented

\maketitle
\setlength{\parindent}{2em}
\begin{abstract}
	%\lipsum[1]
% 	 With the fast development of modern deep learning techniques, the study of dynamic systems and neural networks is increasingly benefiting each other in a lot of different ways. Since uncertainties often arise in real world observations, SDEs (stochastic differential equations) come to play an important role in scientific modeling. To this end, 
In this article, we employ a collection of stochastic differential equations with drift and diffusion coefficients approximated by neural networks to predict the trend of chaotic time series which has big jump properties. Our contributions are, first, we propose a model called L\'evy induced stochastic differential equation network, which explores compounded stochastic differential equations with $\alpha$-stable L\'evy motion to model complex time series data and solve the problem through neural network approximation. Second, we theoretically prove that the numerical solution through our algorithm converges in probability to the solution of corresponding stochastic differential equation, without curse of dimensionality. %Second, we theoretically prove the convergence of our algorithm with respect to hyper-parameters of the neural network, and obtain the error bound without curse of dimensionality. 
Finally, we illustrate our method by applying it to real financial time series data and find the accuracy increases through the use of non-Gaussian L\'evy processes. We also present detailed comparisons in terms of data patterns, various models, different shapes of L\'evy motion and the prediction lengths.

\end{abstract}

% keywords can be removed

\section{Introduction}
 In many scenarios, the past information that time series data contains can help people understand future phenomena. Time series forecasting is hereafter attractive in many research fields and lots of research work are developed to improve the prediction ability of the models. Traditionally, autoregressive models and moving average models are often applied to time series prediction. However, they essentially only capture linear relationship in the data set and can not afford many real cases, such as weather forecasting, financial market prediction or signal processing, which are often complex, nonlinear or even chaotic. Another class of statistical methods is based on Bayes' theorem in probability theory. Bayesian inference, as an applicable tool of Bayesian theory, is leveraged to deal with uncertainty in noisy environment and improve model accuracy by collecting some prior information. For example, Hummer\cite{hummer2005position} applies Bayesian inference to estimate the rate coefficients that can characterize the dynamic behavior of some diffusion models. Bevan et al. \cite{beltran2013self} use the Bayesian inference method to obtain the position-dependent potential energy and diffusivity so as to fit particle dynamics. Sahoo and Patra\cite{sahoo2020river} combine the Bayesian inference method with Markov Chain Monte Carlo (MCMC) algorithm to evaluate water pollution rates. With the help of some prior knowledge, this method can greatly improve the prediction accuracy. However, if the prior distribution is not presumed correctly, the inference will go to the wrong direction. Hence, improving the reliability of prior knowledge is crucial to the model performance. As there always exists various kinds of uncertainties in the forecasting problems, how to accurately capture the intrinsic volatile nature of time series data is still challenging. For example, financial time series data could be unpredictable considering hidden factors, such as market structure, transaction behavior, etc., which are inextricably linked with the carbon price \cite{sun2021new,tian2020point}. In addition, in most often the cases, missing data could occur in many industrial applications, taking the prediction of equipment's temperature time series for example \cite{wu2014multi}. To handle these kinds of complex data with uncertain prior knowledge, some chaotic theory is constructed for prediction where phase space reconstruction is one of the most important ideas. Considering the chaotic characteristic of stock time series, Zhong \cite{Tanwei2007NewAO} improves short-term stock prediction by using Phase Space Reconstruction Theory combined with Recurrent Neural Network. Stergiou and Karakasidis \cite{stergiou2021application} treat Lyapunov time as the length of the neural network input to quantify the safe prediction horizon and use the current prediction as input to predict the next step. Prediction with rolling accumulated training samples is a coin of two sides. On the one hand, with more information put into the training set, it will help to better predict the trend in the next day. On the other hand, the accumulated error will also deviate the long term prediction from the correct track.

% Chaotic theory is suitable for nonlinear prediction where phase space reconstruction is one of the most important ideas. Cao \cite{Cao1997PracticalMF} proposed a practical method to determine the minimum embedding dimension for a scalar time series. This method only needs one parameter, delay time, and a small amount of data when calculating the embedding dimension. Huang \cite{Huang2017NonlinearML} adopted Cao's method for its robust handling of noise and reconstructed observed financial time series into a high-dimensional phase space. Yu \cite{Yu2019StockPP} used mutual information to calculate delay interval and Cao's method to obtain the embedding dimension in order to predict stock prices in reconstructed phase space. Considering the complex non-linear characteristics of stock time series, Zhong \cite{Tanwei2007NewAO} improved short-term stock prediction by using PSRT (Phase Space Reconstruction Theory) combined with RNN (Recurrent Neural Network) in which the input dimension of RNN was determined by the minimal embedding dimension. This inspires us that applying phase space reconstruction to chaotic time series data is an appropriate embedding strategy got input representation.
% \textcolor[rgb]{0,1,0}{do we need to put this paragraph some other place? did you put some of the ref here to Section 4.1??}  

Recently, machine learning models outperform conventional methods in many aspects and have led a growing amount of researches to utilize various deep learning frameworks in the research of time series forecasting. One of the promising models, transformer \cite{Vaswani2017AttentionIA} , is bound to have a prominent prospect. Many variants of  transformer models have emerged in the past two years, such as informer \cite{Zhou2021InformerBE} and autoformer \cite{Wu2021AutoformerDT}. Informer proposes a novel probability based self-attention mechanism, which reduces the network size and improves long-term sequence prediction. In addition, the long sequence output is obtained through the generative decoder, which avoids the cumulative error propagation in the inference stage. Autoformer also boosts the performance of long-term sequence prediction by using the random process theory, which discards the self-attention mechanism of point-wise connections, and replaces with the autocorrelation mechanism of series-wise connections instead. As can be seen from the above, the self-attention mechanism has certain advantages for long-term prediction, but not has too much mathematical theories to explain how and why it works on various prediction time ranges. To this end, Macaron Transformer \cite{Lu2019UnderstandingAI} establishes an analogy between the multi-particle dynamic system in physics and the model structure, which benefits the understanding of the intrinsic mechanisms of Transformer. 
  
Nowadays, the marriage between deep learning and dynamical systems has become a topic of increasingly popularity, as they could benefiting each other in both theoretical and computational aspects. On the computational side, for example, Neural ODE  \cite{Chen2018NeuralOD} is a combination of residual network and ordinary differential equations. The output of the network is calculated with a black box differential equation solver by adjoint methods. The article also builds a time series signal generation model through the variational autoencoder (VAE) framework, and uses the neural network ODE as a part of it. For stochastic cases, neural stochastic differential equation (Neural SDE) \cite{liu2019neural} shows a continuous neural network framework based on stochastic differential equation with Brownian motions. Neural Jump Stochastic Differential Equations \cite{jia2019neural} also provides an approach to learn hybrid system with flow and jump. On the basis of differential equations, an innovative neural network based on stochastic differential equations, SDE-Net \cite{Kong2020SDENetED} captivates us. The core of the SDE-Net is to treat the deep neural network transformation as the state evolution of a stochastic dynamic system, and introduce the Brownian motion term to capture cognitive uncertainty. Moreover, there are also studies using neural networks based on physical regimes. The PINNs(Physics-Informed Neural Networks) \cite{raissi2019physics} model just is a combination of mathematics and deep learning. Lu et al.\cite{lu2021deepxde} summarize PINNs(Physics-Informed Neural Networks) and analyze the feasibility of PINNs to solve partial differential equations. Chen et al. \cite{chen2021solving} propose a new framework based on PINNs(Physics-Informed Neural Networks) with a new loss function to slove inverse stochastic problems. O'Leary et al. \cite{o2021stochastic} use stochastic physics-informed neural networks (SPINN) to learn the hidden physics dynamical systems with multiplicative noise. Kontolati et al. \cite{kontolati2021manifold} propose a multi-scale model based on nonlinear PDEs consistent with a continuum theory, and apply manifold learning and surrogate-based optimization techniques to the framework of machine learning. Moreover, there are even articles that combine finance, mathematics and deep learning. Gonon and Schwab \cite{gonon2021deep} prove that it is reasonable to use DNN in financial modeling of a large basket in a market with a jump.

On the theoretical part, since the neural network is a black box, people are seeking to explain some of its characteristics through numerical analysis of dynamical systems. For instance, Siegel and Xu \cite{siegel2020approximation} get the convergence rate of the two-layer neural network under different conditions with the mathematical tools of Fourier analysis. E et al. \cite{weinan2021barron} define the Barron space and give direct and inverse approximation theorems holding for functions in the Barron space under two-layer neural network. Li et al. \cite{li2019deep} pay attention to the deep residual neural network, which can be regarded as continuous-time dynamical systems and summarized some results on the approximation properties in continuous time. All these ideas have inspired us to prove the convergence of our proposed L\'evy induced stochastic differential equation network (LDE-Net) model.

In many real world applications, L\'evy noise plays an important part as a more general 
random fluctuation including complex non-Gaussian properties. Zhu \cite{zhu2014asymptotic} compares the L\'evy noise with the standard Gaussian noise and concludes that the L\'evy noise is more versatile and has a wider range of applications in such diverse areas as mathematical finance, financial economics, stochastic filtering, stochastic control, and quantum field theory \cite{zhu2014asymptotic}. Sangeetha and Mathiyalagan \cite{sangeetha2020state} indicate that L\'evy motion also works with genetic regulatory networks. Considering the properties of L\'evy motion, we choose it to better capture the uncertainty in financial data, such as abrupt transitions. Thus we introduce L\'evy induced stochastic dynamical system to our neural network framework for modeling complex financial time series data.
  
The rest of this paper is organized as follows. In Section 2, we present our neural network model based on stochastic differential equation induced by l\'evy motion. The theoretical analysis of the model is given in Section 3. In Section 4, experimental results with real financial data are reported, which can help to evaluate the applicability and effectiveness of the model. Finally, we conclude the paper and give future research directions in Section 5.
  
Our main contributions are: First, we design an $\alpha-$stable L\'evy motion induced SDE (stochastic differential equation) to recognize model uncertainty. The experiment results prove that the use of non-Gaussian processes increases the accuracy of model predictions. Second, we prove that the approximated solution through our model (LDE-Net) converges in probability towards the true solution of the corresponding stochastic differential equation. Specially, this convergence bound avoids the curse of dimensionality since it is independent of state dimension. Third, we bring the attention mechanism into the model so that the proposed LDE-Net model can achieve the goal of multi-step prediction.
  
%   To better state our work framework for complex chaotic data prediction, we plot our modling process in the following flow chart. (Figure \ref{structure}).
 
%   \begin{figure}[htbp]
%     \centering
%     \includegraphics[width=0.8\textwidth]{framework.jpeg}
%     \caption{Experiment framework of the stock price series prediction process.}
%     \label{structure}
%   \end{figure}
 
\section{Methods}

\subsection{ResNet to Ordinary Differential Equation}

In the earlier days, linear autoregression models, such as AR, MA, ARIMA, etc., are often adopted to tackle with time series forecasting tasks. There are also some studies that use mathematical methods for time series forecasting. For instance, Wang et al. \cite{wang2020prediction} apply data-driven ordinary differential equation to forecast daily PM2.5 concentration. Li et al. \cite{li2019new} propose a method based on the new bidirectional weakening buffer operator to improve the accuracy of time series forecasting. With the development of deep learning, a lot of neural network models such as RNN, LSTM, GRU, etc. are receiving more and more attention for various prediction problems. Because the structure of neural network lacks interpretability, people try to seek an explanation from mathematical theory. The emergence of residual neural networks (ResNet) gives a novel explanation on the accumulated error of neural networks through numerical analysis of ordinary differential equations.

Consequently, a growing amount of research arises in finding the  relationship between neural network structure and discretized solution of numerical differential equations. For example, as neural network maps the input $x$ to the output $y$ through a series of hidden layers, the hidden representation can be regarded as the state of a dynamic system \cite{Kong2020SDENetED}. Lu et al. \cite{lu2018finite} confirmed that several neural network structures can be interpreted as various kinds of numerical discretizations of ordinary differential equations (ODEs). In particular, ResNet \cite{he2016deep} can be regarded as a form of discretize Euler solution of ODEs. Here is the residual network structure in Figure \ref{Resnet}:

\begin{figure}[H]
    \centering
    \includegraphics[width=0.7\textwidth]{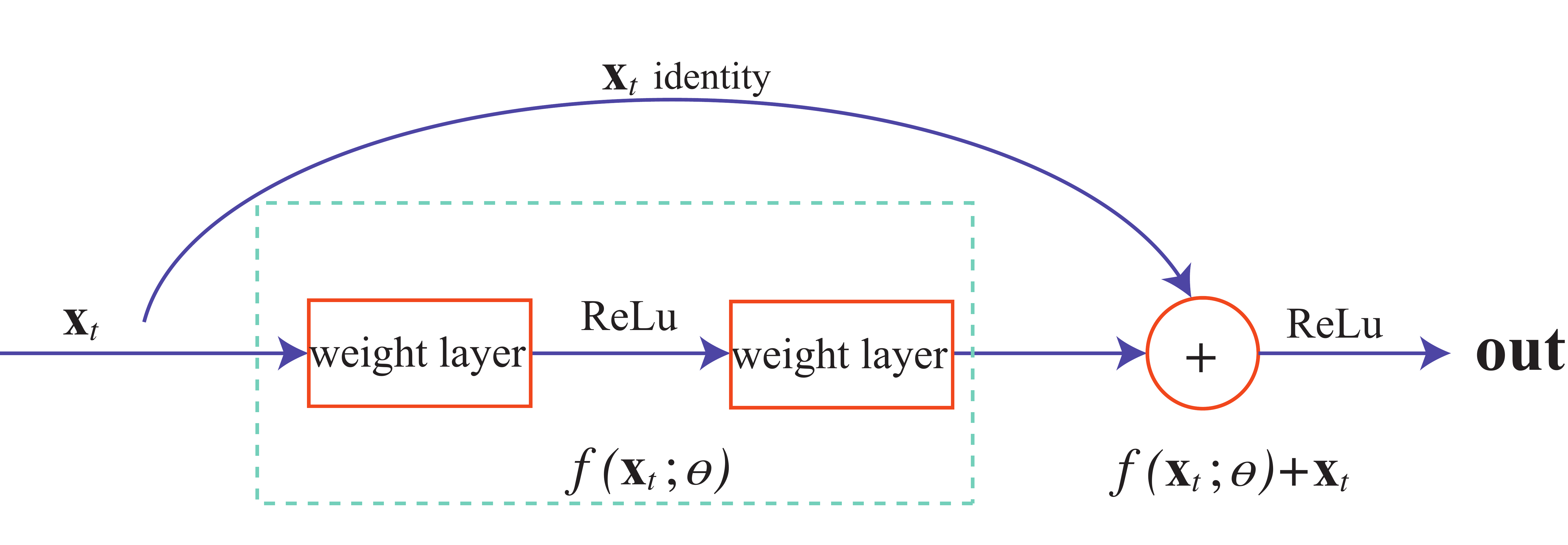}
    \caption{Resnet Structure.( Figure 2 in \cite{he2016deep} )}
    \label{Resnet}
  \end{figure}

The mapping of $f\left(\mathbf{x}_{t}; \theta\right)$ can be realized by feedforward neural networks and the output of identity mapping $\mathbf{x}$ is added to the output of the stacked layers \cite{he2016deep}. Therefore, the hidden state of the network can be described by the following equation:

\begin{equation}
       \mathbf{x}_{t+1}=\mathbf{x}_{t}+f\left(\mathbf{x}_{t}; \theta\right),
       \label{eq:10}
\end{equation}

where $\mathbf{x}_{t}$ is the hidden state of the t-th layer and $t \in\{0 \ldots T\}$. If we add more layers and take smaller steps, the continuous dynamics of hidden units can be shown as follows \cite{Chen2018NeuralOD}:
% where $\mathbf{x}_{t}$ is the hidden state of the t-th layer. If we let $\Delta t=1$, the above formula can be written as:

% \begin{equation}
%       \frac{\mathbf{x}_{t+\Delta t}-\mathbf{x}_{t}}{\Delta t}=f\left(\mathbf{x}_{t};\theta\right).
% \end{equation}

% Let $\Delta t \rightarrow 0$, then:

\begin{equation}
    %  \lim _{\Delta t \rightarrow 0} \frac{\mathbf{x}_{t+\Delta  t}-\mathbf{x}_{t}}{\Delta t}=
    \frac{d \mathbf{x}_{t}}{d t}=f\left(\mathbf{x}_{t}; \theta\right)
    % \Longleftrightarrow d \mathbf{x}_{t}=f\left(\mathbf{x}_{t}; \theta\right) d t.
       \label{eq:12}
\end{equation}

Now we can see that ResNet can indeed match the Euler discretization of ordinary differential equations. This analogy can effectively help us understand the intrinsic behavior of ResNet. 
\subsection{ Introduction to $\alpha$-stable L\'evy Motion} 
% Neural ODEs (Ordinary Differential Equations) explain the state of the hidden layer through a deterministic dynamic system. Although this way is clever, it cannot capture the potential uncertainty caused by the model. In fact, this ability is essential for tasks that value risk. Therefore, It is necessary for us to quantify uncertainties. Kong et al.'s \cite{Kong2020SDENetED} research shows that the uncertainty of model prediction comes from two aspects: (1) aleatoric uncertainty: it comes from the natural randomness in the task itself (such as label noise, etc.) (2) epistemic uncertainty: due to lack of training data, the model's estimate of the input data is inaccurate. For aleatory uncertainty, it is naturally determined by the task itself. You can imagine a training set in which all labels are noisy. A model trained with such a data set will obviously have unreliable prediction results, that is, the uncertainty is very high. For epistemic uncertainty, it is caused by insufficient cognition of the model. When faced with poor training and insufficient training data, the prediction result of the model will have a higher degree of uncertainty. 

In many real-world scientific and engineering applications, stochastic differential equations (SDE) models are widely used to investigate complex phenomena in various noisy systems. For financial market systems, SDE models can be used to quantify the set of optimal pricing barriers \cite{han2016optimal}. In addition, different financial models can be obtained by changing the form of the drift function. For instance, if the drift function $f(x) \equiv x$ in reflected SDE models, it can describe the price dynamics of the reference goods or service \cite{bo2011some}. However, when $f(x) \equiv x^2$, the SDE models represent the risk-neutral term structure of the interest rate model \cite{bo2011some}. In biological systems, SDE models can structurally identify the cell proliferation and death rates \cite{browning2020identifiability}.

Moreover, stochastic differential equations also can model the complex data behaviors with potential uncertainty. SDE-Net \cite{Kong2020SDENetED} uses the stochastic differential equation with Brownian motion to replace the original hidden layer structure of ResNet. However, as many complex real world phenomena exhibit abrupt, intermittent or jumping behaviors, like impulses, using Brownian motion to capture data uncertainty may not be appropriate. Thus, we utilize a class of non-Gaussian noise induced stochastic differential equations with L\'evy motion to improve the universality of predictive models. As we know, L\'evy motion has drawn tremendous attention in various application, such as option pricing model \cite{yuan2022total}, Arctic sea ice system \cite{yang2020tipping}, bearing degradation \cite{song2022long}, climate change \cite{zheng2020maximum}, agricultural product prices \cite{kurumatani2020time} etc.. Now, we use $\alpha$-stable L\'evy motion to capture cognitive uncertainty, which is the core of our proposed LDE-Net model. 

\paragraph{L\'evy Motion}
Let  $X=(X(t), t \geq 0)$  be a stochastic process defined on a probability space  $(\Omega, \mathcal{F}, P)$ . We say that  X  is a L\'evy motion if:
(i)  $X(0)=0$  (a.s);
(ii)  X  has independent and stationary increments;
(iii)  X  is stochastically continuous, i.e. for all  $a>0$  and all  $s \geq 0$ , $\lim _{t \rightarrow s} P(|X(t)-X(s)|>a)=0$.

\paragraph{Symmetric $\alpha$-Stable Random Variables}
$X \sim S_{\alpha}(\sigma, \beta, \gamma)$  is called a symmetric  $\alpha$-stable random variable if  $\beta=0$  and  $\gamma=0$ , that is,  $X \sim S_{\alpha}(\sigma, 0,0)$. When $\sigma=1$, it is called a standard symmetric  $\alpha$-stable random variable, and we denote this by  $X \sim S_{\alpha}(1,0,0)$.

\paragraph{Symmetric $\alpha$-Stable L\'evy Motions }
A symmetric $\alpha$-stable scalar L\'evy motion  $L_{t}^{\alpha}$ , with  $0<\alpha<2$ , is a stochastic process with the following properties:
(i)  $L_{0}^{\alpha}=0$, a.s.
(ii) $L_{t}^{\alpha}$ has independent increments.
(iii) $L_{t}^{\alpha}-L_{s}^{\alpha} \sim S_{\alpha}\left((t-s)^{\frac{1}{\alpha}}, 0,0\right)$.
(iv)  $L_{t}^{\alpha}$  has stochastically continuous sample paths, that is, for every  $s>0$, $L_{t}^{\alpha} \rightarrow   L_{s}^{\alpha}$  in probability, as  $t \rightarrow s$.

\begin{figure}[H]
    \centering
    \includegraphics[width=0.4\textwidth]{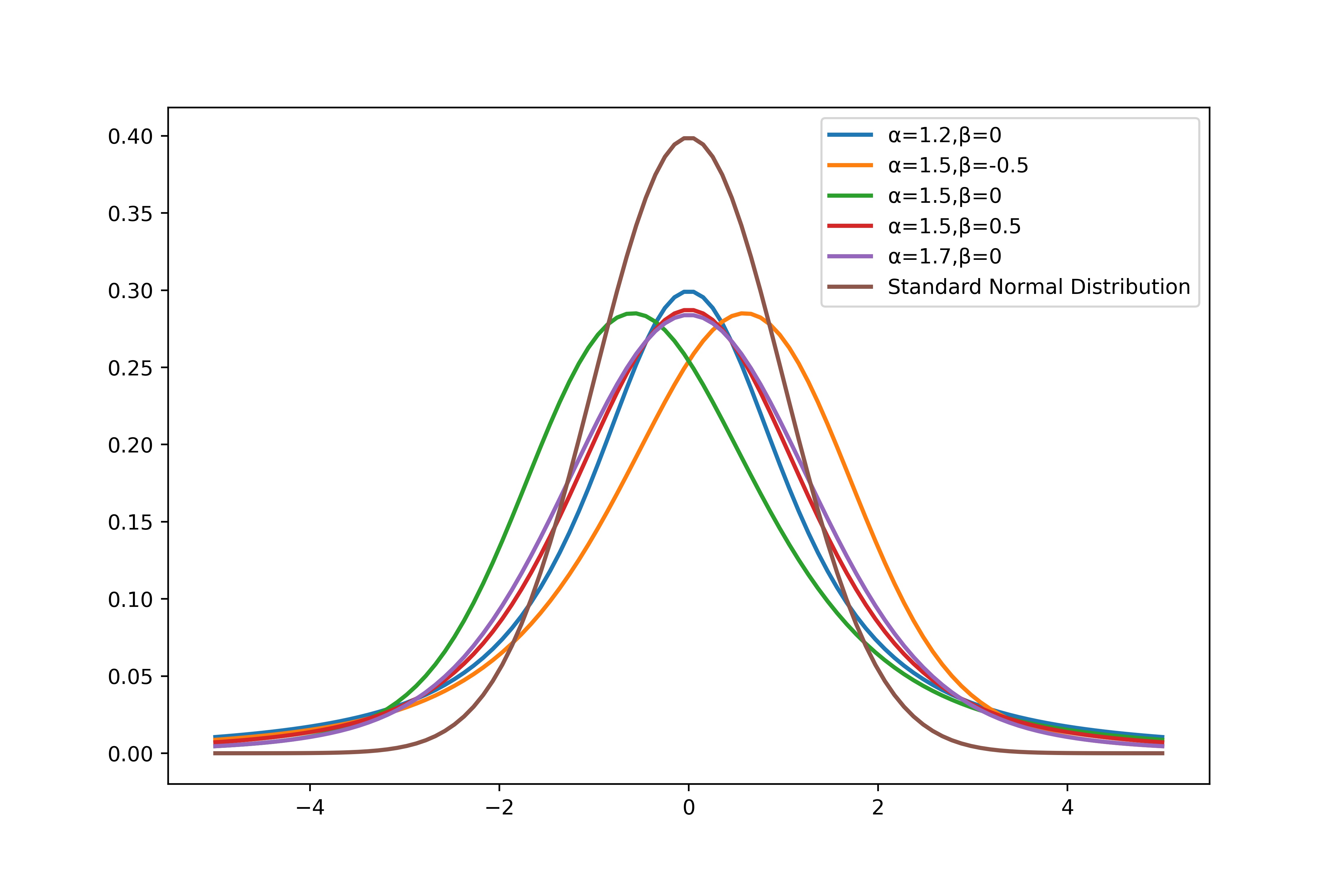}
    \includegraphics[width=0.4\textwidth]{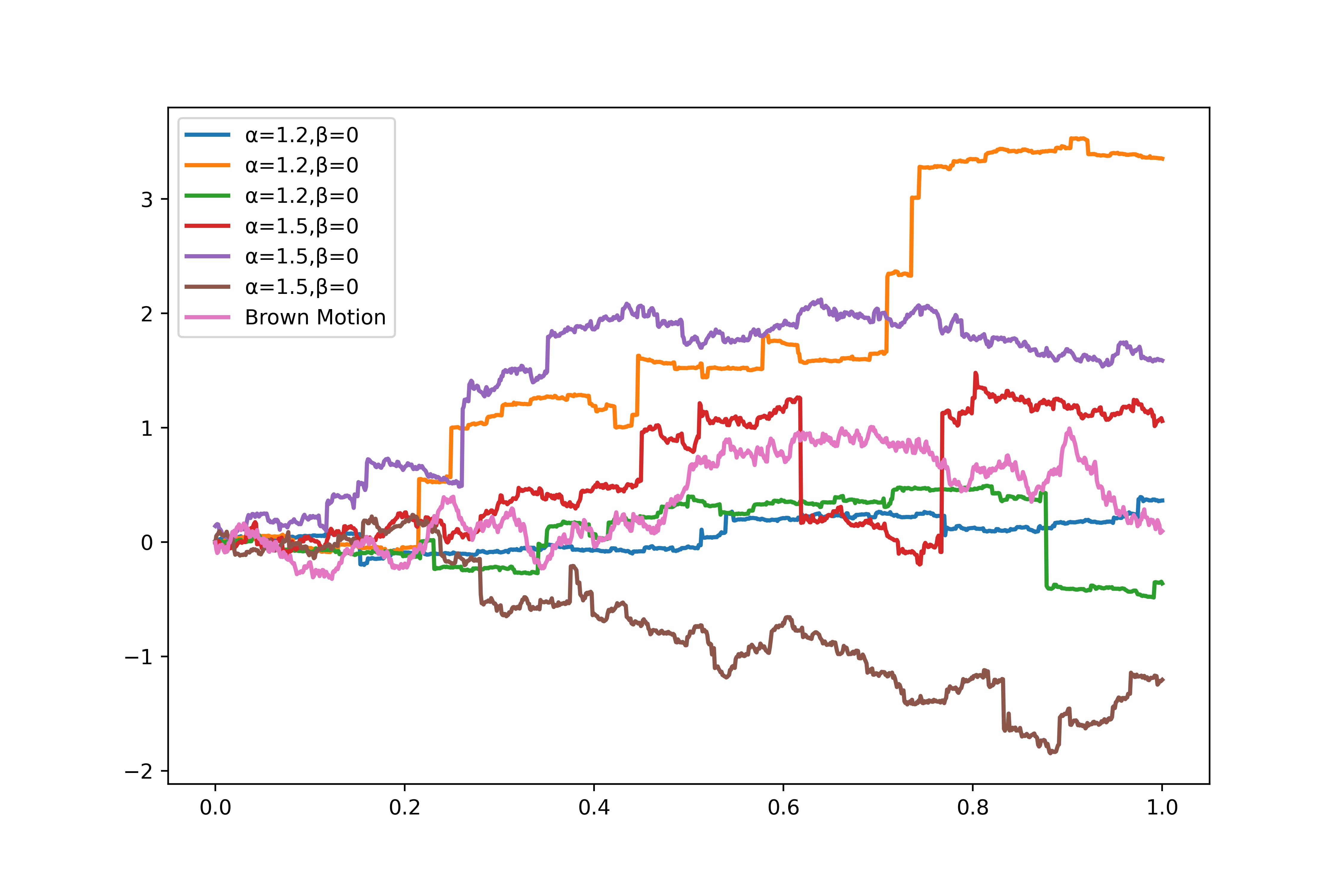}
    \caption{ left: Probability distribution functions of standard normal distribution and L\'evy motion with different $\alpha$ and $\beta$. right: Brownian motion and L\'evy motion with different $\alpha$ and $\beta$}
    \label{levypdf}
\end{figure}

We add  L\'evy motion to the right side of equation(\ref{eq:10}) to make it the Euler discretization form of the stochastic differential equation. The continuous-time dynamics of the system (\ref{eq:12}) are then expressed as:

\begin{equation}
       d x(t)=f(x(t);\theta_f) d t + g(x(t);\theta_g) dL_{t}^{\alpha}, \quad x(0)=x_{0}.
\label{SDE}
\end{equation}
Here, the drift term $f(\cdot;\theta_f): \mathbb{R}^d \to \mathbb{R}^d$ is a locally  bounded Lipschitz continuous function. The diffusion matirx $g(\cdot;\theta_g): \mathbb{R}^d  \to \mathbb{R}^d \times \mathbb{R}^d$  is a diagonal matrix with the same diagonal elements. We then regard $g(\cdot;\theta_g)$ as a scalar function on $\mathbb{R}^d$ which is a locally bounded Lipschitz continuous function.
% are locally bounded and Lipschitz continuous function.
The initial value $x_0 \in \mathbb{R}^d$ is a constant and $x(\cdot):[0,T] \to \mathbb{R}^d$. Here $L^\alpha_t$ is a symmetric $\alpha$-stable L\'evy process in $\mathbb{R}^d$ with $1<\alpha<2$. Moreover,  $\mathbf{\theta}_{f}$ and $\mathbf{\theta}_{g}$ are parameters of the two neural networks.
Next, we will construct two neural networks for the above drift and diffusion coefficients and compound them in a discretized differential equation to accurately predict long-term time series.

\subsection{LDE-Net}
There are various numerical methods for solving differential equations with different accuracy, according to the expansion order of the unknown function. Considering that higher-order numerical methods will increase the computational cost and complicate the input structure, we use a simple Euler-Maruyama scheme with a fixed step size for efficient network training. Based on the Euler-Maruyama scheme, the approximation for equation \eqref{SDE} is
% form of \hl{\bf Euler-Maruyama scheme for} Eq.(\ref{SDE}) is:
\begin{equation}
       \mathbf{x}_{k+1}=\mathbf{x}_{k}+\underbrace{f\left(\mathbf{x}_{k} ; \mathbf{\theta}_{f}\right)}_{\text {drift\, neural\, net }} \Delta t+\underbrace{g\left(\mathbf{x}_{0} ; \mathbf{\theta}_{g}\right)}_{\text {diffusion\, neural\, net }}({\Delta t})^{1/\alpha}L_{k}, \,\,\,\,k=0,1,...,N-1.
\label{EM}
\end{equation}
To further explain equation(\ref{EM}), suppose the real data in our time series is collected on a fixed time interval $T$, that is, the whole time series is obtained at timestamps $[0,\, T,\, 2T, \, 3T, \cdots]$. When predicting at each timestamp, the time interval $T$ is divided into $N$ sub-intervals and we get $\triangle t=\frac TN$. Now we consider the diffusion as N-step random walk with L\'evy noise, here $L_k\sim S_{\alpha}(1,0,0) \, \, (k=0,1,...,N-1)$ are independent standard symmetric $\alpha$-stable random variables, which are denoted as the last term in equation(\ref{EM}). This means we insert $N$ discretized stochastic differential equations between every two consecutive data points.
Here, $\mathbf{x}_{0} = x_0$ is the initial input value of the neural network in $\mathbb{R}^d$ and $x(k\Delta t)$ is approximated by $\mathbf{x}_{k}$.

Moreover, both drift and diffusion coefficients in equation (\ref{EM}) are estimated by two different multi-layer perceptron (MLP) neural networks which are called the drift net and diffusion net respectively.  
% For the diffusion net, the parameters are updated by the same training data. 
In addition, we can also collect noisy data by adding noise to the training set. Then, we train the diffusion net through a binary classification problem using both the original training data and the noisy data.
% and noisy data, which is favorable to improve the robustness of the model. 
The drift net $f$ and diffusion net $g$ are devoted to capture the aleatoric uncertainty and epistemic uncertainty respectively. When applying the above networks into the iteration of equation(\ref{EM}), in the process of moving through the hidden layer, each of our network layers share the same parameters. Simultaneously, the output of the previous layer is transferred to the next layer as the input. And to reduce the computational burden, our diffusion coefficient is only determined by the initial value $x_{0}$. When compounding this collection of discretized Euler iterations, we could predict the time series by the mean and variance of the final iteration $x_N$. The structure of our neural network architecture is shown in Figure \ref{LDE-Net}.

\begin{figure}[htbp]
    \centering
    \includegraphics[width=0.7\textwidth]{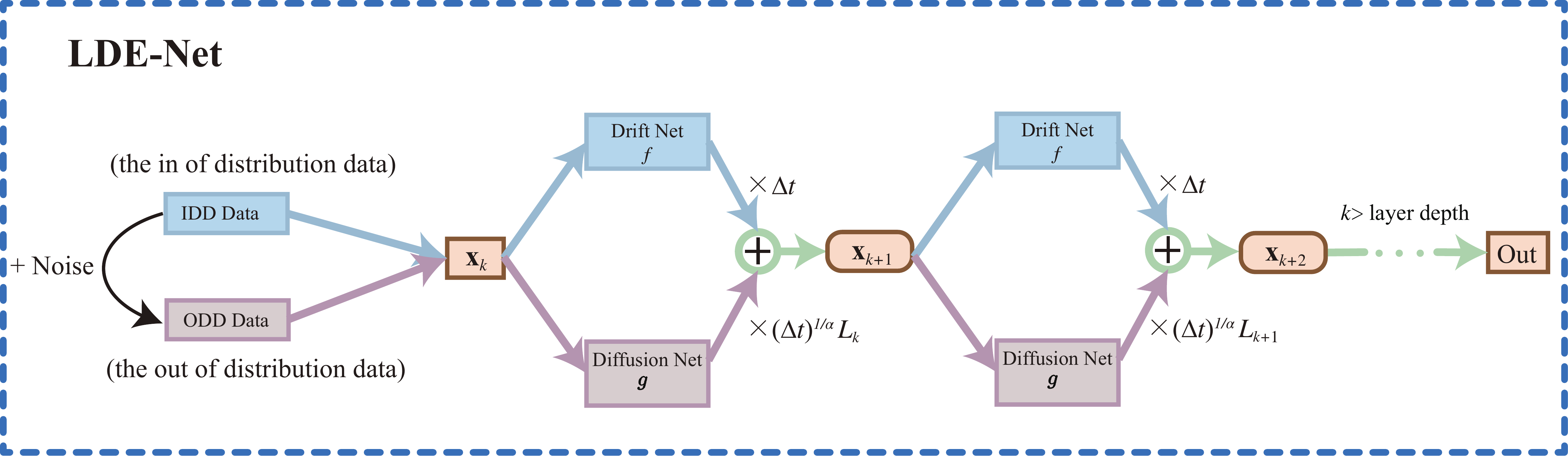}
    \caption{Structure of LDE-Net.}
    \label{LDE-Net}
  \end{figure}

% When compounding this collection of discretized Euler iterations, we could predict the time series by the mean and variance of the final iteration $x_N$.

% Now we first construct two neural networks for both drift and diffusion coefficients in Eq.(\ref{EM}). The drift net $f$ and diffusion net $g$ are devoted to capture the aleatoric uncertainty and epistemic uncertainty respectively. When applying the above networks into the iteration of Eq.(\ref{EM}), in the process of moving through the hidden layer, each of our network layers share the same parameters. And to reduce the computational burden, our diffusion coefficient is only determined by the initial value $x_{0}$. 

Based on the above settings, we have the training loss of our proposed model as follows:

% Since our goal is to capture the uncertainty of the model through stochastic differential equations, the algorithm equipped with the above of the numerical solution of the differential equation as the structure of the neural network.

% Like Kong et al.\cite{Kong2020SDENetED}, the diffusion of the system also should meet the following conditions: (1) For the region in the training distribution, the noise intensity of the L\'evy motion should be small (low diffusion). The system state is dominated by the drift term in this area, and the output noise intensity should be small; (2) For the area outside the training distribution, the noise intensity of the L\'evy motion should be large (high diffusion). 

% Besides, we keep the attack mechanism from SDE-Net \cite{Kong2020SDENetED} for diffusion identification. We also added noise to obtain the out-of-distribution (OOD) data which helps to optimize the parameters of the diffusion term, that is, $\tilde{\mathbf{x}}_{0}=\mathbf{x}_{0} + noise$. 

\begin{equation}
    \min_{\mathbf{\theta}_f} \mathrm{E}_{ P_{\text {train}}}Loss\left(\mathbf{x}_{N}, y_{true} \right) +\min _{\mathbf{\theta}_{g}} \mathrm{E}_{ P_{\text {train}}}g\left(\mathbf{x}_{0};\mathbf{\theta}_{g}\right) + \max _{\mathbf{\theta}_{g}} \mathrm{E}_{ P_{\text {out}}}g\left(\mathbf{\tilde{x}}_{0};\mathbf{\theta}_{g}\right)
    \label{loss}
\end{equation}

Here, the symbol $P_{train}$ represents the distribution for original training set. $P_{out}$ means the distribution of the training data with added noise (see Figure \ref{LDE-Net}). Hence, $\mathbf{\tilde{x}}_{0} = \mathbf{x}_{0} + noise$. Since the training process of the first term is a regression problem, we define $Loss(\mathbf{x}_N, y_{true})$ as the loglikehood loss between the predicted output and true label $y_{true}$ on training data. For the diffusion net, the weights are trained by  the training data with and without added noise. We employ the Binary Cross Entropy as the loss functions in the last two items in equation \eqref{loss}.
 
\paragraph{Remark 1:} Since we utilize non-Gaussian L\'evy motion for modeling our neural network, we will show in the following sessions the better accuracy for time series data prediction both theoretically and experimentally. 

\paragraph{Remark 2:} To better generalize SDE-Net \cite{Kong2020SDENetED}, we also use the attention mechanism to make the proposed model perform multi-step prediction in parallel. 

\paragraph{Remark 3:} This method actually learns $N$ ($= T/\Delta t$) pairs of different drift and diffusion coefficients, corresponding to $N$ ($= T/\Delta t$) different stochastic dynamical systems, where the former equation's output is the input of the next equation. 

% \subsection{Model Architecture}
In summary, our model for the whole training process including data preparation is as follow: first, we check whether the time series data is chaotic or not by Lyapunov exponent (see Appendix \ref{chaos}). Then if this is true, we will calculate intrinsic embedding dimension of input feature with phase space reconstruction (see Appendix \ref{embedding}). Finally, we parallely predict $n$ days' label values for our training data through $n$ different LDE-Nets with attention mechanism. The flow chart of the process is shown in Figure \ref{LDE}. 

\begin{figure}[htbp]
    \centering
    \includegraphics[width=0.7\textwidth]{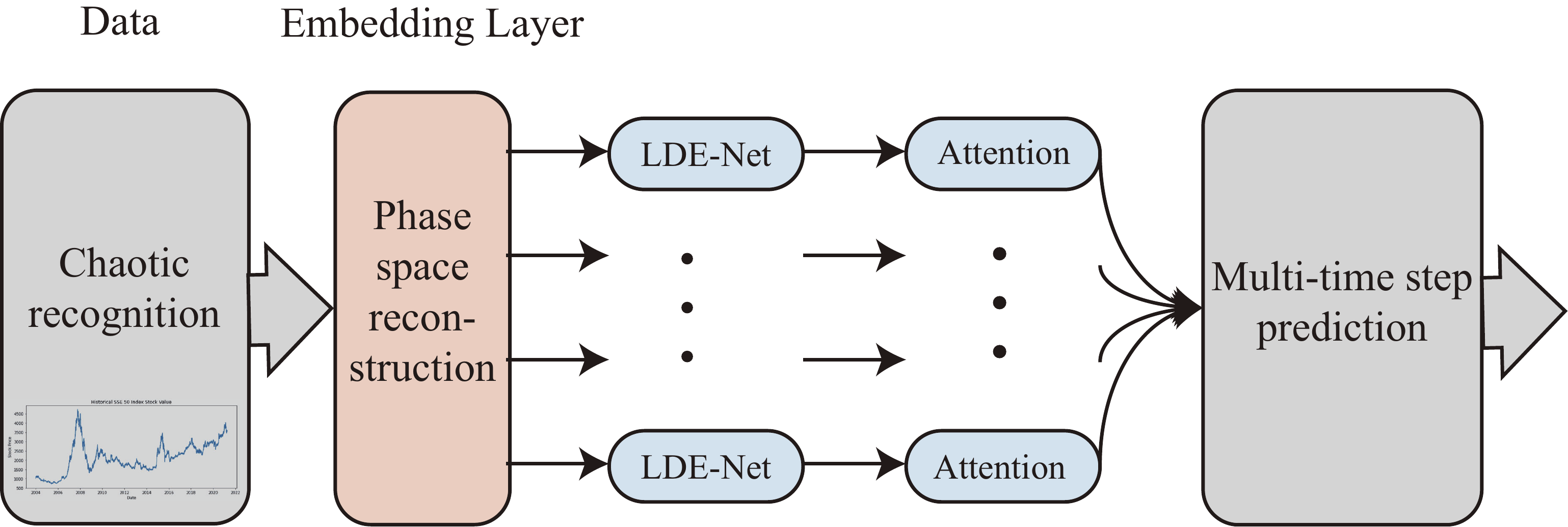}
    \caption{Model construction.}
    \label{LDE}
  \end{figure}

\section{Theoretical Analysis}

In this section, we will prove the convergence of our LDE-Net algorithm. First of all, the existence and uniqueness of the solution $X_{t}$ to a stochastic differential equation with $\alpha$-stable L\'evy noise are proved in Theorem \ref{Uniq}. The proof of the convergence follows three steps: First, under Assumption \ref{Assump1} and Lemma \ref{AL1}, we can get Lemma \ref{AL2} so that we gain the mean uniform error bound of Euler-Maruyama method for stochastic differential equations with Lipschitz continuous drift coefficient by $\alpha$-stable motion.
Then, according to Lemma \ref{2layer}, any function in Barron space can be approximated by a two-layer neural network \cite{weinan2021barron}. Moreover, stochastic gradient descent (SGD) can be proved to be convergent theoretically under some conditions \cite{rotskoff2018trainability}. Based on these two results, we give Assumption \ref{Assump2} and Assumption \ref{Assump3}. Under the assumptions and Lemma \ref{uniconv}, we obtain that our neural networks satisfy the conditions in Proposition \ref{Probconv}, which is crucial to the convergence proof of our LDE-Net model. Then, combining Lemma \ref{AL2} and Proposition \ref{Probconv}, we get Theorem \ref{MR}, which shows that the approximated solution of our algorithm converges in probability to the solution of a corresponding stochastic differential equation.

% Our proof procedure is: Theorem \ref{Uniq} ensures the existence and uniqueness of the solution $X_{t}$ to a stochastic differential equation with $\alpha$-stable L\'evy motion so that we can design an effective network architecture. 

% Since the numerical solution obtained by Euler-Maruyama method contains hyperparameters of the neural network, we hope to prove that the algorithm converges to a solution without hyperparameters.

 First, we introduce Theorem 7.26 in Ch7 in \cite{Duan2015AnIT} and Theorem 9.1 in Ch 4 in \cite{Nobuyuki1989SDEAEP}. Here the function $b(\cdot)$ and $\sigma(\cdot)$ can be viewed as functions as well as two-layer neural networks.

\begin{theorem}
\label{Uniq}
Let $\left\{\boldsymbol{U}, \mathscr{B}_{\boldsymbol{U}}\right\}$ be a measurable space and $\nu$ be a $\sigma$-finite measure on it. Let $U_{0}$ be a set in $\mathscr{B}_{\boldsymbol{U}}$ such that $\nu\left(\boldsymbol{U} \backslash U_{0}\right)<\infty$. If functions $b(\cdot)$ and $\sigma (\cdot)$ satisfy the Lipschitz condition:
\begin{equation*}
    \|b(x_{1})-b(x_{2})\|^{2}+\int_{U_{0}}\|\sigma(x_{1})z-\sigma(x_{2})z\|^{2} \nu(d z) \\
    \leq K|x_{1}-x_{2}|^{2}, \quad x_{1}, x_{2}, z \in R^{d},
\end{equation*}
there exists a unique $\left(\mathscr{F}_{t}\right)$ -adapted right-continuous process  $X(t)$  with left limits which satisfies equation
\begin{equation*}
    X(t)=X(0)+\int_0^tb(X(s)) ds+\int_0^t\int_{\boldsymbol{U}} \sigma (X(s))u\widetilde N( ds, du).
\end{equation*}
Here $\widetilde N$ stands for the compensated Poisson measure. The measure $\nu$ is the L\'evy measure of random poisson measure $N$.
\end{theorem}

Next, we will show the convergence of our model with respect to the number of neurons in hidden layer. In  equation \eqref{EM}, the parameters $\theta_{f}$ and $\theta_{g}$ are only relevant to the width  $m$ of the nerual networks. The parameter $m$ is also called the number of neurons in the hidden layer. Hence, $f(\cdot; \mathbf{\theta}_{f})$ and $g(x_0; \mathbf{\theta}_{g})$ are obtained through neural network training. For simplicity, we use $f_{m}(\cdot)$ to represent $f(\cdot; \mathbf{\theta}_{f})$ and $g_{m}(x_0)$ to represent $g(x_0; \mathbf{\theta}_{g})$. Since $x_0$ is a constant, $g_{m}(x_0)$ is also a constant in $\mathbb{R}$.
We rewrite equation (\ref{EM}) in the following form:
\begin{equation}
      X^{(m)}_{k+1}=X^{(m)}_{k}+f_{m}\left(X^{(m)}_{k} \right) \Delta t+g_{m}\left(X^{(m)}_{0} \right)({\Delta t})^{1/\alpha}L_{k},\quad X^{(m)}_{0}=x_0
      \label{EMNEW}
\end{equation}

where $L_k\sim S_{\alpha}(1,0,0) \, \, (k=0,1,...,N-1)$. To emphasize that functions and variables are from neural network, we add the superscript $(m)$ on them which also denotes the width of the neural network.
% We just change the notation of the variables to distinguish whether  the data of the variables is obtained by neural network training or not. For instance, the superscript $(m)$ represents the variable from neural network training.} 
Let $\Delta t \in (0,1)$ and $N=\frac T{\Delta t}$. We also define the continuous version of equation (\ref{EMNEW}) as follows:
\begin{equation}
  X^{(m)}(t)=X^{(m)}(0)+\int_{0}^{t} f_{m}(\overline{X^{(m)}}(s)) d s+g_{m}\left(X^{(m)}(0) \right)L_{t}^{\alpha}, \quad X^{(m)}(0)=x_0,
  \label{continuousEM}
\end{equation}
where $\overline{X^{(m)}}(t)=X^{(m)}_{k}$ for $t \in[k \Delta t,(k+1) \Delta t)$ and $L_{t}^{\alpha}$ is a symmetric $\alpha$-stable  L\'evy motion. This is similar to the method used in the proof of Liu \cite{liu2019strong}. Then, equation\eqref{SDE} can be correspondingly expressed by:
\begin{equation}
      d x^{(m)}(t)=f_{m}\left(x^{(m)}(t)\right) d t + g_{m}\left(x^{(m)}(0)\right) dL_{t}^{\alpha}.
\label{SDEnew}
\end{equation}
Where $x^{(m)}(0)=X^{(m)}(0)=x_{0}$ and $x_0$ is a constant. We use $\mathcal{B}$ and $\|\cdot\|_{\mathcal{B}}$ to denote the Barron space and Barron norm. $\left | \cdot \right | $ is the Euclidean norm.
\begin{assumption}
\label{Assump1}
Assume that there exists a constant $K>0$ such that
\begin{equation*}
    \vert f_{m}(x_1)-f_{m}(x_2)\vert\leq K\vert x_1-x_2\vert,
\end{equation*}
for any $x_1,x_2 \in \mathbb{R}^d$ and $m \in \mathbb{N}^+$.
\end{assumption}
Note that Assumption \ref{Assump1} actually requires $\{f_m(\cdot)\}_{m \geq 1}$ to be uniformly Lipschitz continuous. Before proving  Lemma \ref{AL2}, we are ready to state the following lemma.

\begin{lemma}\label{AL1}
Suppose that Assumption \ref{Assump1} holds, for every $t \in [0,T]$, the mean absolute difference between the continuous version $X^{(m)}(t)$ and discrete version $\overline{X^{(m)}}(t)$ of Euler-Maruyama method  is
\begin{equation*}
  \mathbb{E}\vert X^{(m)}(t)-\overline{X^{(m)}}(t)\vert\leq C_6{(\Delta t)}^\frac1\alpha,
\end{equation*}
where constant $C_6=4K_1(1+C_5)+2\vert g_m(x_0)\vert C_3$ with $C_5=C_4+e^{6K_1T}$, $C_4=3(\mathbb{E}\vert x_0\vert+2K_1T+C_3T^\frac1\alpha\vert g_m(x_0)\vert)$ and $K_1=\max(K,f_m(x_0))$. Note that $K$ is the Lipschitz constant for $f_{m}(\cdot; \mathbf{\theta}_{f_m})$ and $C_{3}$ is a constant dependent on index $\alpha$ (see the exact expression of $C_{3}$ in \rm{\cite{Samorodnitsky1995StableNR})}. The proof of Lemma \ref{AL1} can be found in Lemma 3.3 in \cite{liu2019strong}.
\end{lemma}

\begin{lemma}\label{AL2}
Suppose that Assumption \ref{Assump1} holds, the mean uniform error bound of the Euler-Maruyama method for equation \eqref{SDEnew} satisfies
\begin{equation*}
   \mathbb{E}\underset{0\leq t\leq T}{\sup}\vert x^{(m)}(t)-X^{(m)}(t)\vert\leq KC_6{(\Delta t)}^\frac1\alpha T\exp(KT),
\end{equation*}
where $x^{(m)}(t)$ is the solution of equation \eqref{SDEnew} and $X^{(m)}(t)$ is the continuous version of the corresponding Euler-Maruyama method in equation \eqref{continuousEM}.
\end{lemma}

\begin{proof}[\bf Proof of Lemma \ref{AL2}]
For every $0< \tau \le t \leq T$,
\begin{equation*} 
\begin{aligned}
     \vert x^{(m)}(\tau)-X^{(m)}(\tau)\vert&=\vert\int_0^\tau \lbrack f_m(x^{(m)}(s))-f_m(\overline {X^{(m)}}(s))\rbrack ds\vert\\
     &\leq \int_0^t \vert f_m(x^{(m)}(s))-f_m(X^{(m)}(s))\vert ds+\int_0^t\vert f_m(X^{(m)}(s))-f_m(\overline{X^{(m)}}(s))\vert ds\\
     &\leq K\int_0^t\vert x^{(m)}(s)-X^{(m)}(s)\vert ds+K\int_0^t\vert X^{(m)}(s)-\overline{X^{(m)}}(s)\vert ds \\
     & \leq K\int_0^t \underset{0\leq r\leq s}{\sup} \vert x^{(m)}(r)-X^{(m)}(r)\vert ds+K\int_0^T\vert X^{(m)}(s)-\overline{X^{(m)}}(s)\vert ds.
\end{aligned}
\end{equation*}
Then, we get
\begin{equation*} 
  \underset{0\leq \tau\leq t}{\sup}\vert x^{(m)}(\tau)-X^{(m)}(\tau)\vert\leq K\int_0^t \underset{0\leq r\leq s}{\sup} \vert x^{(m)}(r)-X^{(m)}(r)\vert ds+K\int_0^T\vert X^{(m)}(s)-\overline{X^{(m)}}(s)\vert ds.
\end{equation*}
Let $v(t)=\underset{0\leq r\leq t}{\sup}\vert x^{(m)}(r)-X^{(m)}(r)\vert$. Then, taking expectations on both sides and by Fubini Theorem, we have 
\begin{equation*} 
  \mathbb{E}v(t)\leq K\int_0^t \mathbb{E}v(s)ds+K\int_0^T \mathbb{E} \vert X^{(m)}(s)-\overline{X^{(m)}}(s)\vert ds.
\end{equation*}
By Lemma \ref{AL1}, we can get
\begin{equation*} 
  \mathbb{E}v(t)\leq K\int_0^t \mathbb{E}v(s)ds+KTC_6{(\Delta t)}^\frac1\alpha.
\end{equation*}
Using Gronwall’s inequality and taking $t=T$, we get
\begin{equation*} 
  \mathbb{E}v(T)\leq KC_6{(\Delta t)}^\frac1\alpha T\exp(KT),
\end{equation*}
which means
\begin{equation*}
  \mathbb{E}\underset{0\leq r\leq T}{\sup}\vert x^{(m)}(r)-X^{(m)}(r)\vert\leq KC_6{(\Delta t)}^\frac1\alpha T\exp(KT).
\end{equation*}
\end{proof}

\begin{lemma}
\label{2layer}
For every $f \in \mathcal{B}$ and constant $m>0$, there exists a two-layer neural network $f_{m}(\cdot ; \Theta)$, $f_{m}(\mathbf{x} ; \Theta)= \sum_{k=1}^{m} \tilde{a}_k \sigma\left(\bm{b_{k}}^{T} \mathbf{x}+c_{k}\right)$, where $\tilde{a}_k = \frac{a_{k}}{m}$. ($\Theta$ denotes the parameters $\{(a_{k}, \bm{b_{k}}, c_{k}), k = 1, \dots, m\}$, such that
\begin{equation*}
    \left\|f(\cdot)-f_{m}(\cdot ; \Theta)\right\|_{L^2}^{2} \leq \frac{3\|f\|_{\mathcal{B}}^{2}}{m}.
\end{equation*}

Furthermore, we have
\begin{equation*}
    \|\Theta\|_{\mathcal{P}}:= \sum_{j=1}^{m}\left|\tilde{a}_{j}\right|\left(\left\|\bm{b_{j}}\right\|_{1}+\left|c_{j}\right|\right) \leq 2\|f\|_{\mathcal{B}}.
\end{equation*}
\end{lemma}
Here, $m$ is the number of neurons in the hidden layers and  $\{(\tilde{a}_k, \bm{b_k}, c_k), k = 1, \dots, m\}$ are the parameters of our neural network $f_m$. The proof of Lemma \ref{2layer} can be found in Theorem 1 in \cite{weinan2021barron}.

Based on Lemma \ref{2layer}, two-layer neural networks in our model approximate function $f$ in the following sense:
\begin{equation}
\label{ineq}
   \vert\vert f_m(\cdot)-f(\cdot) \vert\vert_{L^2} \lesssim \frac{\|f \|_{\mathcal{B}}^{2}}{m^{\frac12}} \leq C_1m^{-\frac12},
\end{equation}
where $C_{1}$ is a constant. Considering the input of the diffusion neural network $g_m$ is always the initial value $x_0$ and the stochastic gradient descent (SGD) algorithm ensures its convergence \cite{rotskoff2018trainability}, we assume that $g_{m}(x_0)\to  g(x_0)$ when updating the parameters of the neural network.
\begin{assumption}
    \label{Assump2}
    Assume that there is an underlying SDE:
    \begin{equation}\label{mainSDE}
        Y(t)=Y(0)+\int_0^t f(Y(s)) ds+ g(x_0)L_{t}^{\alpha}, \quad Y(0)=x_0,
    \end{equation}
    where $L_{t}^{\alpha}$ is a symmetric $\alpha$-stable  L\'evy motion with $1<\alpha<2$. Continuous function $f$ and constant $g(x_0)$ satisfy the Lipschitz condition in Theorem \ref{Uniq}.
    % is bounded Lipschitz continuous with Lipschitz constant $L$. $g(x_0)$ is a constant relative to initial point $x_0$.
\end{assumption} 

\begin{assumption}
\label{Assump3}
Assume for every compact domain $D \subset \mathbb{R}^d$, the neural networks trained functions $f_m(\cdot)$ are uniformly Lipschitz continuous and converge to $f(\cdot)$ in $L^2(D)$. Assume constants $g_m(x_0)$ converge to $g(x_0))$.
% Assume neural network trained functions $f_m(\cdot)$ are the approximation function stated in Lemma \ref{2layer} for function $f(\cdot)$ in \eqref{mainSDE}. Assume constants $g_m(x_0)$ converge to $g(x_0))$.
\end{assumption} 
Based on Lemma \ref{2layer}, we can gain the convergence of $f_m(\cdot)$ in $L^2(D)$. Meanwhile, since the Barron norm of $f$ is bounded, it ensures the uniformly Lipschitz continuity of $f_m(\cdot)$ on every compact sets $D$. In practice, we choose a sufficiently large $D$ that real data varies within $D$ to guarantee this assumption.

Here, equation (\ref{mainSDE}) can be written as:

\begin{equation*} 
    Y(t)=Y(0)+\int_0^t f(Y(s))ds +\int_0^t\int_{|u|>0}g(x_0)u \widetilde{N}(ds,du), \quad Y(0)=x_0,
\end{equation*}
 where $\widetilde{N}(ds,du)$ is the Poisson random measure with compensator $ds \otimes \nu_\alpha(du)$. The $\alpha$-stable L\'evy measure $\nu_\alpha$ is 
\begin{equation*} 
    \nu_\alpha(du)=\frac{1}{|u|^{1+\alpha}}du, \quad 1<\alpha<2.
\end{equation*}

\begin{definition}
     Continuous functions $f_n(\cdot)$ \emph{locally uniformly converge} to function $f(\cdot)$ if for every compact subset $D$ in $\mathbb{R}^d$,
    \begin{equation*}
        \sup_{z \in D}|f_n(z)-f(z)| \to 0, \quad n \to \infty.
    \end{equation*}
\end{definition}

\begin{lemma}\label{uniconv} 
    For a sequence of locally uniformly Lipschitz continuous functions $\eta_n(\cdot) \in L^2_{\textrm{Loc}}(\mathbb{R}^d \cap C(\mathbb{R}^d))$ and a locally Lipschitz continuous function $\eta(\cdot) \in L^2(\mathbb{R}^d)$, suppose for every compact domain $D \subset \mathbb{R}^d$, 
    \begin{equation*} 
        \lim_{n \to \infty} \|\eta_n(\cdot)-\eta(\cdot)\|_{L^2(D)} = 0.
    \end{equation*}
     Then, functions $\eta_n(\cdot)$ also uniformly converge to function $\eta(\cdot)$ on every compact domain in $\mathbb{R}^d$.
    % which means for every bounded closed domain $D$ in $\mathbb{R}^d$, functions $\eta_n$ are uniformly Lipschitz continuous and bounded on $D$. 
\end{lemma}
\begin{proof}
    Given a compact domain $D \subset \mathbb{R}^d$, we first prove that every subsequence of $\eta_n(\cdot)$ has a further subsequence that converges uniformly on the domain $D$.

    For every subsequence of $\{\eta_n(\cdot)\}_{n \geq 1}$, still denoted by $\{\eta_n(\cdot)\}_{n \geq 1}$, we know that $\eta_n(\cdot)$ converges to $\eta(\cdot)$ in $L^2(D)$. Then there is a subsequence of $\eta_n(\cdot)$, denoted by $\eta_{n_k}(\cdot)$, that converges to $\eta(\cdot)$ for $z \in D\backslash D_0$. $D_0 \subset D$ is a set of Lebesgue measure zero. 
    
    For every $\epsilon>0$, let $\delta=\epsilon/3L$, where $L$ is the largest Lipschitz constant of functions $\eta_{n_k}(\cdot)$ and $\eta(\cdot)$. 
    The open balls of radius $\delta$, denoted by $\{B(z,\delta),z\in D\backslash D_0\}$, form a open cover of domain $D$. And there is a finite sub-cover of domain $D$, denoted by $\{B(z_n,\delta)\}_{1 \leq n\leq M}$. Let $N$ be the largest number that for all $j>N$ and $\{z_n\}_{1 \leq n\leq M}$, $|\eta_{n_j}(z_n)-\eta(z_n)|<\epsilon/3$. Then, for every $z \in  D$, there exists a point $z_{l} \in \{z_n\}_{1 \leq n\leq M} $ such that $\|z_l-z\|<\delta$. According to the uniform Lipschitz continuity of $\eta_{z_k}(\cdot)$ on domain $D$, 
    \begin{equation*} 
        |\eta_{n_j}(z)-\eta(z)| \leq |\eta_{n_j}(z)-\eta_{n_j}(z_l)|+|\eta_{n_j}(z_l)-\eta(z_l)|+|\eta(z_l)-\eta(z)|<\epsilon.
    \end{equation*}
    It means that the subsequence $\eta_{n_k}(\cdot)$ converges uniformly to $\eta(\cdot)$ on $ D$.
   
    Next, we will show that $\eta_n(\cdot)$ converges uniformly to $\eta(\cdot)$ on $D$ with proof by contradiction. 

    Suppose there is $\tilde{\delta}>0$ and for all $k>0$, there exists $\tilde{z}_k\in  D$ and $n_k>k$ such that 
    \begin{equation*} 
        |\eta_{n_k}(\tilde{z}_k)-\eta(\tilde{z}_k)|>\tilde{\delta}.
    \end{equation*}
     For above $k$, $\{\eta_{n_k}(\cdot)\}_{k \geq 1}$ form a subsequence of $\{\eta_n(\cdot)\}_{n \geq 1}$. The previous argument shows the subsequence of $\{\eta_{n_k}(\cdot)\}_{k \geq 1}$ converges uniformly to $\eta(\cdot)$ on $ D$, which contradicts with the assumption. As a result, we have proved that $\eta_n(\cdot)$ converges uniformly to $\eta(\cdot)$ on $D$, which means functions $\eta_n(\cdot)$ uniformly converge to function $\eta(\cdot)$ on every compact domain in $\mathbb{R}^d$.
\end{proof}

Let $\mathbb{D}_{\mathbb{R}^d}[0,\infty)$ denote the Skorokhod space of functions from $[0,\infty)$ to $\mathbb{R}^d$. The following proposition is the cornerstone for proving the convergence of our model.

\begin{proposition}\label{Probconv}
    For the $d$-dimensional approximating SDE 
    \begin{equation}\label{approxSDE}
        y_n(t)=y_n(0)+\int_0^t b_n(y_{n}(s))ds + \sigma_n L^\alpha_t, \quad y_n(0)=x_0,
    \end{equation}
   where $L_t^\alpha$ is the $d$-dimensional symmetric $\alpha$-stable L\'evy process. 
    Assume the $d$-dimensional functions $b_n(\cdot)$ 
    satisfy the Lipschitz condition in Theorem \ref{Uniq} for all $n \geq 1$ and locally uniformly converge to $f(\cdot)$ of equation \eqref{mainSDE}.
    Suppose the diagonal entries of $d \times d$ diagonal matrix $\{\sigma_n\}_{n\geq 1}$ converge to $g(x_0)$ of equation \eqref{mainSDE} . Then the solution $y_n(\cdot)$ converge to $Y(\cdot)$ of equation \eqref{mainSDE} in probability with respect the Skorokhod distance in $\mathbb{D}_{\mathbb{R}^d}[0,\infty)$.
\end{proposition}
\begin{proof}
    The proof relies on Theorem 5.4 and Corollary 5.6 of \cite{Kurtz1990} and we will verify conditions C2.2(i) and C5.4 therein. Let
    \begin{equation*}
        F_n(z,s)=\left(\begin{array}[]{cc}
            b_n(z) & \sigma_n
        \end{array}\right), \quad Z_n(s)=\left(\begin{array}[]{c}
            s \\
            L_s^\alpha
        \end{array}\right) \quad \text{and } U_n(s)=U(s)=x_0,
    \end{equation*}
    % : \mathbb{R}^d \times [0,\infty) \to \mathbb{R}^{d\times(d+1)}
     where $Z_n$ are a $(d+1)$-dimensional stochastic processes and $F_n$ are $d\times(d+1)$ matrix-valued functions. Then equation \eqref{approxSDE} can be written in the following form,
    \begin{equation} 
        y_n(t)=U(t)+\int_0^t F_n(y_n(s),s)dZ_n(s).
    \end{equation}
     We first verify condition C2.2(i) of \cite{Kurtz1990}. According to L\'evy-Ito decomposition, we have
    \begin{equation*} 
        L_t^\alpha  =\int_0^t\int_{|u|\leq 1} u \widetilde{N}(ds,du)+\int_0^t\int_{|u|> 1} u N(ds,du):= M_t+A_t,
    \end{equation*}
     and $Z_n$ has the following decomposition,
    \begin{equation*}
        Z_n(t)  =\left(\begin{array}[]{c}
            0 \\
            M_t
        \end{array}\right)+\left(\begin{array}[]{c}
            t \\
            A_t
        \end{array}\right) :=\widetilde{M}_t+\widetilde{A}_t.
    \end{equation*}
We know process $M_t$ is a martingale and $A_t$ is of finite total variation. We denote the total variation of $A_t$ by $T_t(A)$. The quadratic variation of $M_t$, denoted by $[M]_t$, and its expectation are
    \begin{equation*} 
        [M]_t=\int_0^t \int_{|u|\leq 1} |u|^2 N(ds,du) \quad \text{and } \quad \mathbb{E}([M]_t)=\int_0^t \int_{|u|\leq 1} |u|^2 ds\nu_\alpha(du)<\infty.
    \end{equation*}
    It means that for every stopping time $\tau$ and time $t$, we have 
    \begin{equation*} 
        \mathbb{E}([\widetilde{M}]_{t\wedge \tau}+T_{t\wedge \tau} (\widetilde{A}))\leq \mathbb{E}([M]_{t}+T_{t}(A)+t) < \infty.
    \end{equation*}
    Then, condition C2.2(i) is satisfied.
    Next, we will verify condition C5.4.

    Let $T_1[0,\infty)$ denote the collection of non-decreasing mappings $\lambda$ of $[0,\infty)$ onto $[0,\infty)$ satisfying $\lambda(t+h)-\lambda(t) \leq h$ with respect to all $t,h \geq 0$. By Example 5.3 of \cite{Kurtz1990}, we know continuous functions $F_n(\cdot)$ and $F(\cdot)$ define mappings $G_n$ and $G$ through 
    \begin{equation*} 
       G_n(\xi,\lambda,s)=F_n(\xi(\lambda(s))), \quad  G(\xi,\lambda,s)  =F(\xi(\lambda(s))), \quad \lambda \in T_1[0,\infty),
    \end{equation*}
    where function $\xi\in \mathbb{D}_{\mathbb{R}^{(d+1)}}[0,\infty)$. With a given compact set $\mathscr{H} \subset D_{\mathbb{R}^{d+1}}[0, \infty) \times T_1[0, \infty)$, we know the projections
    \begin{equation*}
        \mathscr{H}_1=\{\xi: (\xi, \lambda) \in \mathscr{H}\}, \quad \mathscr{H}_2=\{\lambda: (\xi, \lambda) \in \mathscr{H}\} 
    \end{equation*}
    are compact sets in $D_{\mathbb{R}^{d+1}}[0, \infty)$ and $T_1[0, \infty)$ respectively. Then, for every $t>0$, there exists $M>0$ such that, 
    \begin{equation*} 
        \sup_{(\xi,\lambda)\in\mathscr{H}} \sup_{s \leq t}(|\xi(\lambda(s))|+|\lambda(s)|) \leq M. 
    \end{equation*}
     As a result, we have
    \begin{equation*}
    \begin{aligned}
        \sup_{(\xi,\lambda)\in\mathscr{H}} \sup_{s \leq t}|G_n(\xi,\lambda,s)-G(\xi,\lambda,s)|& =\sup_{(\xi,\lambda)\in\mathscr{H}} \sup_{s \leq t}(|b_n(\xi(\lambda(s)))-b(\xi(\lambda(s)))|+|\sigma_n-\sigma|) \\
        & \leq \sup_{|u|\leq M}(|b_n(u)-b(u)|+|g_n-g|) \to 0, \quad \textrm{as } n \to \infty.
    \end{aligned}
    \end{equation*}
     Condition C5.4(i) of \cite{Kurtz1990} is then verified. Condition C5.4(ii) is a direct result of the continuity of functions $F_n$ and condition C5.4 is verified. By Theorem \ref{Uniq}, we know that solution $Y$ to equation \eqref{mainSDE} exists uniquely. Then, Corollary 5.6 in \cite{Kurtz1990} can be applied to ensure that solutions $y_n$ of equations \eqref{approxSDE} converge in probability to solution $Y$ of equation \eqref{mainSDE} with respect to the Skorokhod distance in $\mathbb{D}_{\mathbb{R}^d}[0,\infty)$.
\end{proof}

 The following theorem is the main result of this section. It proves the convergence of our model in $d$-dimension.
\begin{theorem}\label{MR}
     Suppose the two-layer neural network trained $d$-dimensional functions $f_m(\cdot)$ and $d \times d$ diagonal matrices $g_m(x_0)$ satisfy the Assumption \ref{Assump1} and Assumption \ref{Assump3}. Assume further the Assumption \ref{Assump2}. Then, the continuous dynamics of our LDE-Net model $X^{(m)}$ converge in probability to the solution $Y$ of equation \eqref{mainSDE}.
\end{theorem}
\begin{proof}  
    According to Lemma \ref{AL2}, for  $X^{(m)}$  and $x^{(m)}$ defined in \eqref{continuousEM} and \eqref{SDEnew} respectively, we have  
    \begin{equation*} 
        \mathbb{E}(d_T(X^{(m)}, x^{(m)})) \leq \mathbb{E}\big(\sup_{0\leq t \leq T}\vert X^{(m)}-x^{(m)}\vert\big) \to 0, \quad \textrm{as } \Delta t \to 0,
    \end{equation*}
     where distance $d_T$ is the distance in the Skorokhod space $\mathbb{D}_{\mathbb{R}^d}[0,T]$. It means that $X^{(m)}$ converge to $x^{(m)}$ in probability. 
    % the distance $d(X^{(m)}, x^{(m)})$ converges in probability to $0$. 
    % It implies that $X^{(m)}$ converges weakly to  $x^{(m)}$ in the Skorokhod topology of $\mathbb{D}(0,T)$.

    % Note that the activation function for $f_m(\cdot)$ is ReLU function. According to Lemma \ref{2layer}, neural networks learned functions $f_m(\cdot)$ are locally uniformly bounded Lipschitz continuous. 
    With Assumption \ref{Assump3} and Lemma \ref{uniconv}, we know continuous functions $f_m(\cdot)$ converge to $f(\cdot)$ on every compact set  $D \subset \mathbb{R}^d$. Take $f_m(\cdot)$ and $g_m(x_0)$ as the drift term $b_n(\cdot)$ and the diffusion term $\sigma_n$ respectively in equation \eqref{approxSDE}. 
    Then, functions $f_m(\cdot)$ satisfy the condition proposed for the drift term $b_n(\cdot)$ in Proposition \ref{Probconv}. As constants $g_m(x_0)$ are bounded and also converge to $g(x_0)$, by Proposition \ref{Probconv}, we know that solution $x^{(m)}$ of equation \eqref{SDEnew} converges in probability to solution $Y$ of equation \eqref{mainSDE} with respect to the Skorokhod distance in $\mathbb{D}_{\mathbb{R}^d}[0,T]$. Then we know for every $\delta>0$,
    \begin{equation*} 
        \mathbb{P}(d_T(Y,X^{(m)})>\delta)\leq \mathbb{P}(d_T(X^{(m)},x^{(m)})>\frac{\delta}{2})+(d_T(Y,x^{(m)})>\frac{\delta}{2}) \to 0, \quad \textrm{as } m \to \infty \textrm{ and } \Delta t \to 0.
    \end{equation*}
    Then, we have proved that $X^{(m)}$ converge in probability to the solution $Y$ of equation \eqref{mainSDE} with respect the Skorokhod distance in $\mathbb{D}_{\mathbb{R}^d}[0,T)$.
    %  {\color{red} \bf $L^2$-convergence of $f_m$+locally uniformly boundedness and uniformly continuity imply conditions \eqref{convecond1} and \eqref{convecond2}? }

    % According to Lemma \ref{weakconv}, it is clear that $x^{(m)}$ converges weakly to $Y$ in the Skorokhod topology of $\mathbb{D}(0,T)$.
    % Applying Theorem 3.1 in Section 3, Chapter 1 of  \cite{Billingsley1999}, we obtain the desired result.
\end{proof}

\section{Experiments}
In this section, we first preprocess the real data through Phase Space Reconstruction with Cao's method in Sec. 4.1 (see Appendix \ref{appendixA}). After that, we compare our model performance under various scenarios from Sec 4.2 to Sec. 4.5. To be specific, in order to bring the real data into our LDE-Net model, we find the optimal delay time and embedding dimension for the input vector in Sec. 4.1, and then show the prediction accuracy for three different stock trending patterns in Sec. 4.2. Furthermore, we choose two appropriate patterns of them to compare our model with several baselines on the same test data with the same evaluation measurements in the following sections. In addition, we also numerically analyze the effect of different $\alpha$ values in  L\'evy motion and forecasting days on prediction accuracy. The code of all experiments has been open sourced on our GitHub repository: 
\href{github}{https://github.com/senyuanya/LDENet}.%\cite{Github}.

\subsection{Dataset}
Time series in many real scenarios usually appear to be in irregular and hard to understand, so it is important to uncover the  governing laws hidden under a non-linear mapping. Chaotic time series can be exactly a type of sequence that restores its original system through mapping. The chaotic theory of nonlinear complex dynamics science holds that: all realistic nonlinear systems are in dynamic evolution, and may present periodic ordered states and non-periodic disordered states (that is, chaotic states) during their evolution \cite{Zhanhui2009ChaosCO}. \cite{tian2020point} just employs phase space reconstruction to reconstruct the carbon price data for feature selection. \cite{stergiou2021application} takes Lyapunov time as the safe horizon and compares the long-term prediction effects of four different neural network models inside and outside the safe horizon. It was found that even outside the safe horizon, high accuracy can also be obtained. Considering the effectiveness of this methodology, we also use phase space reconstruction to effectively extract data feature representation by viewing the dynamics of monitoring data.

Due to the importance of the financial market, which is related to national economic development and social stability, we select stock data as time series. Certainly, our approach is still applicable to time series data such as weather and electricity. In this paper, we consider different types of stock data and divide them separately into a training set and a testing set according to the ratio of $4:1$. Then we use the Wolf method \cite{wolf1985determining} to calculate the maximum Lyapunov exponent of data and determine whether the financial time series is chaotic. As for the embedding time delay $\tau$, the autocorrelation function or mutual information \cite{karakasidis2009detection} is a universal and effective method for the stationary series. However, according to Zbilut \cite{zbilut2005use} and Yao \cite{yao2017recurrence}, choosing a lag value $\tau = 1$ is usually appropriate for non-stationary data like financial time series. So we also choose the delay time $\tau = 1$. Next, we analyze the effect of the embedding dimension on the maximum Lyapunov exponent. Then, we calculate the best embedding dimension of the stock data and reconstruct the state space with the Cao's method \cite{Cao1997PracticalMF}. Finally, we use the reconstructed data as the input of the LDE-Net model and perform a multi-step prediction, that is, using M (the best embedding dimension) days of data prior to the prediction date and forecast the trend for the next four days. In addition, we compare our proposed model with some other different models and find out the impact of different $\alpha$ on the prediction.

\subsubsection*{Three Types of stock Patterns}
According to the different fluctuations of data sets, we select three stocks corresponding to SSE Energy Index, SSE50 Index, and SSE Consumer Index. For SSE Energy Index, it fluctuates greatly on the training set, but behaves relatively smoothly on the testing set. For SSE50 Index, it shows large fluctuations on the training set and relatively small on the testing set. For SSE Consumer Index, it performs relatively smoothly on the training set but fluctuates greatly on the testing set. Furthermore, we select SSE Energy index with a total sample of 3086, SSE50 Index with a total sample of 4194 and SSE Consumer Index with a total sample of 2497 separately from 2009-2021, 2004-2021, and 2011-2021. (See Figure \ref{fig:TT})

\begin{figure}[H]
    \centering    
    \subfigure {
    \label{fig:SSEETT}     
    \includegraphics[width=5cm,height=3.5cm]{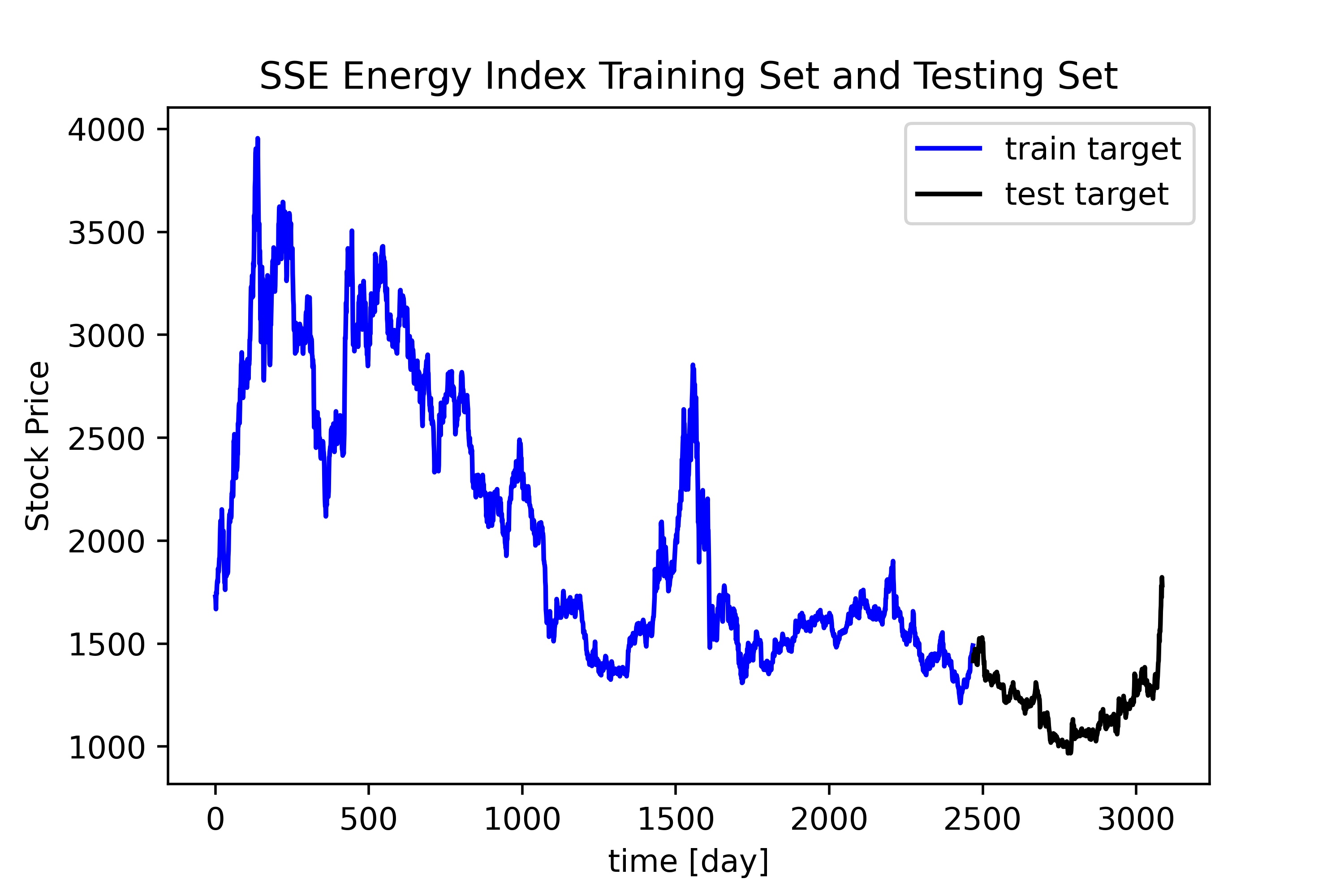}  
    }     
    \subfigure{ 
    \label{fig:SSE50TT}     
    \includegraphics[width=5cm,height=3.5cm]{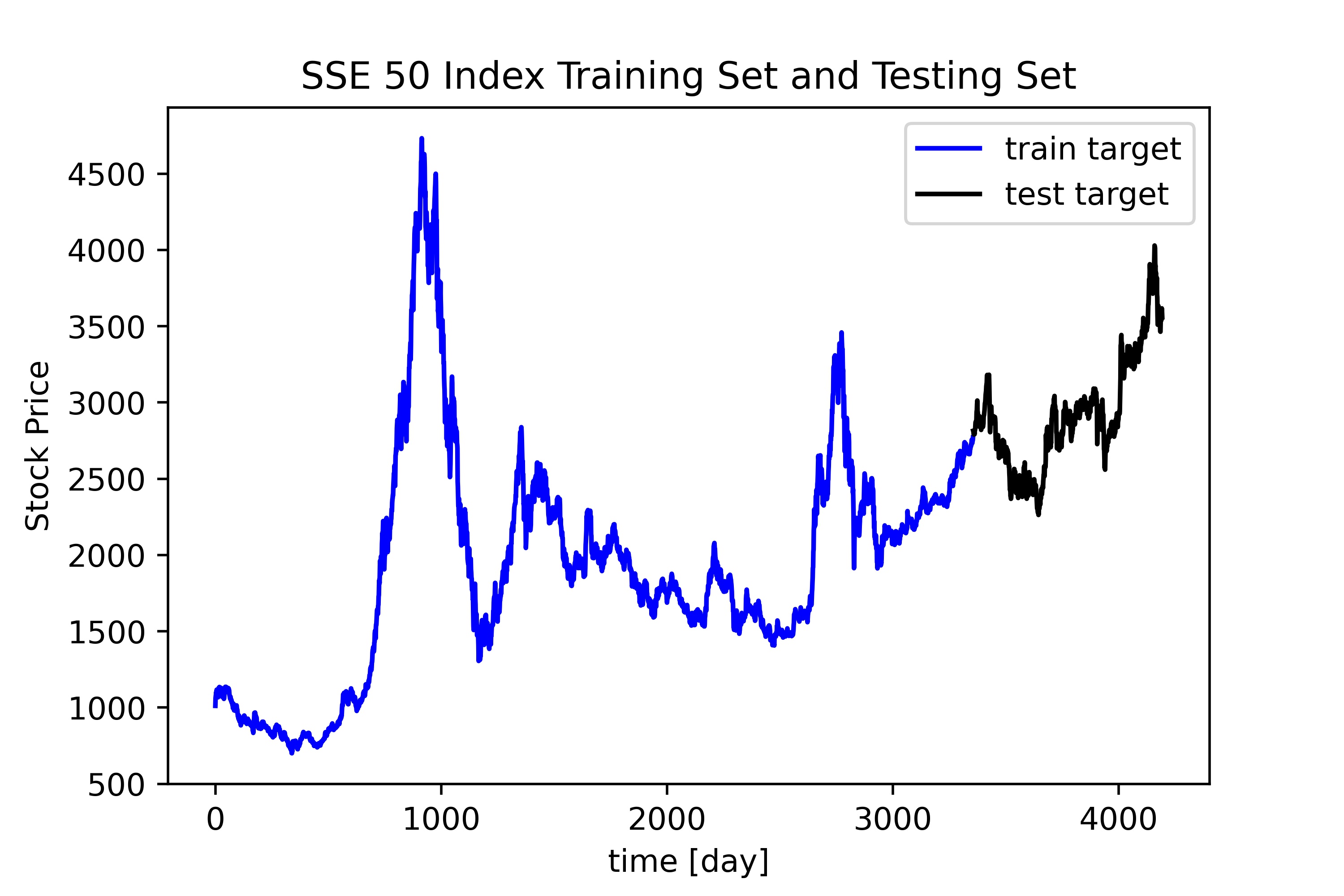}     
    }    
    \subfigure { 
    \label{fig:SSECTT}     
    \includegraphics[width=5cm,height=3.5cm]{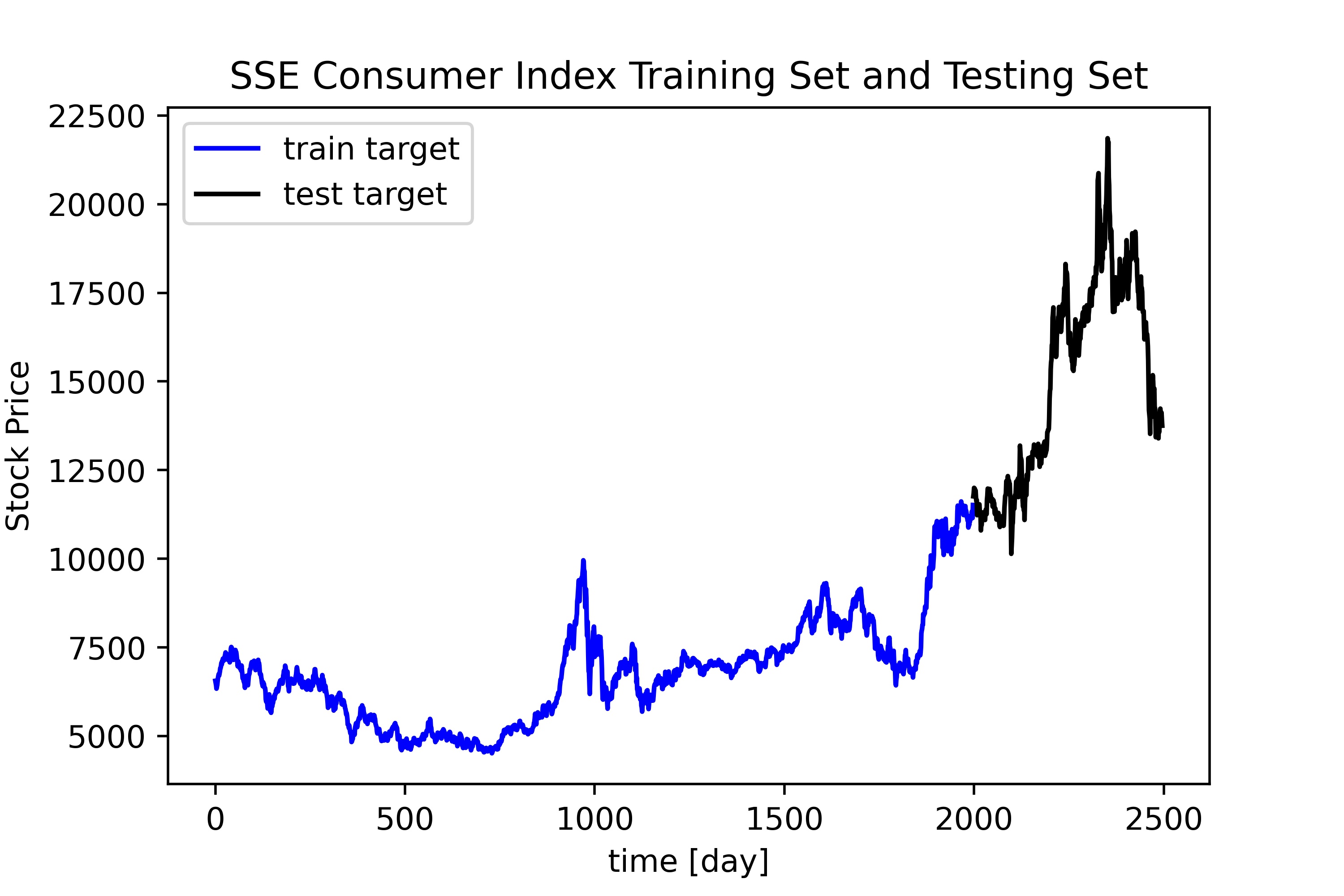}     
    }   
    \caption{ left: Data set of SSE Energy index. middle: Data set of SSE 50 index. right: Data set of SSE Consumer Index.}     
    \label{fig:TT}     
\end{figure}

\subsubsection*{ Lyapunov Exponent and Intrinsic Embedding Dimension}
After completing stock selection, we need to determine the input structure. This requires the use of the phase space reconstruction technique we mentioned earlier. Considering that stock data cannot involve future information in practical applications, we choose training data to identify chaotic characteristics and determine embedding dimensions. Therefore, the sample sizes used for phase space reconstruction for the above three stocks are 2469, 3356, and 1998 respectively. According to the calculation method of the maximum Lyapunov exponent (see Appendix \ref{chaos}), we conclude that the three stocks are all chaotic time series. Moreover, we also show the safe prediction horizons \cite{stergiou2021application} according to Lyapunov time in Table \ref{tab:The Maximum Lyapunov Exponent}.

\begin{table}[H]
	\caption{ Phase Space Reconstruction Indicators}
	\centering
	\begin{tabular}{ccc}
		\toprule
		Data Set     & Maximum Lyapunov Exponent     & Lyapunov Time ( Safe Prediction Horizon ) \\
		\midrule
		SSE Energy Index & 0.0231  & 43   \\
		SSE50 Index & 0.0242  & 41     \\
		SSE Consumer Index     & 0.0324 & 30     \\
		\bottomrule
	\end{tabular}
	\label{tab:The Maximum Lyapunov Exponent}
\end{table}

Subsequently, we explore the correlation between the maximum Lyapunov exponent and embedding dimension under the condition of fixed delay time. From Figure \ref{fig:EDL}, all the maximum Lyapunov exponents have descending trend and tend to be stable with the increment of embedding dimension. This indicates that the maximum Lyapunov exponent will converge over the embedding dimension. Frank explains that this phenomenon is caused by the fact that high-dimensional embedding can simultaneously reduce and diffuse the neighbor distribution, resulting in a reduced Lyapunov exponent estimation \cite{frank1990recovering}. To determine the optimal embedding dimension when the maximum Lyapunov exponent plateaus, we choose Cao's method (see Appendix \ref{embedding}). After calculation, the optimal embedding dimensions for the three stocks are: $m_{1}=20$, $m_{2}=23$, $m_{3}=15$ respectively. Based on these results, we obtain the input data structure. Moreover, considering that Lyapunov time can be regarded as the safe prediction horizon \cite{stergiou2021application}, we take the safe horizon as the benchmark input size and compare the performance of models inside and outside this benchmark. From the tables in Appendix \ref{Safe prediction horizon} (Table \ref{tab: SSEenergyhorizon}, Table \ref{tab: SSE50horizon}), we can know that our model has good prediction performance even outside the safe horizon, which is similar to the result in \cite{stergiou2021application}.

\begin{figure}[H]
    \centering    
    \includegraphics[width=16cm,height=4.5cm]{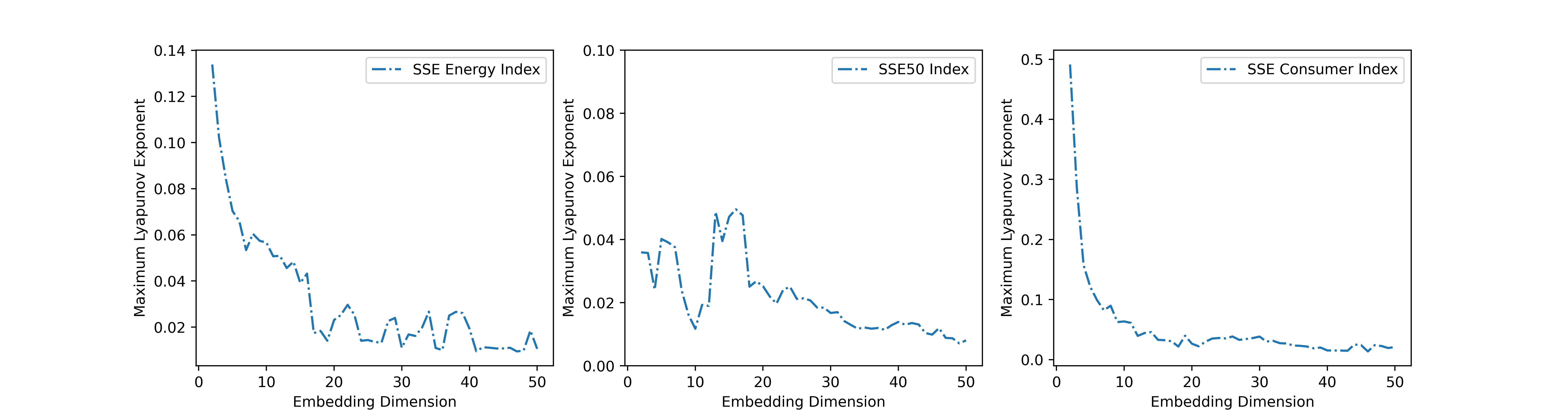}     
    \caption{ Relationship between embedding dimensions and maximum Lyapunov exponent. ( left: SSE Energy Index, middle: SSE50 Index, right: SSE Consumer Index) }
    \label{fig:EDL}   
\end{figure}

\begin{figure}[H]
    \centering    
    \subfigure {
    \label{fig:SSENYED}     
    \includegraphics[width=5cm,height=3.5cm]{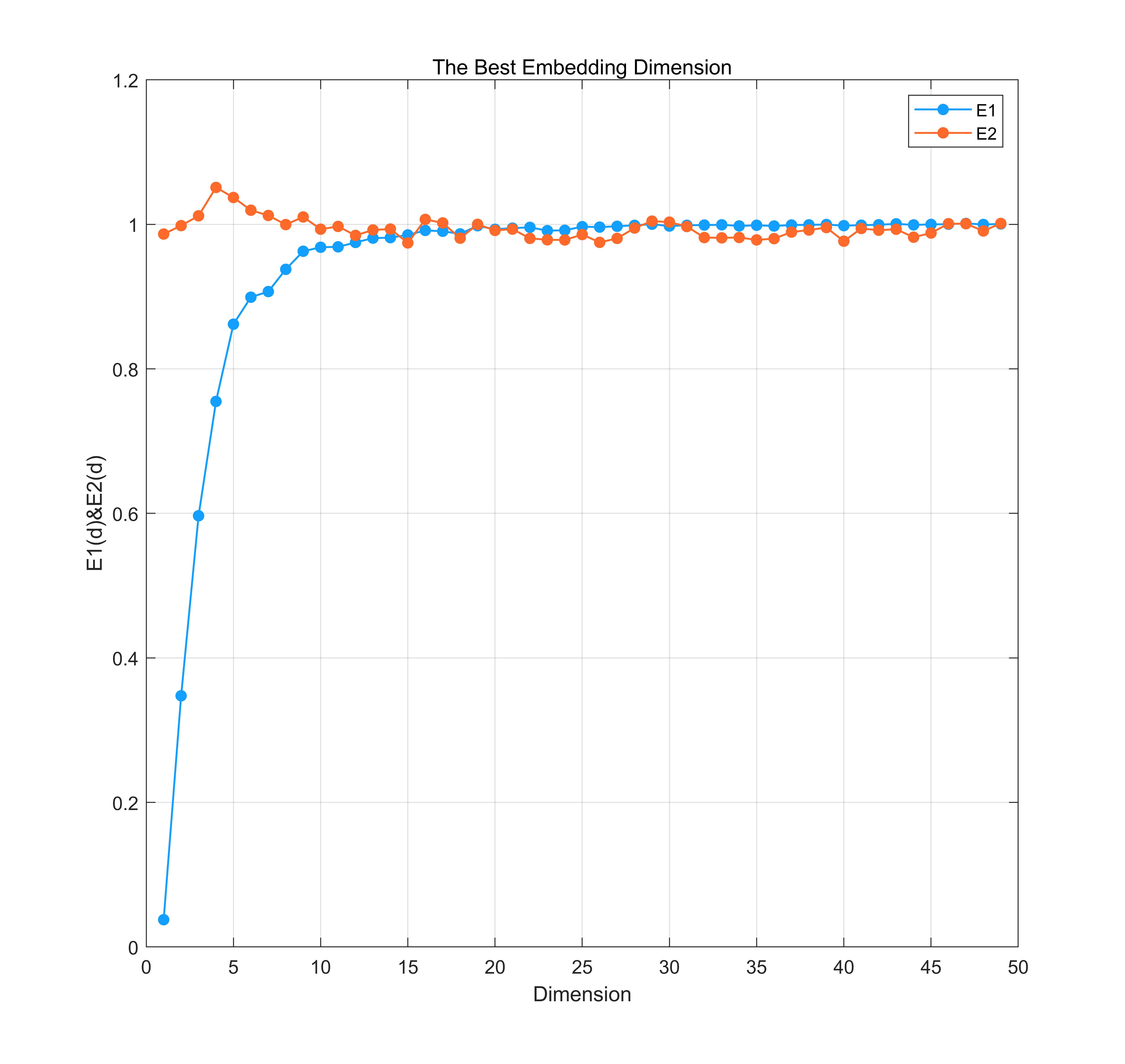}  
    }     
    \subfigure{ 
    \label{fig:SSE50ED}     
    \includegraphics[width=5cm,height=3.5cm]{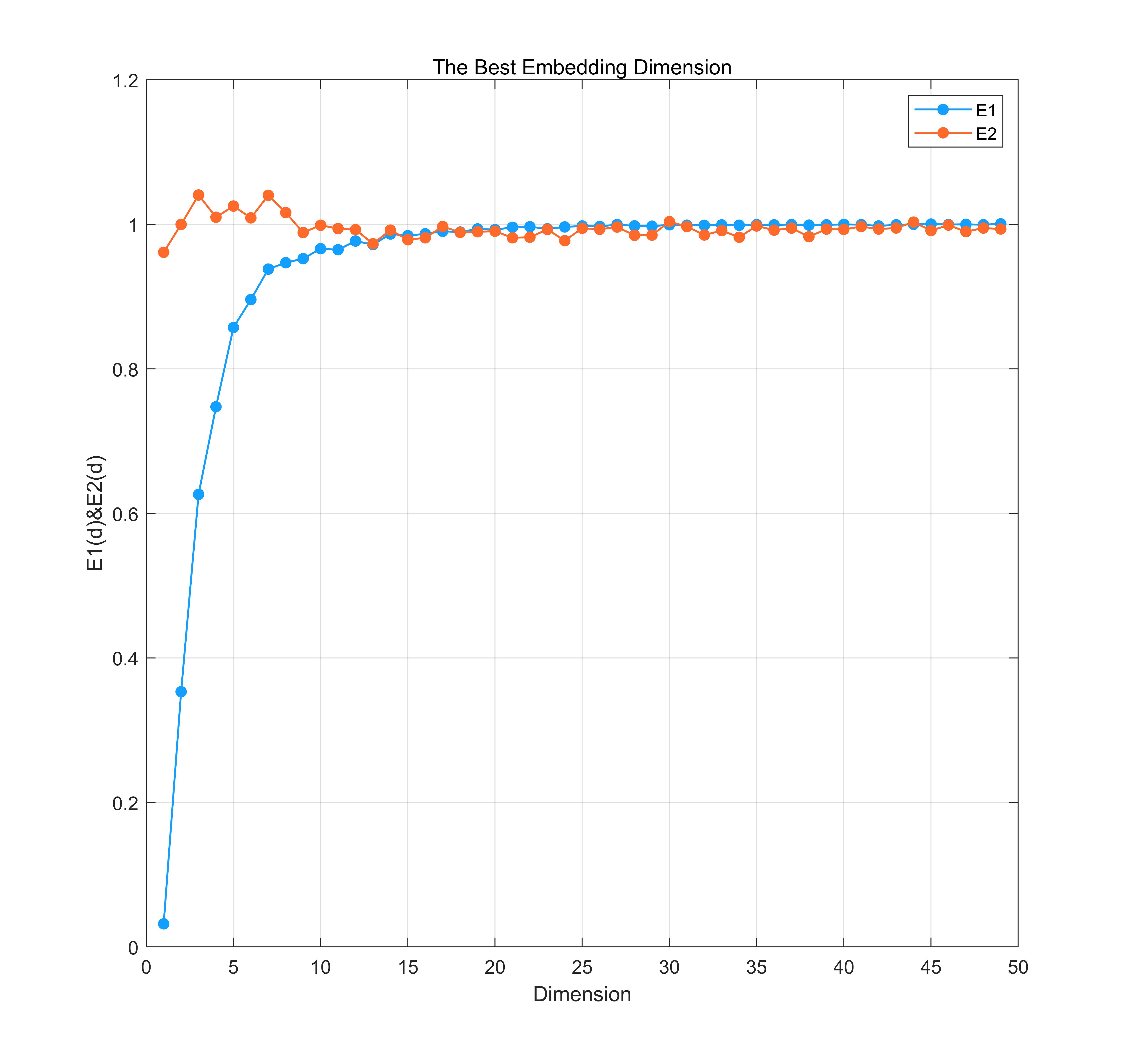}     
    }    
    \subfigure { 
    \label{fig:SSEXFED}     
    \includegraphics[width=5cm,height=3.5cm]{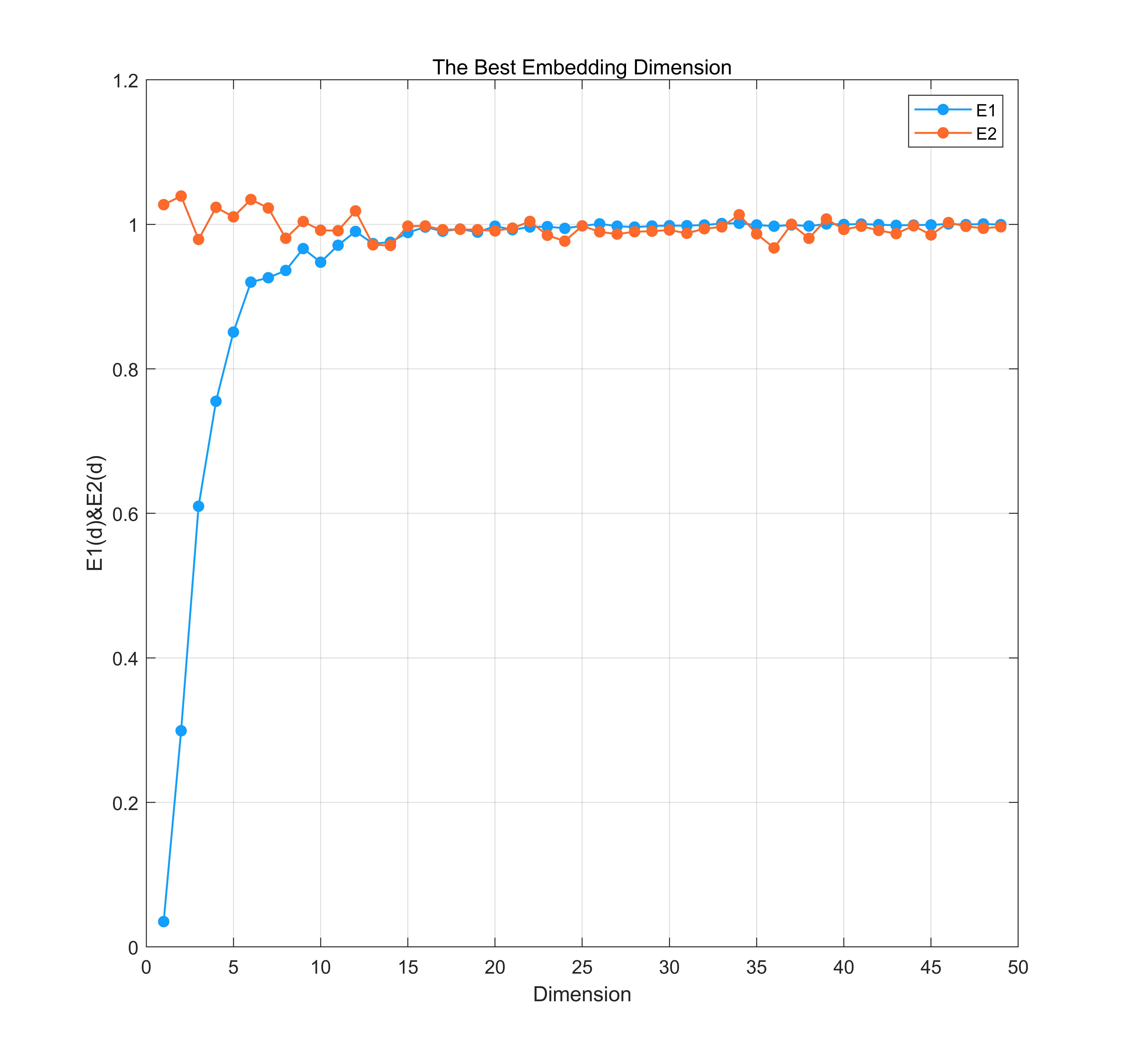}     
    }   
    \caption{ The Best Embedding Dimensions. ( left: SSE Energy Index, middle: SSE50 Index, right: SSE Consumer Index) }     
    \label{fig}     
\end{figure}

\subsection{Experiment 1: Forecasting with different stock patterns}\label{exp1}
When comparing different stock trend types, we only consider the forecast effect of the fourth day under the same $\alpha=1.5$. By observing the results of loss, we can determine which type of stock data our model is more suitable for. Here we mainly observe the prediction effect of the fourth day for judgment.

\begin{figure}[H]
    \centering    
    \subfigure {
     \label{fig:SSENYt+4-1.5}     
    \includegraphics[width=5cm,height=4cm]{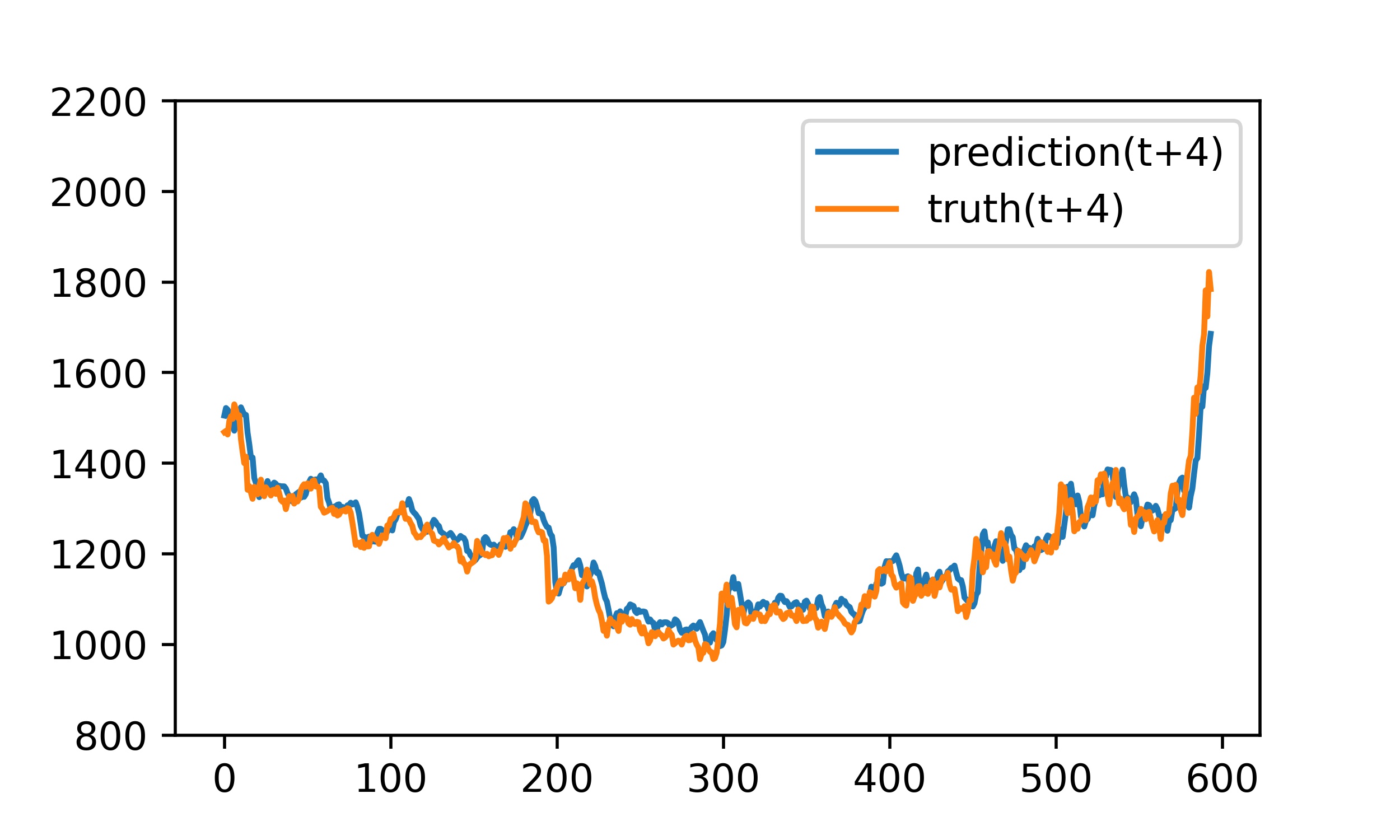}  
    }     
    \subfigure{ 
    \label{fig:SSEt+4-1.5}     
    \includegraphics[width=5cm,height=4cm]{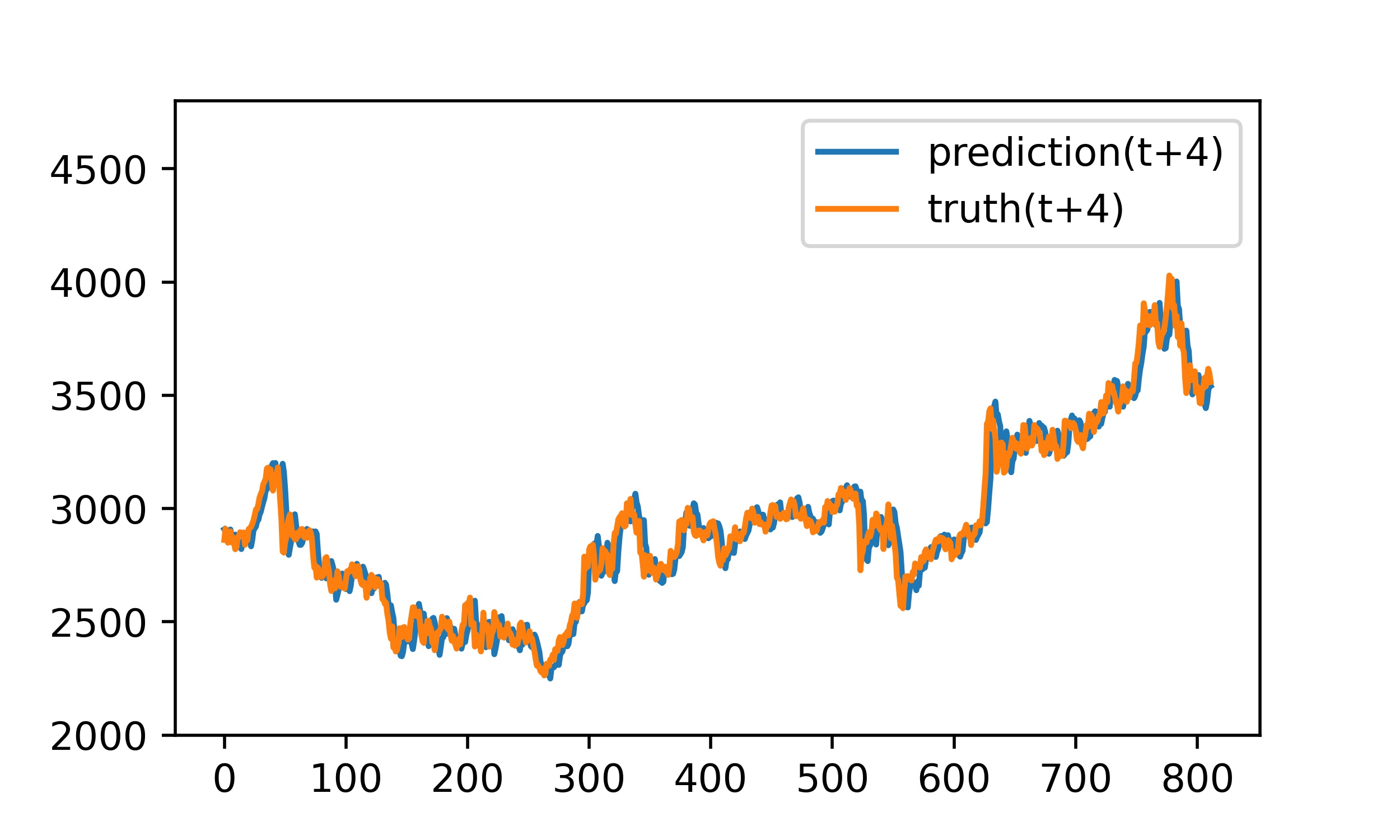}     
    }    
    \subfigure { 
    \label{fig:SZXFt+4-1.5}     
    \includegraphics[width=5cm,height=4cm]{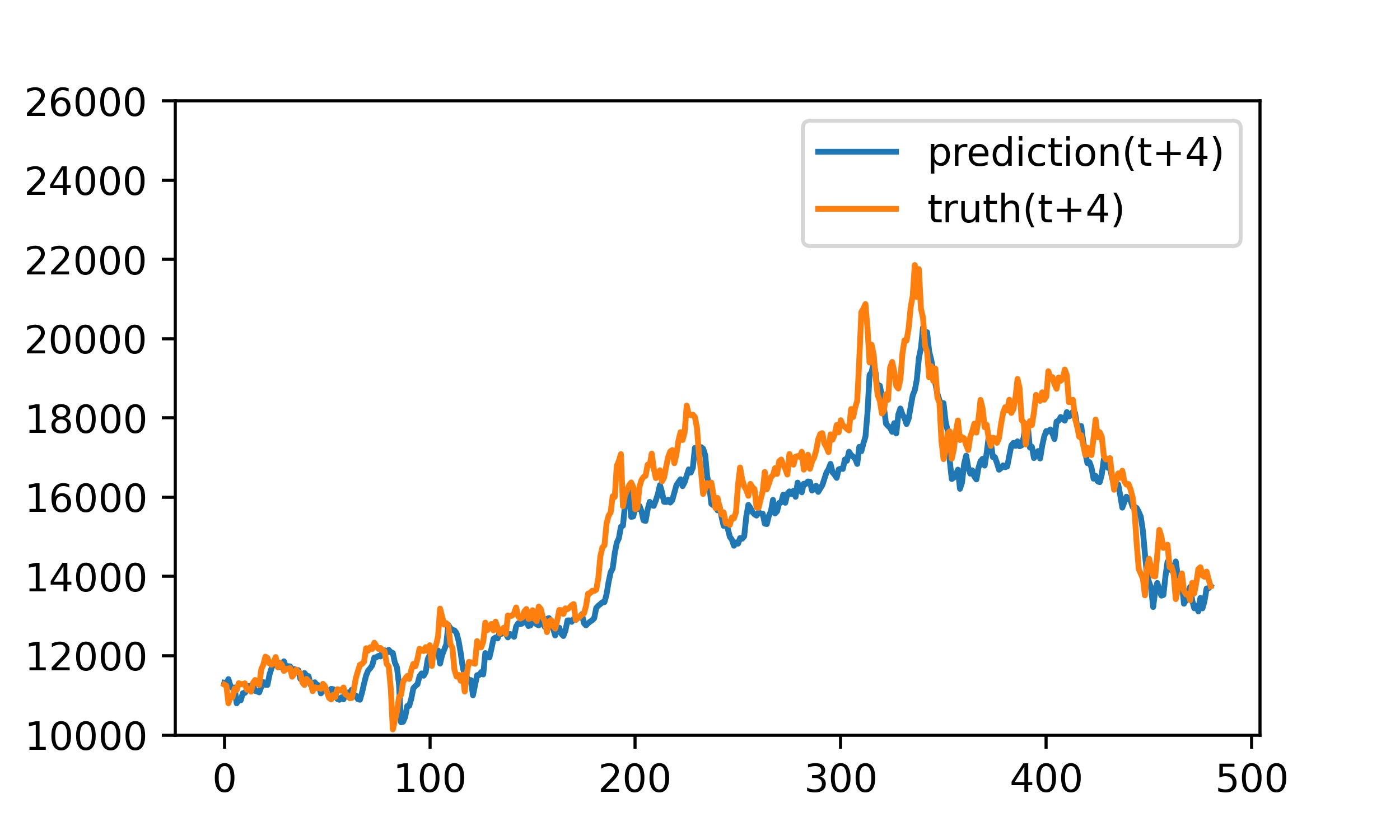}     
    }   
    \caption{ The prediction effect of different stocks on the fourth day. ( left:SSE Energy Index, middle: SSE50 Index, right: SSE Consumer Index) }     
    \label{fig:5}     
\end{figure}

\begin{table}[H]
	\caption{Forecast results of different stock ($\alpha=1.5$)}
	\centering
	\setlength{\textwidth}{15mm}{
	\begin{tabular}{cccccc}
		\toprule
		Data  & MSE(t+1) & MSE(t+2)& MSE(t+3)& MSE(t+4) \\
		\midrule
        SSE Energy Index &0.0010&0.0022&0.0034& 0.0045    
        \\
        SSE50 Index &0.0029&0.0058&0.0091&0.0125
        \\
        \rowcolor{mypink} SSE Consumer Index
        &0.2024&0.2210&0.1951&0.3301 
        \\
	\bottomrule
	\end{tabular}}
	\label{tab:different stock}
\end{table}

As can be seen from Figure \ref{fig:5}, the prediction effect of the SSE Consumer Index is the worst and the LDE-Net model is suitable for predicting the future value of stocks with large fluctuations in historical data. When the training set fluctuates more frequently with larger amplitude, L\'evy noise can better capture the uncertainty in the prediction process, so as to achieve more accurate prediction. However, historical data in Consumer Index flatten while the prediction set has bigger volatility, which is just the opposite of the trend situation described earlier. Therefore, the fitting effect between the predicted result and the true value is not ideal for Consumer Index. In the follow-up part, we will choose SSE Energy Index and SSE50 Index for experimentation to confirm the significance of LDE-Net.
%\vspace{-0.8cm}
\subsection{Experiment 2: Forecasting with different models}\label{exp2}
Under the condition of setting $\alpha = 1.5$, we compare the proposed model with the traditional SDE-Net \cite{Kong2020SDENetED} and commonly used forecasting models for stock data to check whether the proposed model makes sense. We have chosen LSTM \cite{hochreiter1997long} and ARIMA \cite{box2015time}, for common models. Since the accuracy of forecasting one day is the highest, we only compare the effect of forecasting one day. Here we only show the prediction effect images of the SSE Energy Index with the same coordinate scales.
\begin{figure}[H]
    \centering
    \begin{minipage}[t]{0.48\textwidth}
        \centering
        \includegraphics[width=6cm]{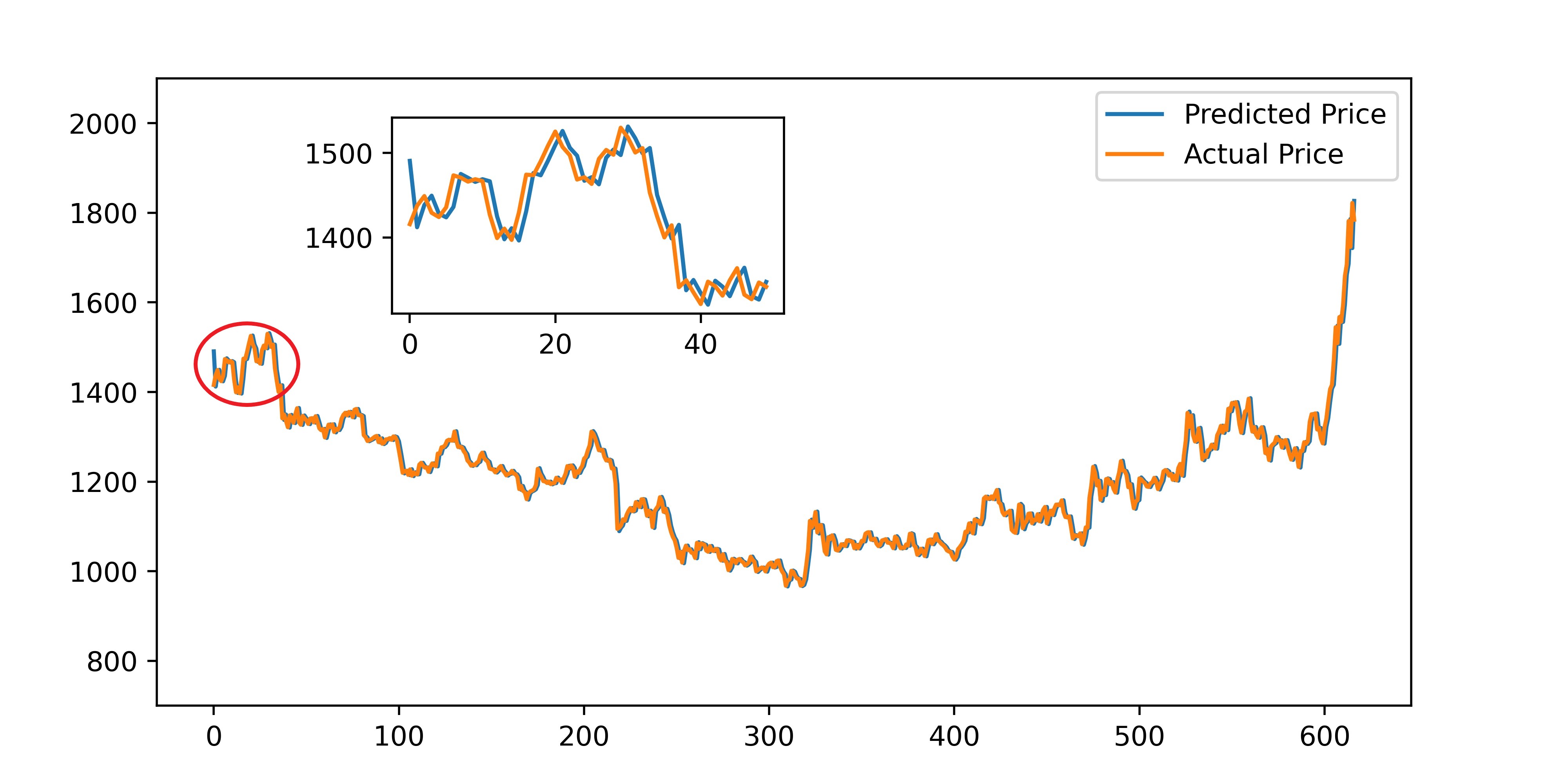}
        \label{fig:SZNYARIMA}
    \end{minipage}
    \begin{minipage}[t]{0.48\textwidth}
        \centering
        \includegraphics[width=6cm]{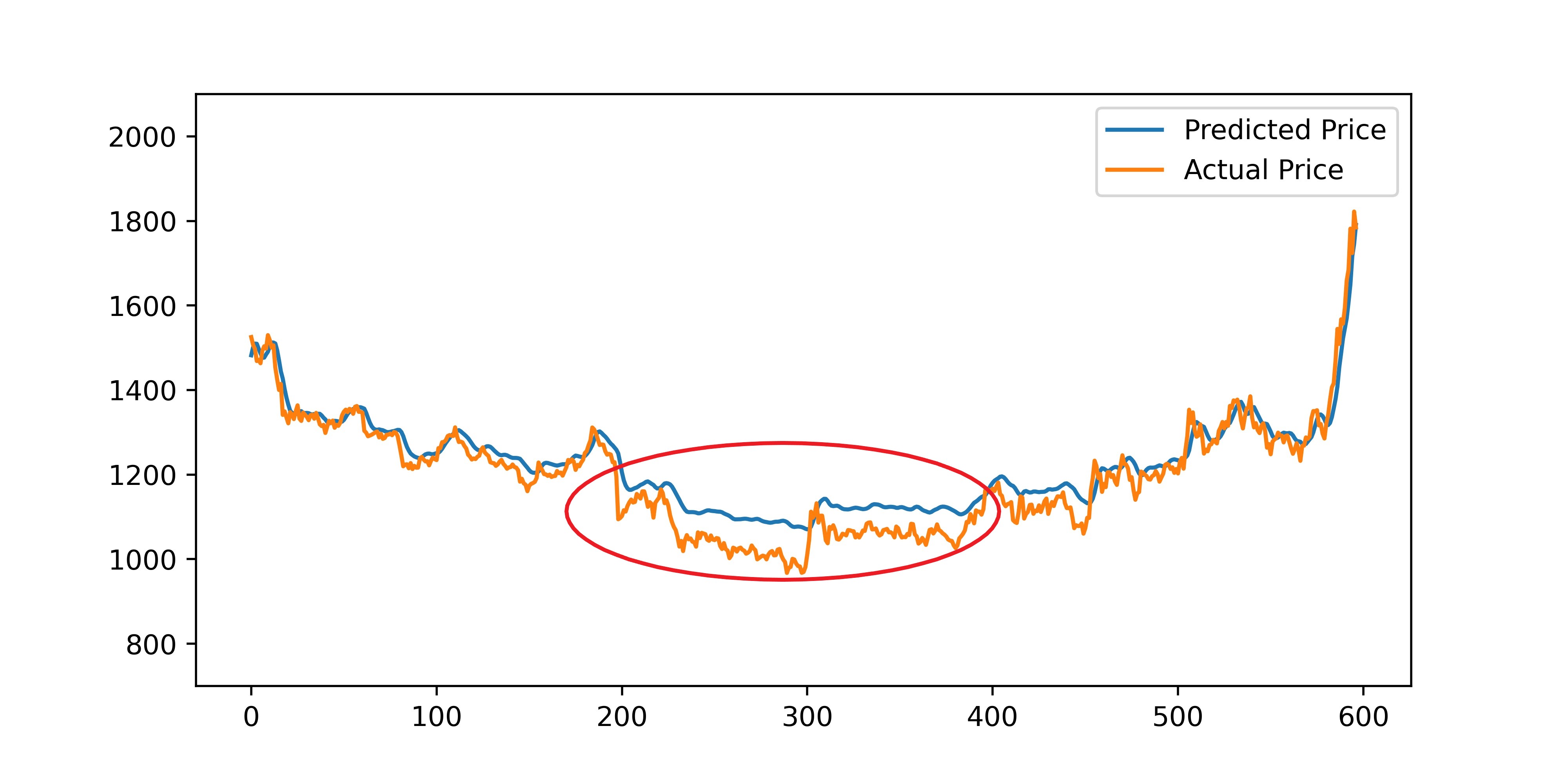}
        \label{fig:SZNYLSTM}
    \end{minipage}
    \begin{minipage}[t]{0.48\textwidth}
        \centering
        \includegraphics[width=6cm]{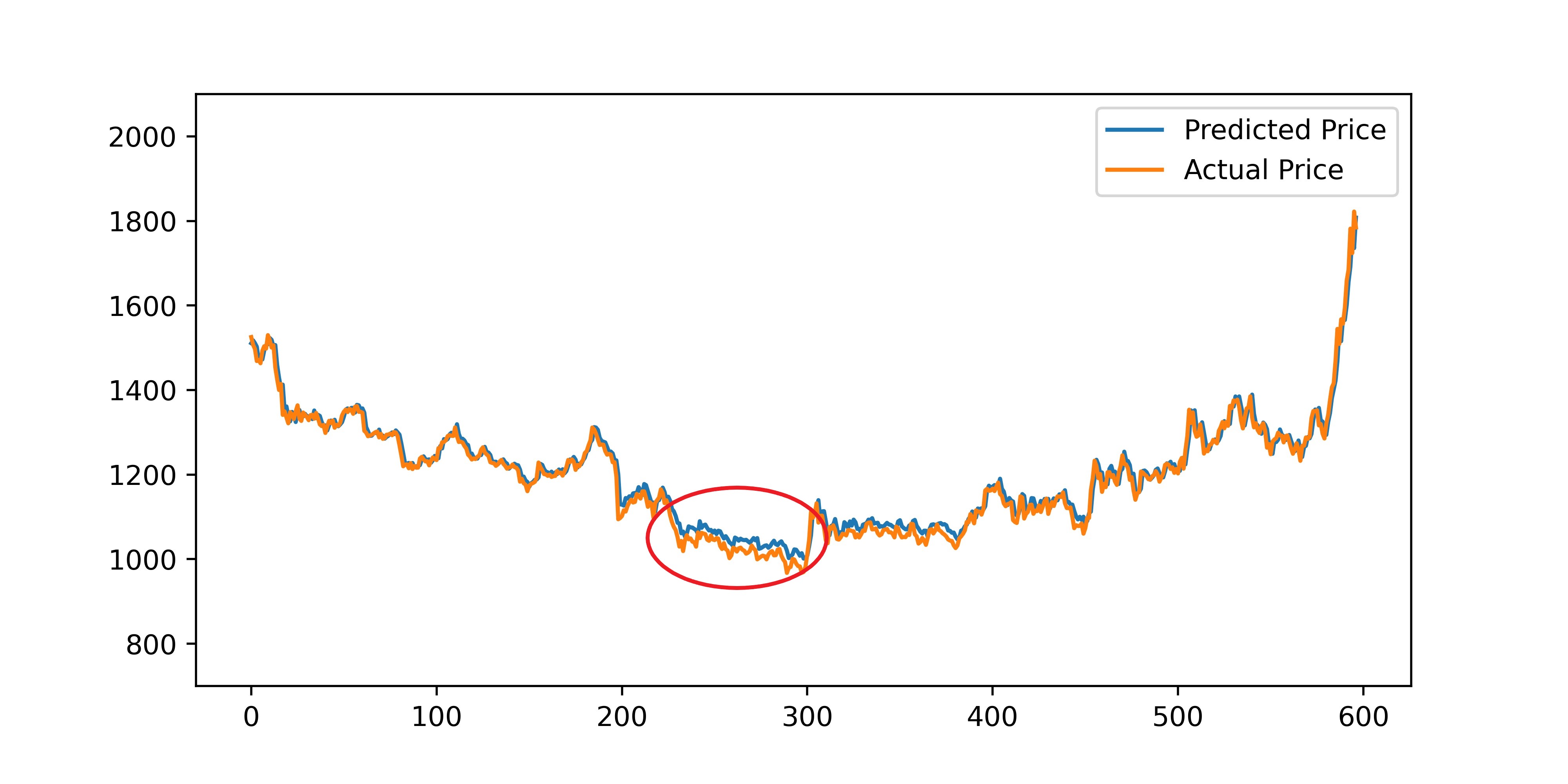}
        \label{fig:SZNYOLD}
    \end{minipage}
    \begin{minipage}[t]{0.48\textwidth}
        \centering
        \includegraphics[width=6cm]{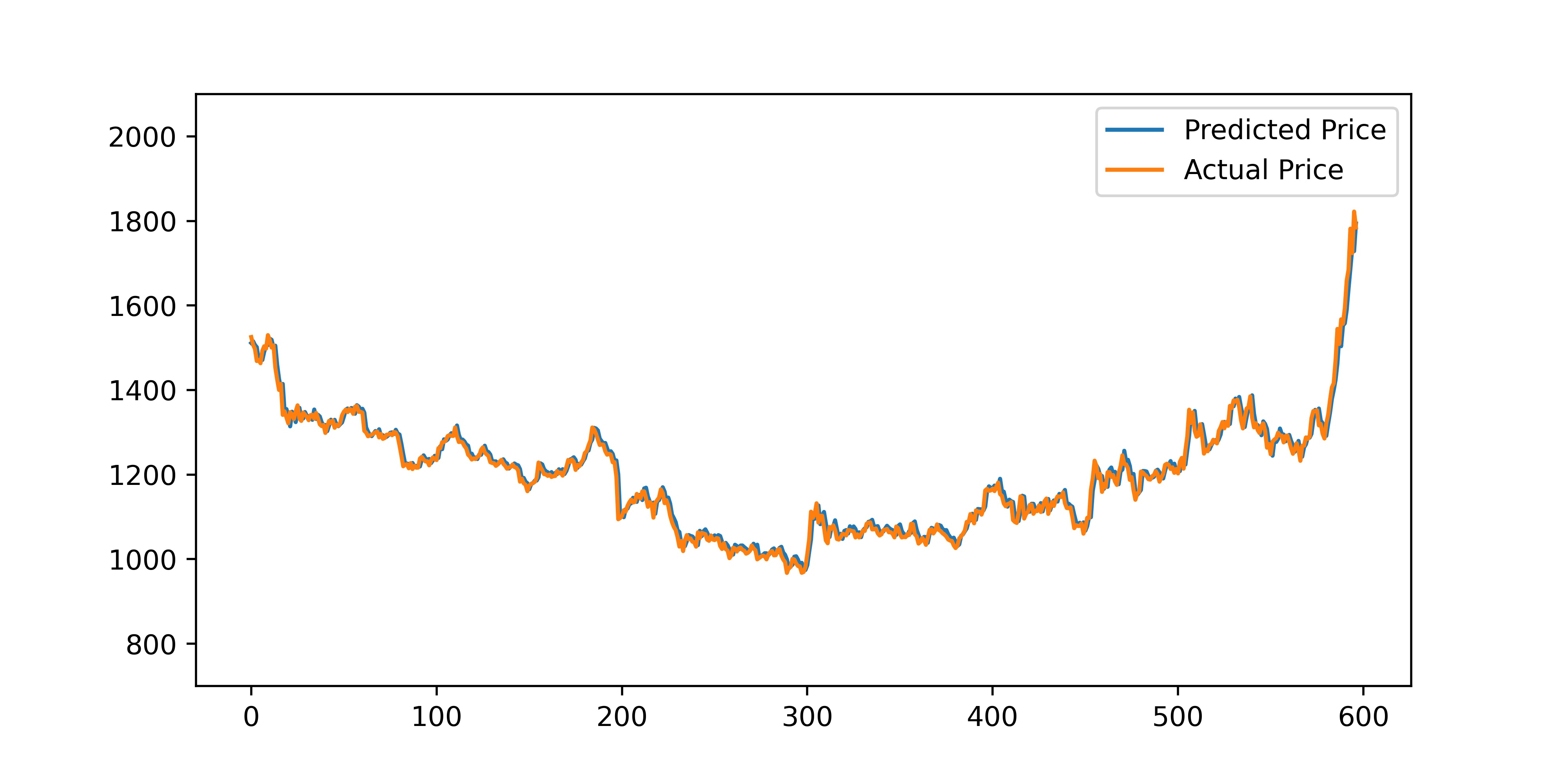}
        \label{fig:LDE-Net}-
    \end{minipage}
    \caption{For SSE Energy Index, the predictive effects of different models. (The upper left is ARIMA, the upper right is LSTM, the lower left is SDE-Net, and the lower right is LDE-Net.)}
    \label{fig:all models}
\end{figure}
%\vspace{-0.8cm} 
From the above figures (Figure \ref{fig:all models}) only for SSE Energy Index, we can intuitively find that the prediction effect of LDE-Net is the best. The predicted values of LSTM and SDE-Net in the next 200-300 days (testing set) are obviously not close to the true values. Although the predicted values of ARIMA seem to fit very well overall, it can be seen that the first few real values are not close to the predicted values. We select the first 50 data of the testing set and refine the coordinate scale to zoom in and observe the prediction effect. It can be found that the curve fitted by ARIMA has a significant time lag. Therefore, ARIMA seems to be accurate on the surface, but it is a visual error caused by the large coordinate scale. To see the difference more specifically, we give a comparison table of loss defined by MSE, RMSE, and MAE for the SSE Energy Index and SSE 50 Index. 

\begin{table}[H]
	\caption{Forecast results of different models}
	\centering
	\begin{tabular}{ccccc}
		\toprule
		Data      & Model     & MSE &RMSE & MAE \\
		\midrule
        \multirow{5}*{SSE Energy Index} & ARIMA & 0.0188&0.1371&0.0945   \\
		& LSTM & 0.0054& 0.0735&0.0597  \\
		& SDE-Net & 0.0011 & 0.0335& 0.0256     \\
		& \multicolumn{1}{>{\columncolor{mygreen}}l}{LDE-Net} &\multicolumn{1}{>{\columncolor{mygreen}}l}{ \textbf{0.0009}}  &\multicolumn{1}{>{\columncolor{mygreen}}l}{\textbf{0.0295}}& \multicolumn{1}{>{\columncolor{mygreen}}l}{\textbf{0.0209}}   \\
		\midrule
        \multirow{5}*{SSE 50 Index} & ARIMA & 0.0113&0.1064&0.0767     \\
		& LSTM  &  0.0031 &0.0554&0.0408   \\
		& SDE-Net & 0.0030&0.0546&0.0394      \\
		& \multicolumn{1}{>{\columncolor{mygreen}}l}{LDE-Net} & \multicolumn{1}{>{\columncolor{mygreen}}l}{\textbf{0.0029}}&\multicolumn{1}{>{\columncolor{mygreen}}l}{\textbf{0.0537}}&\multicolumn{1}{>{\columncolor{mygreen}}l}{\textbf{0.0386}}
		\\
		\bottomrule
	\end{tabular}
	\label{tab:different models}
\end{table}

From the three indicators in Table \ref{tab:different models}, it is verified that the conclusions we have obtained based on the images are correct. For the two types of stock representative data, the prediction effect of our LDE-Net is better than other models. For SSE Energy Index, the MSE value is about 0.08672$\%$, and the accuracy is 29$\%$ higher than the original SED-Net. In addition, RMSE is about 2.945$\%$, and the error is 0.401$\%$ less than SED-Net. Also, MAE is reduced by about 0.47$\%$. For SSE 50 Index, the values of RMSE and MAE are about 0.1$\%$ less than SED-Net. In general, the proposal of LDE-Net has a certain meaning. 

\subsection{Experiment 3: Forecasting with different \texorpdfstring {$\alpha$} .}\label{exp3}
Considering that different $\alpha$ will produce different jumps, we want to observe the effect of model predictive ability without a fixed $\alpha$=1.5. Therefore, we input the reconstructed data into the LDE-Net model with $\alpha=1.2,1.3,1.4,1.5,1.6,1.8$ for the same stock. 
%\vspace{-0.4cm} 
\begin{figure}[H]
    \centering
    \begin{minipage}[t]{0.3\textwidth}
        \centering
        \includegraphics[width=5cm]{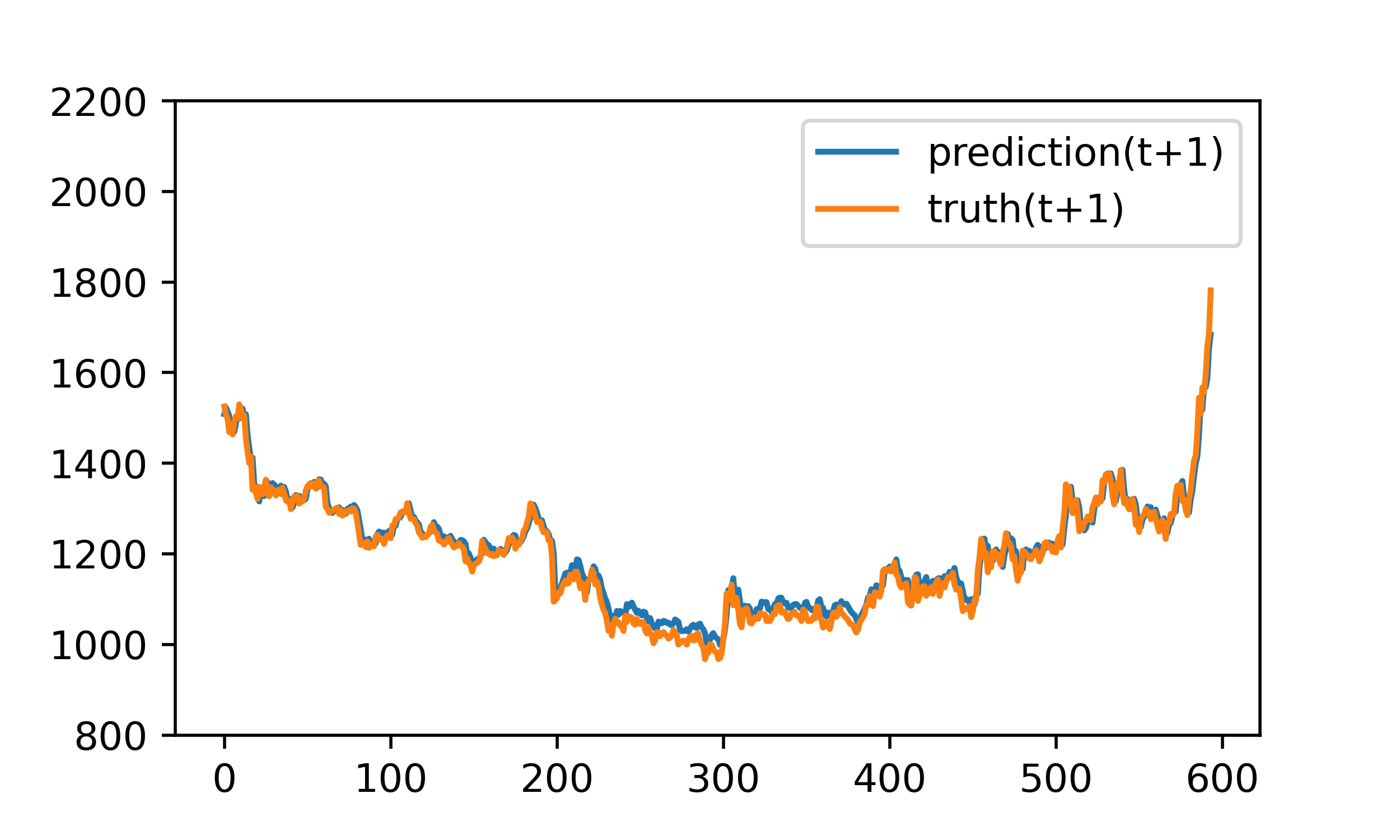}
        \label{fig:SZNY1.2}
    \end{minipage}
    \begin{minipage}[t]{0.3\textwidth}
        \centering
        \includegraphics[width=5cm]{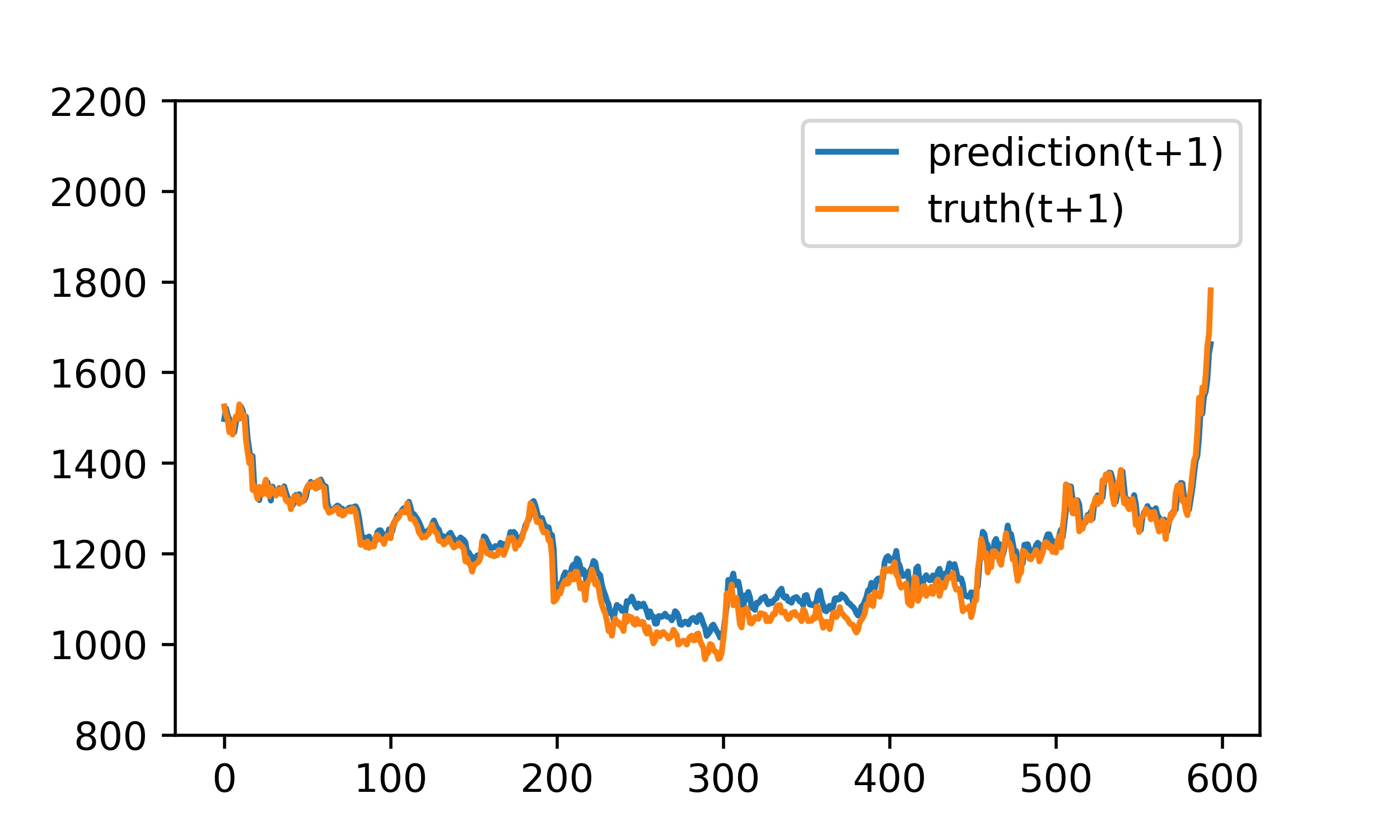}
        \label{fig:SZNY1.3}
    \end{minipage}
    \begin{minipage}[t]{0.3\textwidth}
        \centering
        \includegraphics[width=5cm]{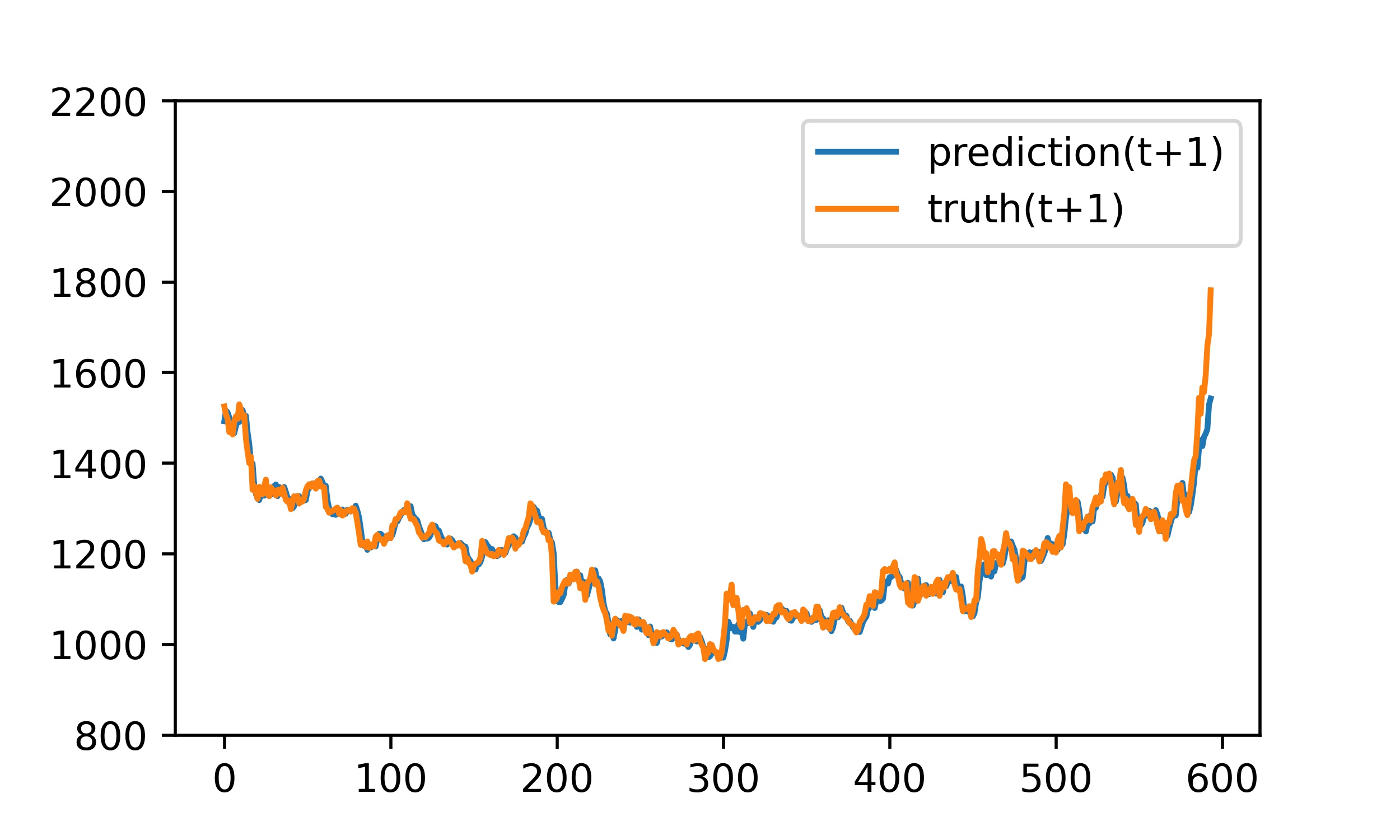}
        \label{fig:SZNY1.4}
    \end{minipage}
    \begin{minipage}[t]{0.3\textwidth}
        \centering
        \includegraphics[width=5cm]{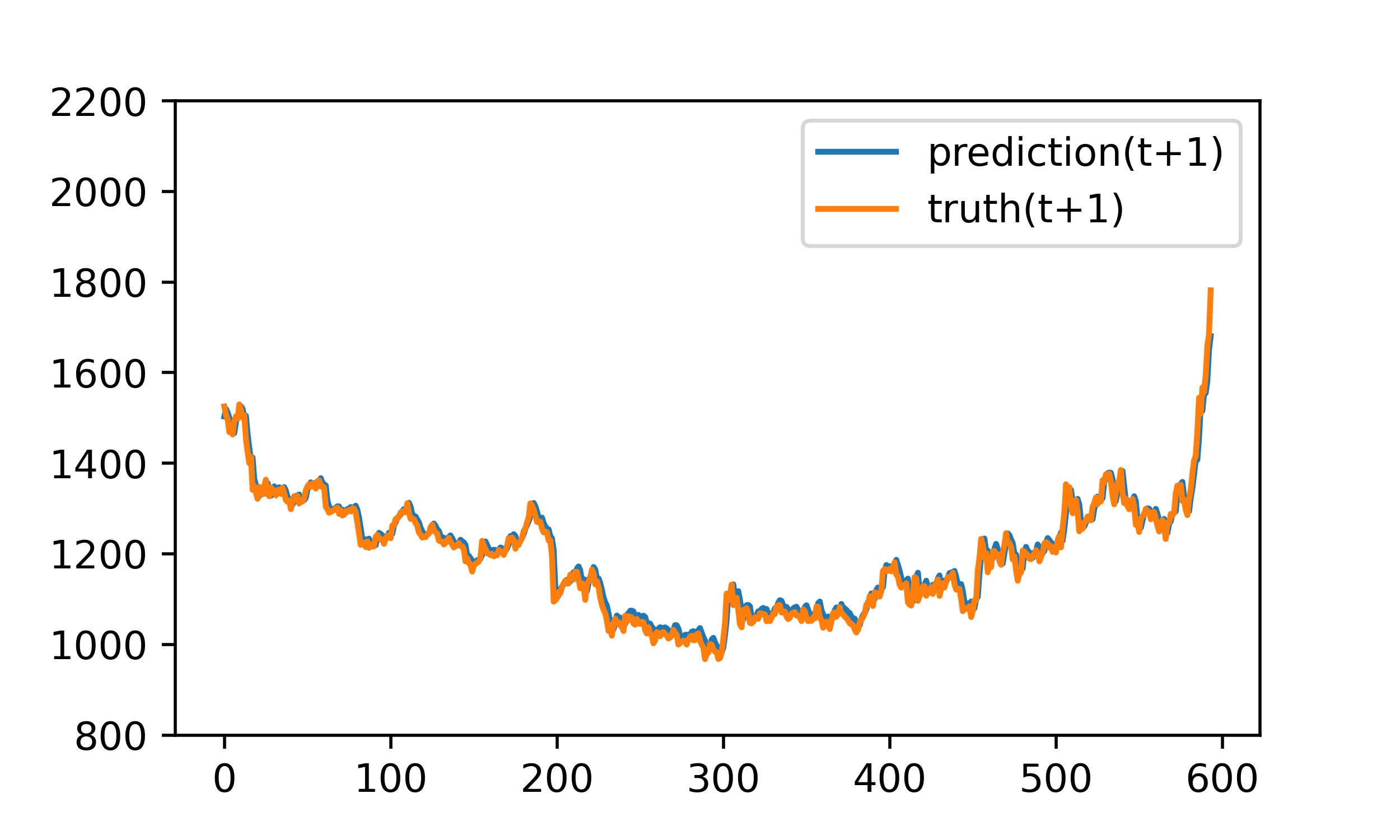}
        \label{fig:SZNY1.5}
    \end{minipage}
    \begin{minipage}[t]{0.3\textwidth}
        \centering
        \includegraphics[width=5cm]{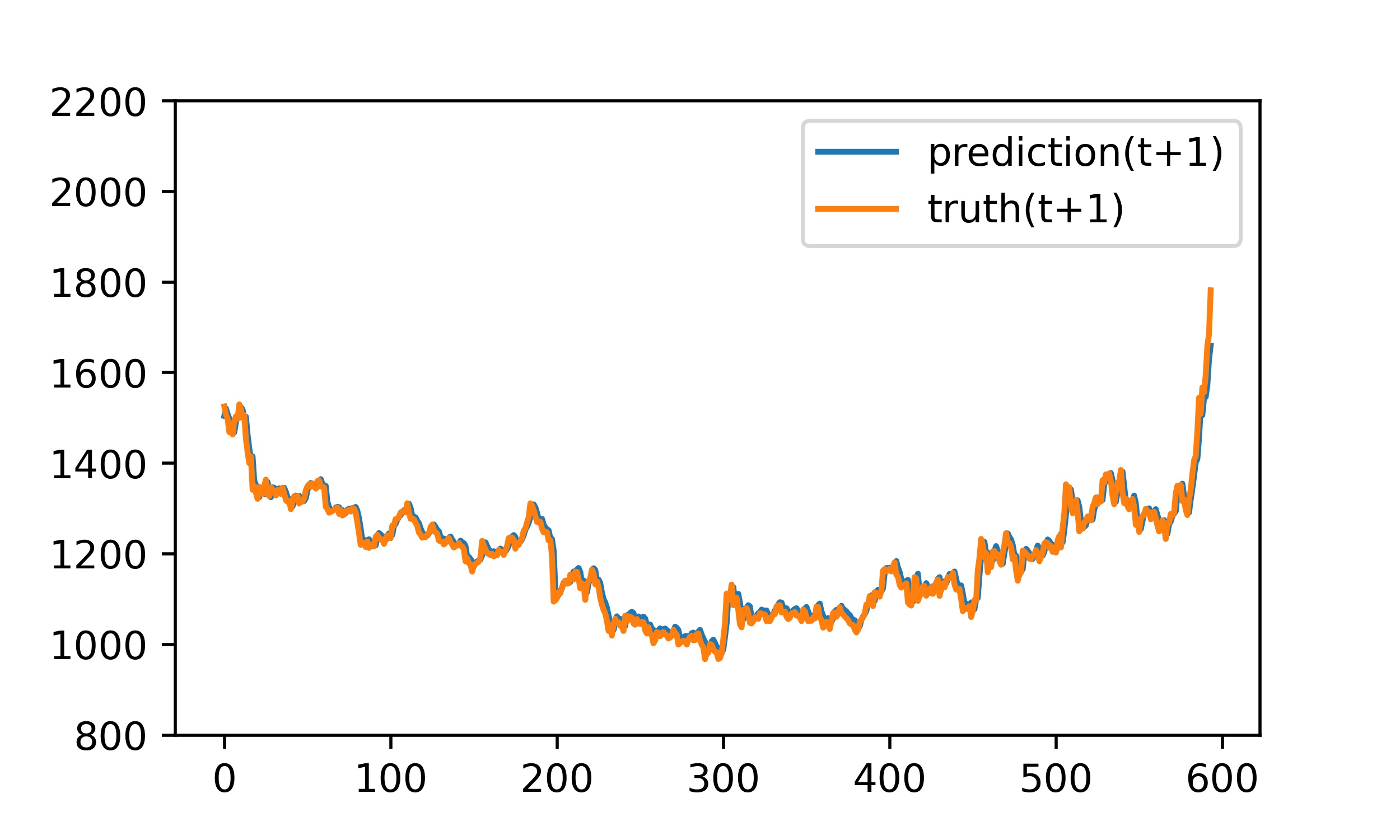}
        \label{fig:SZNY1.6}
    \end{minipage}
    \begin{minipage}[t]{0.3\textwidth}
        \centering
        \includegraphics[width=5cm]{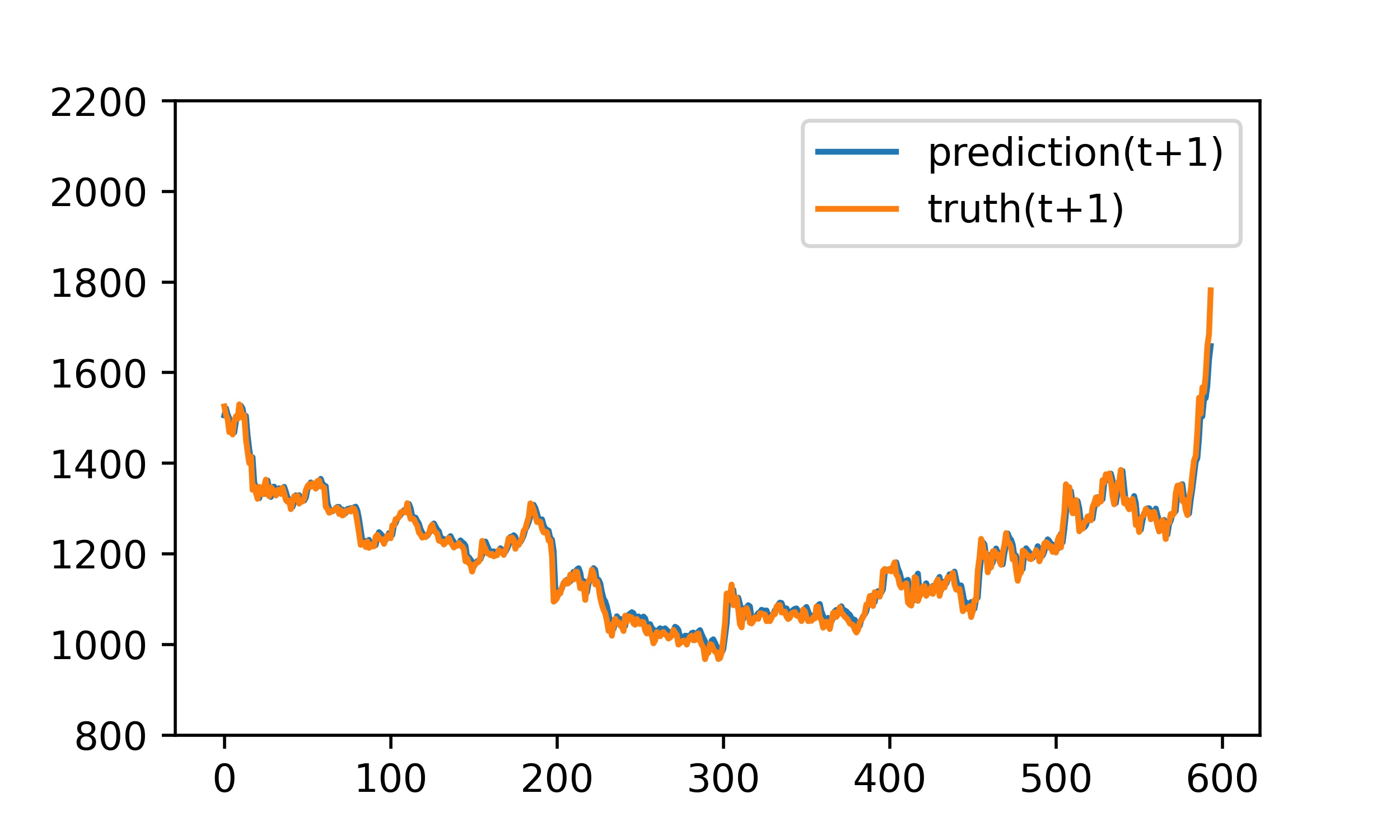}
        \label{fig:SZNY1.8}
    \end{minipage}
    \caption{For SSE energy index, the predictive effect of different $\alpha$. (The upper left is $\alpha=1.2$, the upper middle is $\alpha=1.3$, the upper right is $\alpha=1.4$, the lower left is $\alpha=1.5$, the lower middle is $\alpha=1.6$, and the lower right is $\alpha=1.8$)}
\end{figure}

\begin{table}[H]
	\caption{Forecast results of different $\alpha$}
	\centering
	\setlength{\textwidth}{15mm}{
	\begin{tabular}{ccccccc}
		\toprule
		Data      &  Value of $\alpha$     & MSE(t+1)1.0*1e-3 & MSE(t+2)1.0*1e-3& MSE(t+3)1.0*1e-3& MSE(t+4)1.0*1e-3 \\
		\midrule
		\multirow{6}*{SSE Energy Index} 
        & 1.2 &1.1555&2.8272&3.9941&5.2052    \\
        & 1.3 &2.0929&2.2900&4.1513&5.3823     \\
        & 1.4 &1.7146&2.4714&3.1843&4.6635
        \\
        & 1.5 &1.0317&2.2017&3.3650&4.5471     \\
        & \colorbox{green}{1.6}&\textbf{0.9455}&\textbf{2.1151}&\textbf{2.8728}&\textbf{5.4028}    
        \\
        & 1.8 &0.9890&2.0429&2.8211&5.5467  
        \\
		\midrule
        \multirow{6}*{SSE50 Index} 
        & 1.2 &3.2909&6.2504&9.3327&13.3091     \\
		& 1.3 &3.1861&6.0212&9.6572&12.8533     \\
		& 1.4 &2.9535&5.9002&9.4551&12.8543     \\
		& \colorbox{green}{1.5} &\textbf{2.9415}&\textbf{5.8059}&\textbf{9.1026}&\textbf{12.4943}     \\
		& 1.6 &2.9809&6.0121&9.6383&12.8007    
		\\
		& 1.8 &2.9804&6.0358&9.4742&12.6342  
		\\
	\bottomrule
	\end{tabular}}
	\label{tab:different alpha}
\end{table}

From the results of the MSE value in Table \ref{tab:different alpha}, it can be seen that the value of $\alpha$ can have a certain influence on the prediction of the model. For example, for the same stock, it can be seen that the best $\alpha$ of the SSE Energy Index is 1.6, while the best $\alpha$ of the SSE 50 Index is 1.5. For different stocks, the greater the data fluctuation, the smaller $\alpha$ value of the stock, which is in line with the nature of L\'evy motion that is smaller $\alpha$ can give greater jump.

\subsection{Experiment 4: Effect of forecasting days on prediction accuracy }\label{exp4}

We also investigate the order of accuracy over predicting days and get some interesting results. For example, the error is roughly proportional to the length of forecasting days, even if we change the value of $\alpha$. This linear relationship indicates that the error rate (increment of error vs. prediction time) is of $O(1)$. For longer forecasting days such as ten days, the error still has the same phenomenon (See Figure \ref{fig:B3} in Appendix \ref{different steps}). Moreover, from the figures (Figure \ref{fig:B4}, Figure \ref{fig:B5}), we can see that there will be a prediction shift with longer forecasting days. Hence, long-term prediction is our future improvement direction.

\begin{figure}[H]
    \centering    
    \subfigure {
     \label{fig:SZNYalpha}     
    \includegraphics[width=5cm,height=4cm]{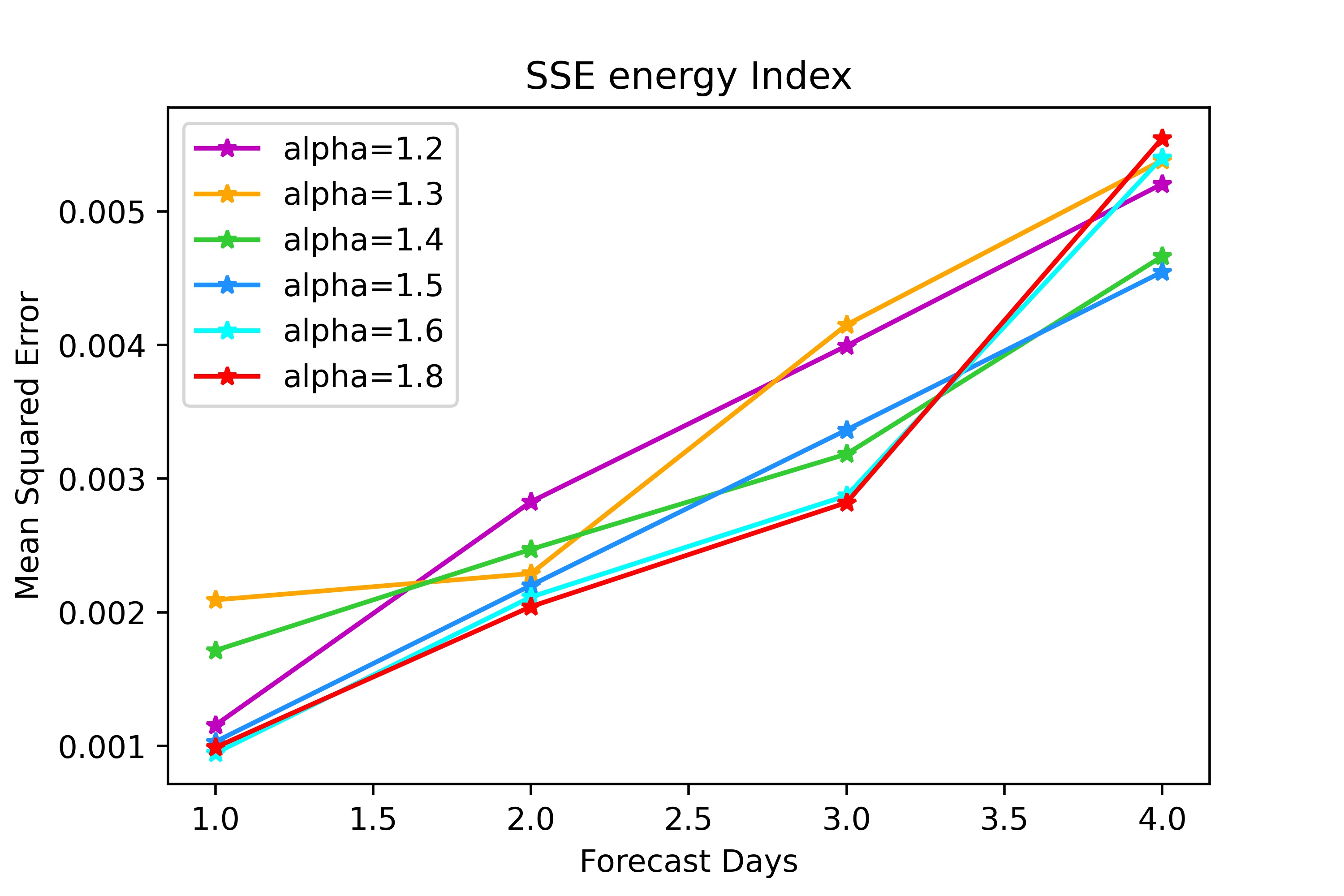}  
    \includegraphics[width=5cm,height=4cm]{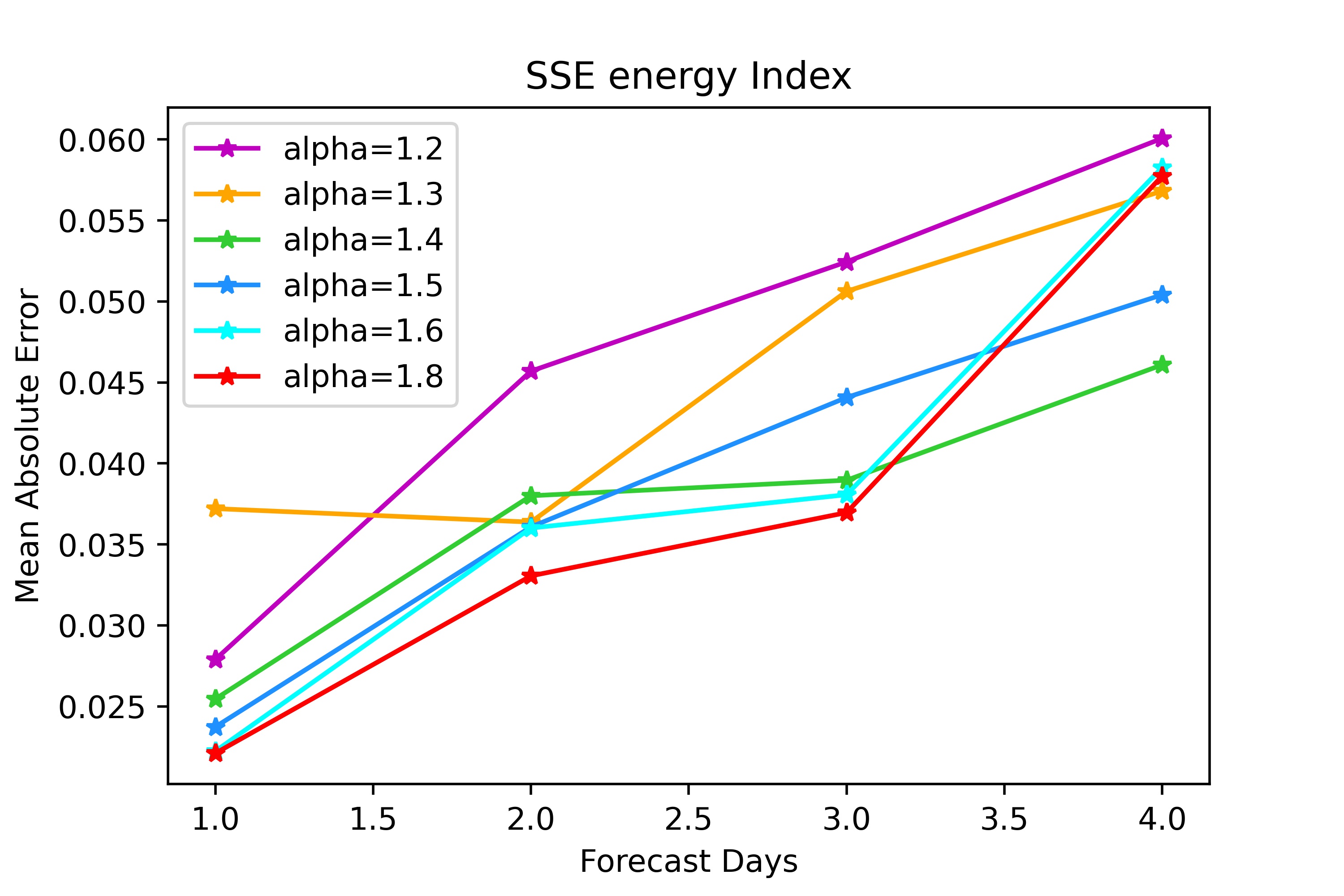}  
    }     
    \subfigure{ 
    \label{fig:SZ50alpha}     
    \includegraphics[width=5cm,height=4cm]{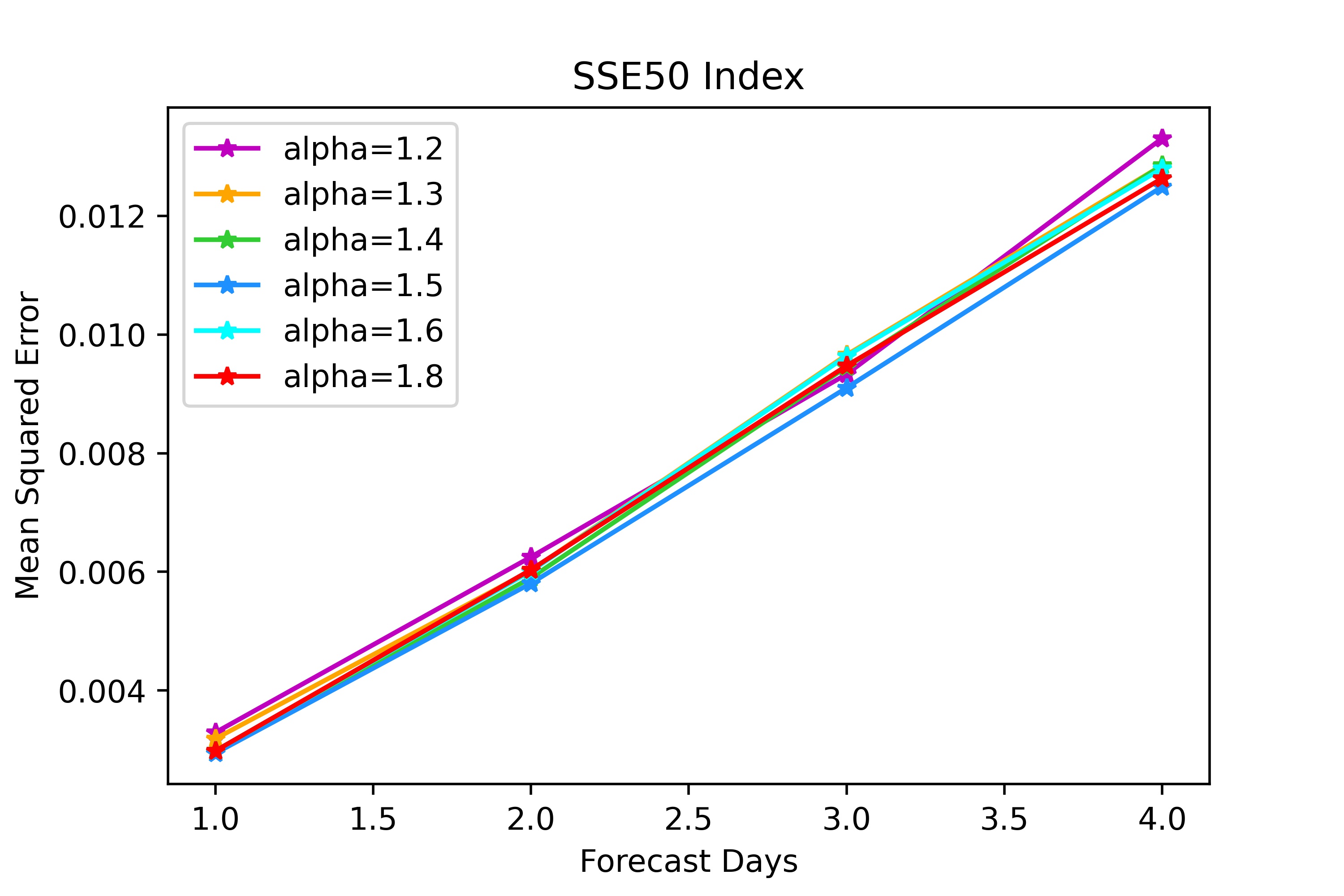}     
    \includegraphics[width=5cm,height=4cm]{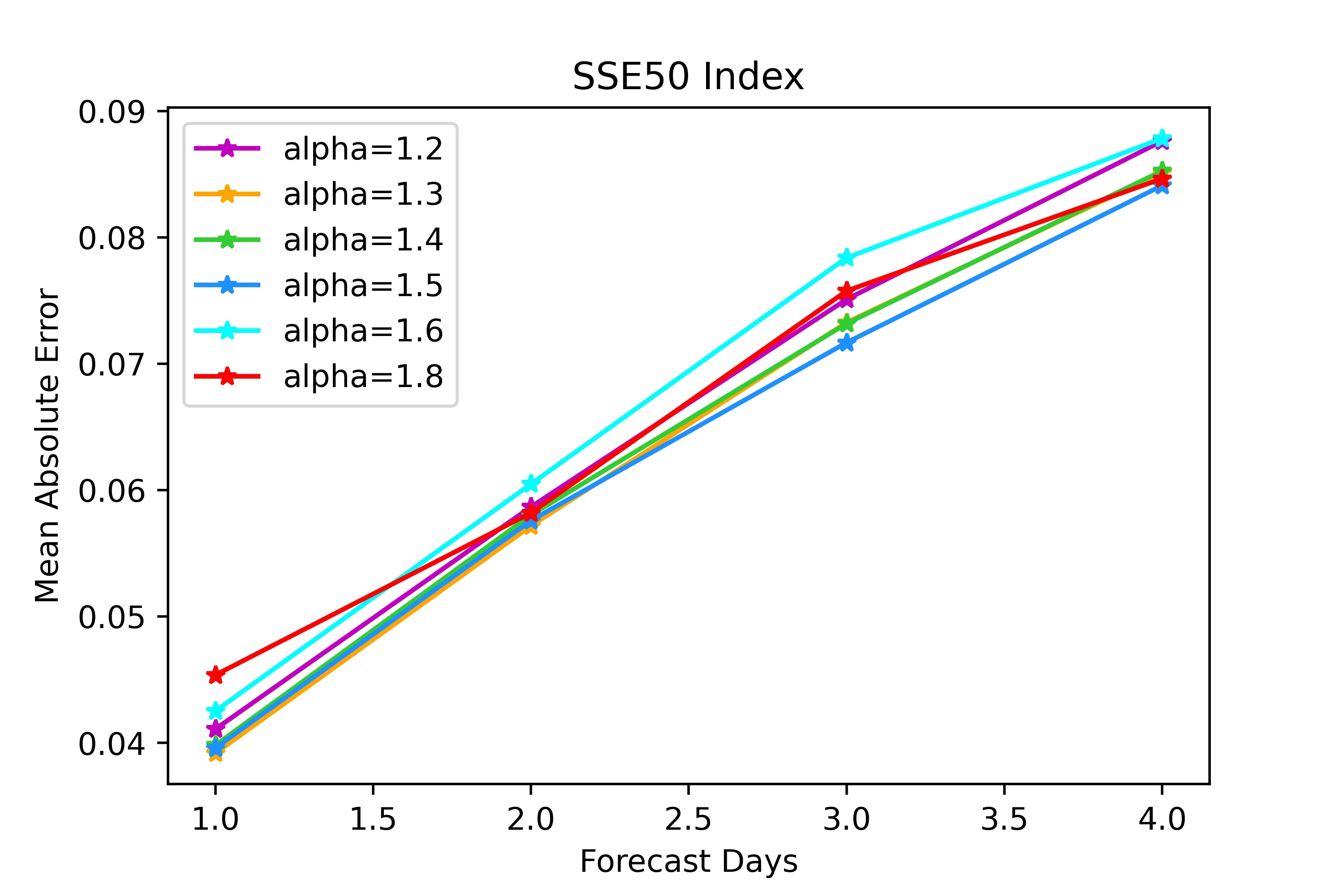} 
    }    
    \caption{ The prediction error rate over forecasting days: mean squared error of SSE Energy Index (upper left), the mean absolute error of SSE Energy Index  (upper right), mean squared error of SSE50 Index (lower left), mean absolute error of SSE50 Index (lower right).} 
    \label{fig:10}     
\end{figure}

\section{Conclusions and discussion}
In this paper, we present a series of latent stochastic differential equations equipped with neural networks to forecast financial data, which is non-stationary with noise.
The special properties of $\alpha$-stable L\'evy motion induced stochastic differential equations could help us to capture data and model uncertainty in a better way. To effectively utilize the information of real financial data, on the one hand, we explore its chaotic behavior and obtain the optimal intrinsic dimension of our input training data. And on the other hand, we use the attention mechanism to perform parallel operations on multiple models to achieve multi-step forecasting. All in all, we conclude our main findings in the following: First, the model LDE-Net we proposed is more suitable for predicting time series data that fluctuates with high volatility, as $\alpha$-stable L\'evy motion can better capture the training data trend with epistemic jumps. Second, we have theoretically proved the convergence of our algorithm, without curse of dimensionality. Finally, we also illustrated that the order of accuracy linearly changes over forecasting days.

Due to the fact that our model can better capture the underlying information in the case of historical data with jumps, LDE-Net algorithm with high accuracy can be applied to various chaotic time series predictions, such as index stocks, agricultural product prices, electricity load, climate change, remaining useful life prediction in bearing degradation, etc. For financial data, it can help investors to select the best portfolio investment as well as some indicators for timing strategies. 
%Considering that our modeling is based on real time series, our model can also be parallelized to other arbitrary time series applications with a certain authenticity and wide applicability. 
In addition, with multi-step forecasting in our model, people can obtain long-term information.

For future research work, one aspect can be optimizing the selection of $\alpha$ values and another direction is to estimate generalization error bound in terms of data complexity. Simultaneously, considering that self-attention has a good performance in long-term prediction, we hope it can improve our attention mechanism. In addition, the application of the adjoint sensitivity method to stochastic differential equations is also worthy of further consideration. Since the discretized Euler method is relatively rough, the neural ordinary differential equation (Neural ODE) applies the ODE Solver to the adjoint system, which can both reduce computational cost and improve model accuracy efficiently. Considering the advantages of ODE Solver and the benefits of time continuation, it might be interesting to extend our current model into a continuous version under a similar framework of Neural ODE. As in the neural stochastic differential equation (Neural SDE), the adjoint stochastic differential equation is constructed under the assumption of Stratonovich noise, which is easier compared with the It\^o case. Therefore, for stochastic differential equations induced by L\'evy motion, it is novel to extend the adjoint sensitivity method correspondingly.

Our method has similar leanings on drift coefficient and diffusion coefficient like the work of stochastic physics-informed neural networks (SPINN), yet they have different goals. The former one is aiming for prediction with implicit compounded drift and diffusion coefficients while the latter can identify an unknown physical system explicitly. Therefore, it is interesting to learn an inverse problem through the framework of our model.

%  Specifically, stochastic differential equation is a good description of the data we care about, and the kinetic properties of stochastic differential equation are in turn determined by drift and diffusion terms. Therefore, we implicitly learn drift coefficient and diffusion coefficient through the neural network, and then obtain a stochastic differential equation (no specific form) implicitly described by the neural network that can well describe the characteristics of the data.}

 %\href{code}{https://github.com/senyuanya/LDENet} .

\section*{Acknowledge}
This work was supported by National Natural Science Foundation of China (NSFC) 12141107. We would like to thank Wei Wei, Zibo Wang, Ao Zhang and Huifang Huang for helpful discussion.  We also thank the editor, the associated editor and the anonymous reviewers for numerous constructive
suggestions which help improve the presentation of this article.

\section*{Data Availability}
The data that support the findings of this study are available from the corresponding author upon reasonable request. Data source: \href{data}{https://www.joinquant.com/} .%\cite{JQ}.

\bibliographystyle{unsrt} 
\bibliography{references}  %%% Uncomment this line and comment out the ``thebibliography'' section below to use the external .bib file (using bibtex) .

\newpage
\begin{appendices}

\section{Appendix A }

\setcounter{equation}{0}
\renewcommand{\theequation}{A.\arabic{equation}}
\renewcommand{\thefigure}{B.\arabic{figure}}
\setcounter{equation}{0}
\setcounter{figure}{0}
\renewcommand{\thetheorem}{A.\arabic{theorem}}\label{appendixA}
% \section{Appendix B }

\subsection{ Lyapunov Exponent}
\label{chaos}
Based on Takens Theorem \cite{Takens1981DetectingSA}, we know that chaotic time series can restore the original system to establish a prediction model. Therefore, identifying the chaos of time series is the basis of prediction. There are many ways to identify whether a time series is chaotic, such as the largest Lyapunov exponent method, power spectrum method, system phase diagram structure method etc. Here we use the largest Lyapunov exponent method.

The Lyapunov exponent can describe the characteristics of the dynamics of the system. Assuming if that the value of Lyapunov is represented by $\lambda$, we have three cases: (i) $\lambda<0$, the motion state of the system tends to be stable and is not sensitive to the initial state of the system at this time. (ii) $\lambda=0$, the system is in a stable state. (iii) $\lambda>0$, the system movement will enter a chaotic state, and the corresponding mapping is called chaotic mapping. Therefore, judging the magnitude and sign of the Lyapunov exponent becomes a criterion for whether the system enters chaos.The calculation methods of the maximum Lyapunov exponent mainly include definition method, Wolf method \cite{wolf1985determining}, orthogonal method and small number method. For unknown nonlinear dynamic systems, we use the wolf method of orbit tracking to solve the maximum Lyapunov exponent of the sequence. 

\subsubsection*{Wolf Method}  %不带序号

Since we cannot find the exact differential equation of the original dynamic system, we only calculate the maximum Lyapunov exponent through univariate time series by using analysis and orbit tracking methods. In 1985, a method to estimate the maximum Lyapunov exponent was proposed \cite{wolf1985determining} based on the evolution of phase trajectories, phase planes, and phase volumes. This method is based on orbit tracking which is widely used in the prediction of chaotic time series.

Assuming that the chaotic time series is $x_1,x_2\dots x_n$, the embedding dimension is $m$, and the delay time is $\tau$, then the reconstructed phase space is

\begin{equation}
       Y(t_{i})=(x(t_{i}),x(t_{i+\tau}),\dots,x(t_{i+(m-1)\tau})),\,\,i=1,2,\cdots,N.
\end{equation}

Taking the initial point as $Y(t_{0})$ and the nearest point of it as $Y_{0}(t_{0})$, we define $L_{0}$ as the distance between them. Then, we track the evolution of the distance between these two points over time until $L_{0}$ reaches the critical value $\varepsilon$ ($\varepsilon >0$) at $t_{1}$, that is $L_{1}=\vert Y(t_{1})-Y_{0}(t_{1})\vert>\varepsilon$. Similarly, we keep $Y(t_{1})$ and find its nearest point $Y_{1}(t_{1})$. We set $L_{1}^{'}=\vert Y(t_{1})-Y_{1}(t_{1})\vert<\varepsilon$ and hope that the angle of $L_{1}$ and $L_{1}^{'}$ is as small as possible. Continue the above process until $Y(t)$ reaches the end of the time series $N$. The total number of iterations of this process to track the evolution is $M$. Then, the maximum Lyapunov exponent is :

\begin{equation}
      \lambda=\frac1{t_M-t_0}\sum_{i=1}^{M}\ln(\frac{L_{i}^{'}}{L_i-1}).
\end{equation}

If the maximum Lyapunov exponent is positive, we can get the Lyapunov time $h$:

\begin{equation}
      h = 1/\lambda
\end{equation}

Since the Lyapunov time is the limit of the predictability of the system, it can be regarded as the safe predication horizon.

\subsection{Feature Embedding with Phase Space Reconstruction}
\label{embedding}
After showing that the time series has chaotic behavior by the maximum Lyapunov exponent, we need to represent our feature embedding. The phase space reconstruction technology has two key parameters: the embedded dimension and the delay time. We assume that the time delay $\tau$ is not related to the reconstruction dimension $m$ during reconstruction. Since the current moment is usually related to the information at the previous moment, we choose the time delay $\tau=1$ to ensure the integrity of the information. The next step is to calculate its embedding dimension. 

Due to the fact that the chaotic sequence is actually the result of the trajectory generated by the chaotic system in the high-dimensional phase space being projected into one dimension, this will cause the motion trajectory to be distorted during the projection process. In order to make the information extracted from the low-dimensional data incomplete and accurate, we need to choose the reconstruction dimension to recover the chaotic motion trajectory from the chaotic time series. The usual method of determining the embedding dimension in practical applications is to calculate some geometric invariants of the attractor (such as correlation dimension, Lyapunov exponent, etc.). From the analysis of Takens embedding theorem, it can be seen that these geometric invariants have the geometric properties of attractors. When the dimension is greater than the minimum embedding dimension, the geometric structure has been fully opened. At this time, these geometric invariants have nothing to do with the embedding dimension. Based in this theory, the embedding dimension when the geometric invariant of the attractor stops changing can be selected as the reconstructed phase space dimension.

\subsubsection*{Cao's Method}  %不带序号

In terms of determining the embedding dimension of the phase space, Cao False Nearest Neighbour Method (Cao's  Method) \cite{Cao1997PracticalMF} is considered to be one of the effective methods to calculate the embedding dimension. There are three advantages of using Cao's method to calculate the embedding dimension: First, it is only necessary to know the delay time parameter in advance. Second, only small amount of data is required. Third, random signals and deterministic signals can both be effectively identified. 

Suppose the reconstructed data is:

\begin{equation}
       X(t_{i})=(x(t_{i}),x(t_{i+\tau}),\dots,x(t_{i+(m-1)\tau})), i=1,2,\cdots,N.
\end{equation}

Then the distance $R_{m(i)}$ between it and its nearest neighbor is

\begin{equation}
       R_{m}(i)=\left\|X_{m}(i)-X_{m}^{F N N}(i)\right\|_{2},
\end{equation}

where $ X_{m}^{F N N}(i) $ is the nearest neighbor of $X_{m}(i)$. 

Define

\begin{equation}
       a(i, m)=\frac{R_{m+1}(i)}{R_{m(i)}}=\frac{\left\|X_{m+1}(i)-X_{m+1}^{F N N}(i)\right\|}{\left\|X_{m}(i)-X_{m}^{F N N}(i)\right\|},
\end{equation}

when $X_{m+1}^{F N N}(i)$ and $X_{m}^{F N N}(i)$ are close, use the next adjacent point instead \cite{Lihua2011financialChaos}.

Next, calculate

\begin{equation}
      E(m)=\frac{1}{N-m \tau} \sum_{i=1}^{N-m \tau} a(i, m),
\end{equation}

\begin{equation}
      E_{1}(m)=E(m+1)/E(m).
\end{equation}

When the value of $m$ keeps increasing and is greater than a certain value, the amplitude of the change of $E_{1}(m)$ can be neglected, resulting in saturation. At this time, the corresponding $m+1$ value is the best embedding dimension of the system.

Finally, we have

\begin{equation}
     E^{*}(m)=\frac{1}{N-m \tau} \sum_{i=1}^{N-m \tau}\left|x_{i+m \tau}-x_{n(i, m)+m \tau}\right|,
\end{equation}

\begin{equation}
     E_{2}(m)=E^{*}(m+1) / E^{*}(m).
\end{equation}

If for any $m$, $E_{2}(m)=1$, it is a random time series, that is, each value of data is independent. If there is always some $m$, such that $E_{2}(m) \neq 1$, it is a deterministic time series, that is, the relationship between data points depends on the change of the embedding dimension $m$.

After calculating the maximum Lyapunov exponent, $E_{1}(m)$ and $E_{2}(m)$, we can identify the chaos and predictability of the sequence, and determine the reconstruction dimension. In this way, the dimension of the time series is restored to the original high dimension so that its potentially valuable features are extracted. At the same time, we also get the best input feature embedding for our model.

\section{Appendix B }
  \subsection{Full Results of \ref{exp1} of the main paper}
  \begin{figure}[htbp]
    \centering    
    \subfigure {
        \label{fig:SSENYt+1-1.5}     
        \includegraphics[width=4.5cm,height=3.5cm]{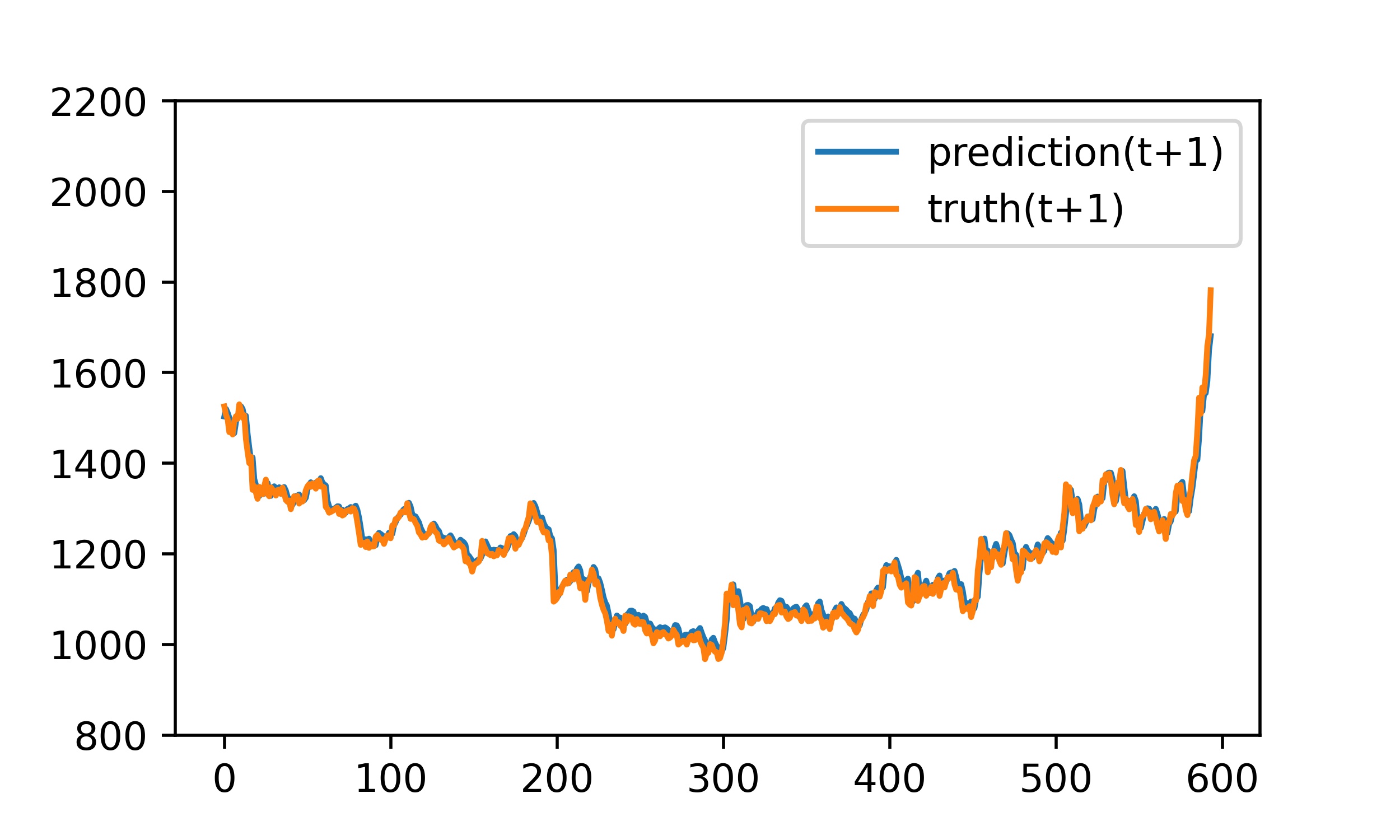}  
    }     
    \subfigure{ 
        \label{fig:SSEt+1-1.5}     
        \includegraphics[width=4.5cm,height=3.5cm]{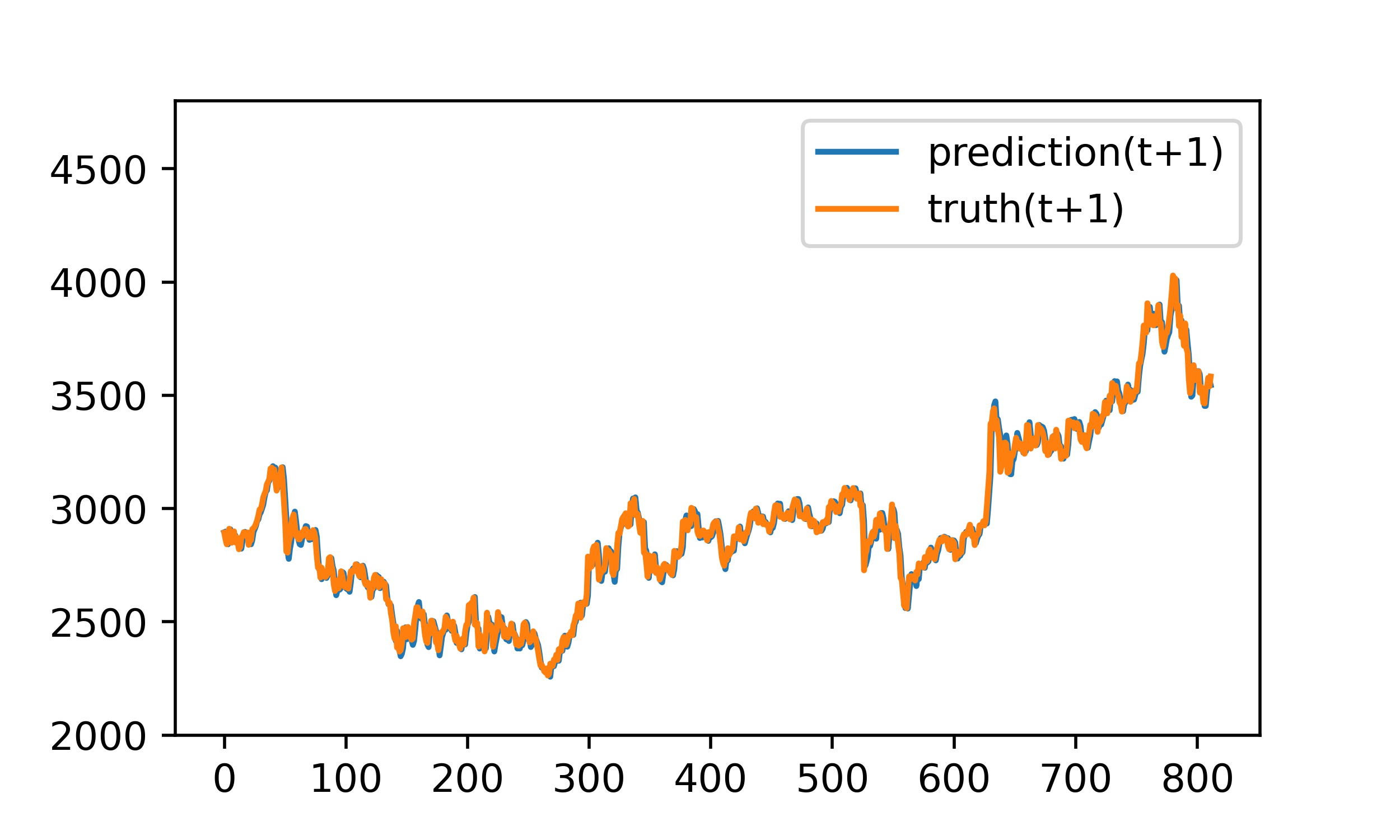}     
    }    
    \subfigure { 
        \label{fig:SZXFt+1-1.5}     
        \includegraphics[width=4.5cm,height=3.5cm]{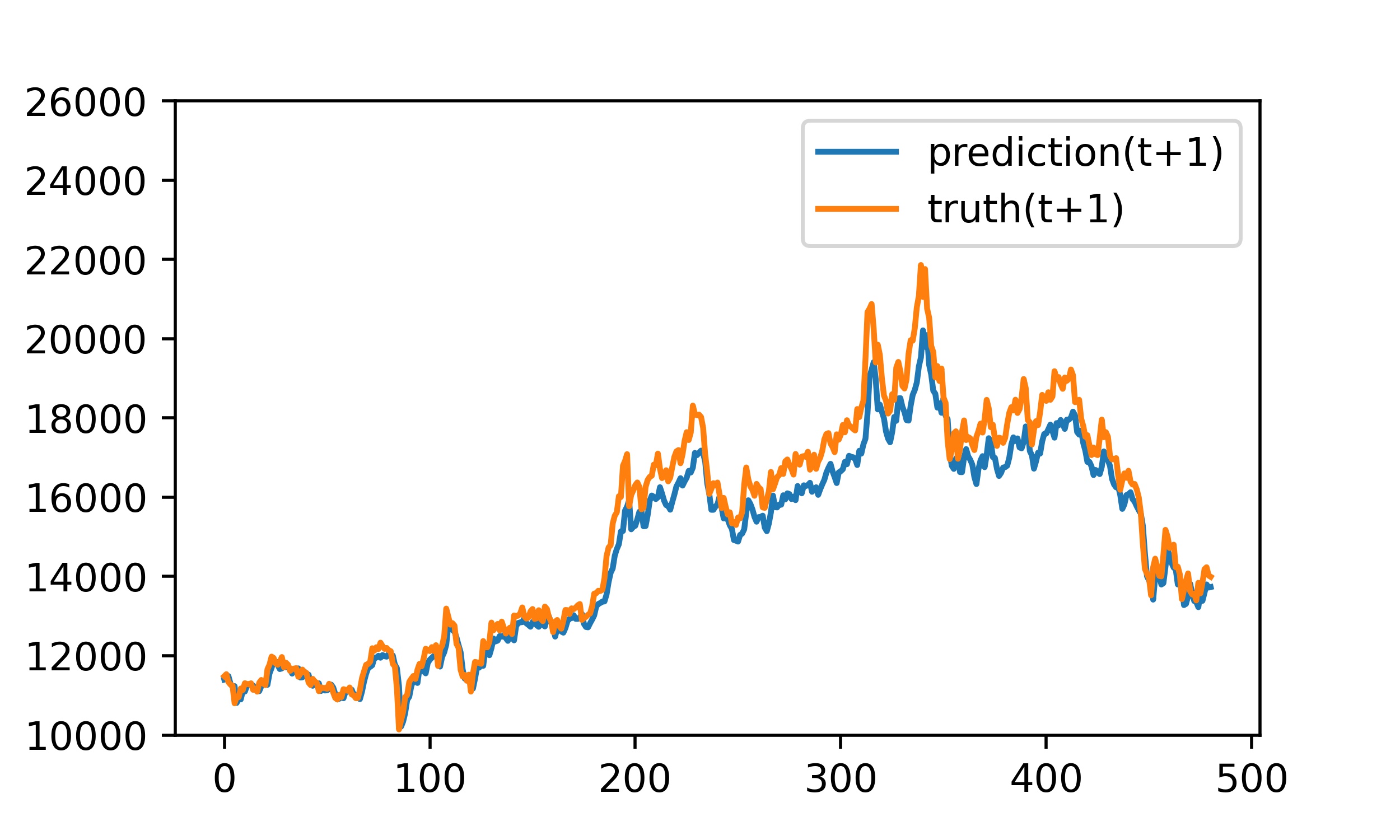}     
    }
     \subfigure {
        \label{fig:SSENYt+2-1.5}     
        \includegraphics[width=4.5cm,height=3.5cm]{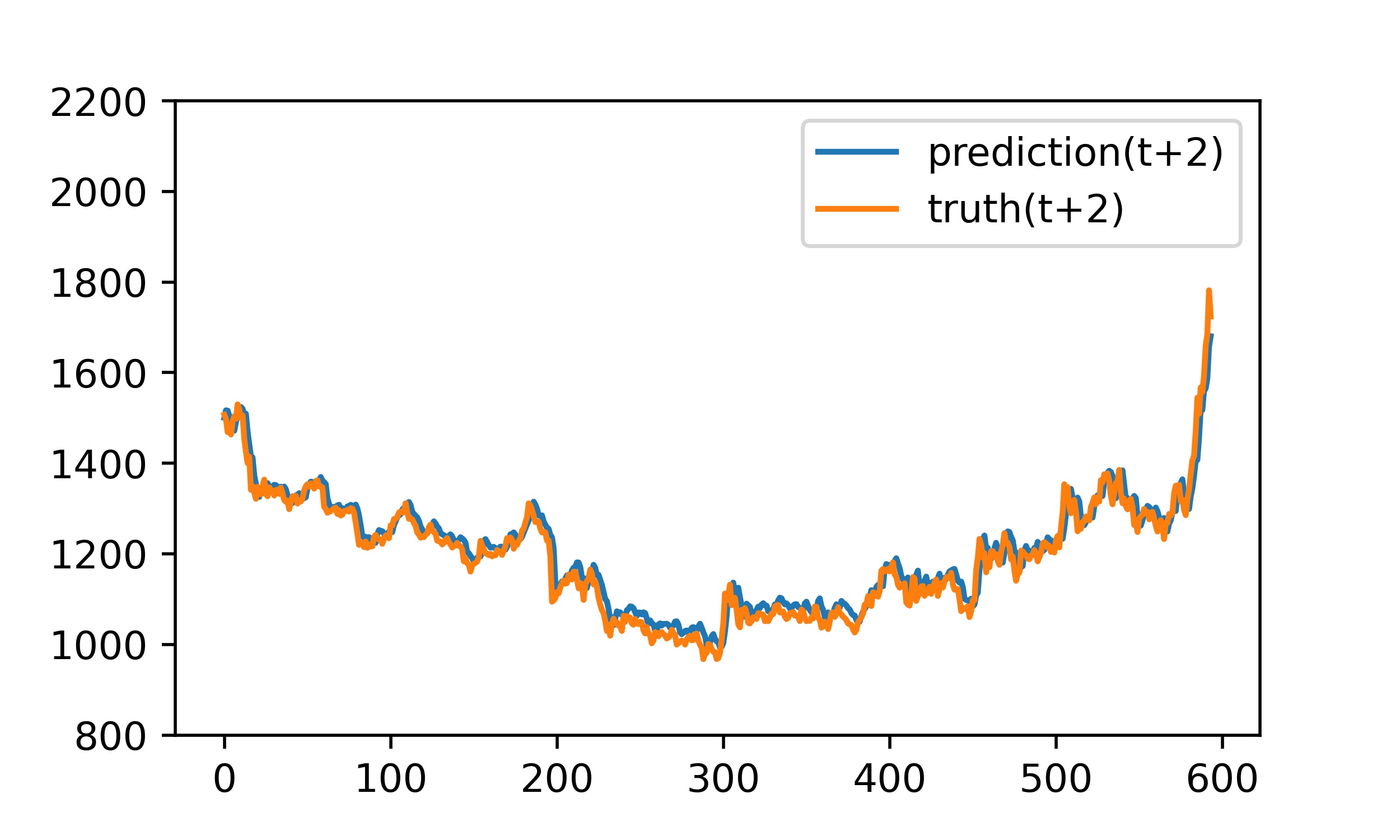}  
    }     
    \subfigure{ 
        \label{fig:SSEt+2-1.5}     
        \includegraphics[width=4.5cm,height=3.5cm]{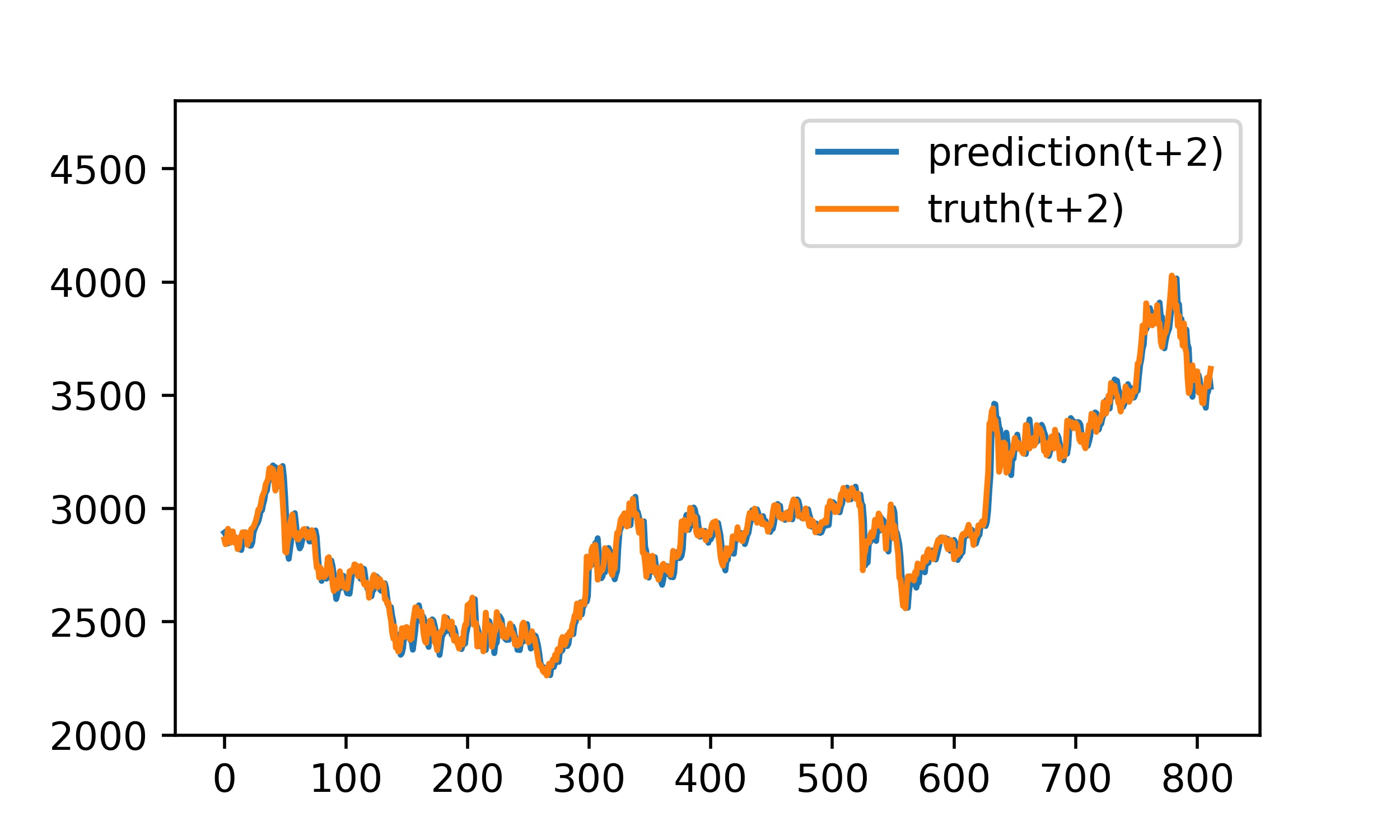}     
    }    
    \subfigure { 
        \label{fig:SZXFt+2-1.5}     
        \includegraphics[width=4.5cm,height=3.5cm]{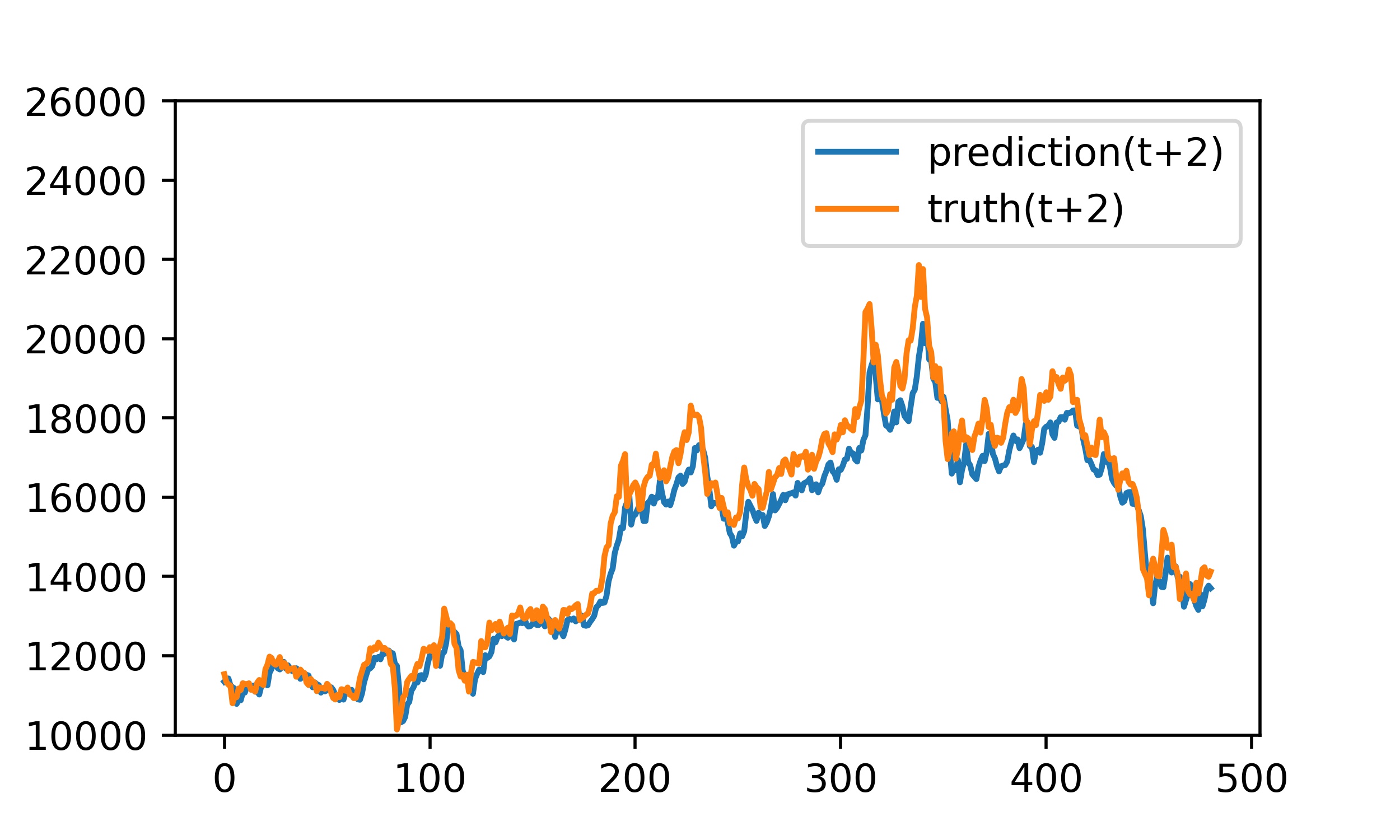}     
    }   
    \subfigure {
        \label{fig:SSENYt+3-1.5}     
        \includegraphics[width=4.5cm,height=3.5cm]{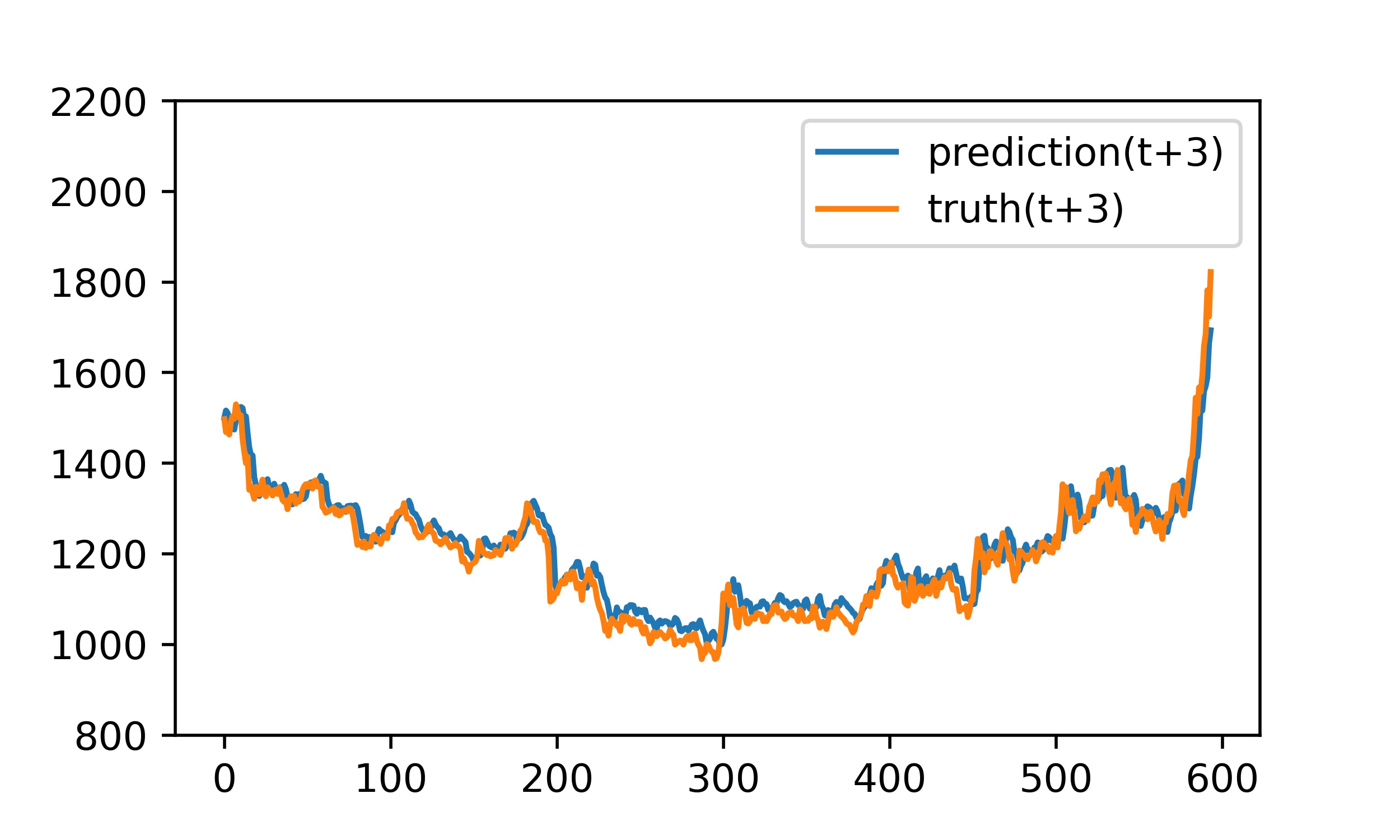}  
    }     
    \subfigure{ 
        \label{fig:SSEt+3-1.5}     
        \includegraphics[width=4.5cm,height=3.5cm]{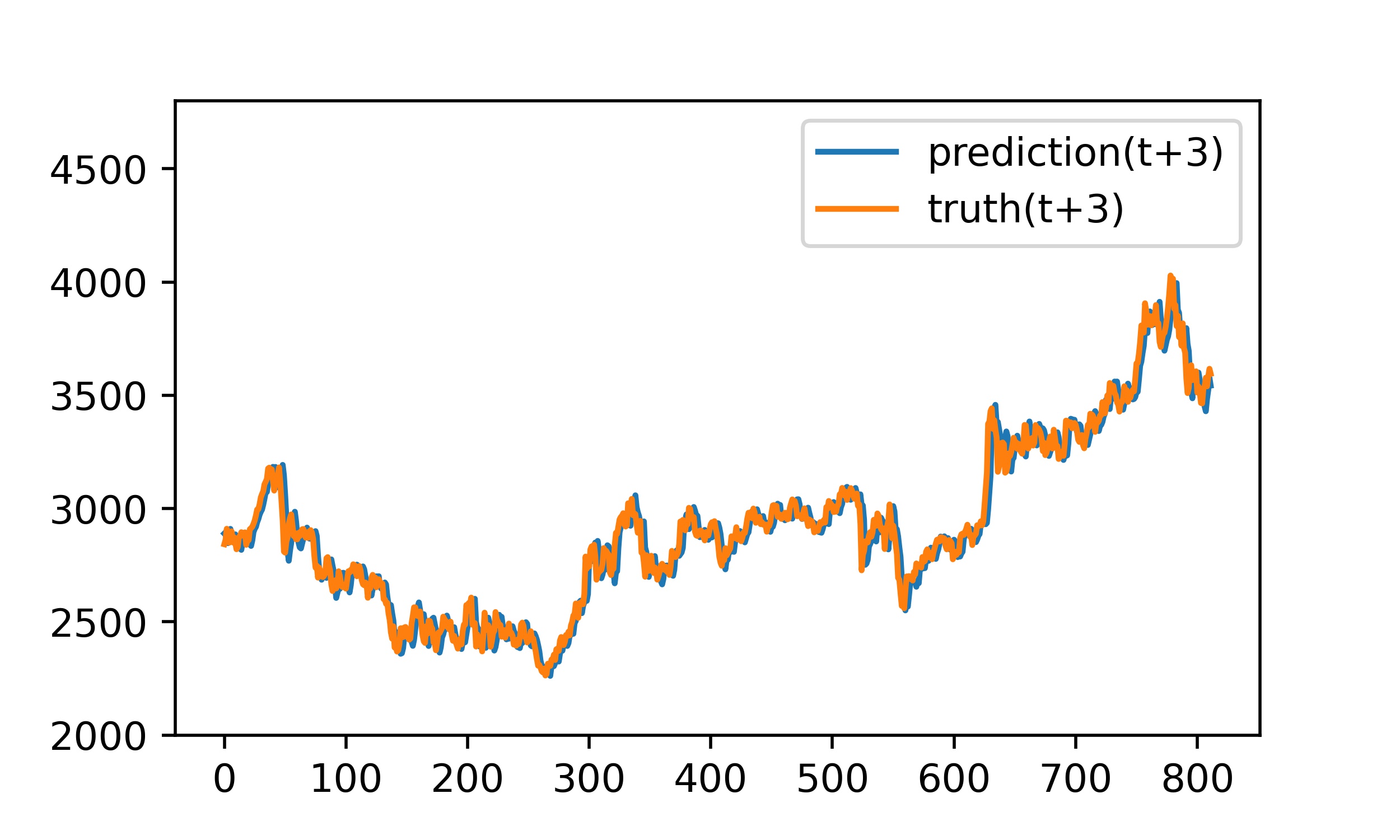}     
    }    
    \subfigure { 
        \label{fig:SZXFt+3-1.5}     
        \includegraphics[width=4.5cm,height=3.5cm]{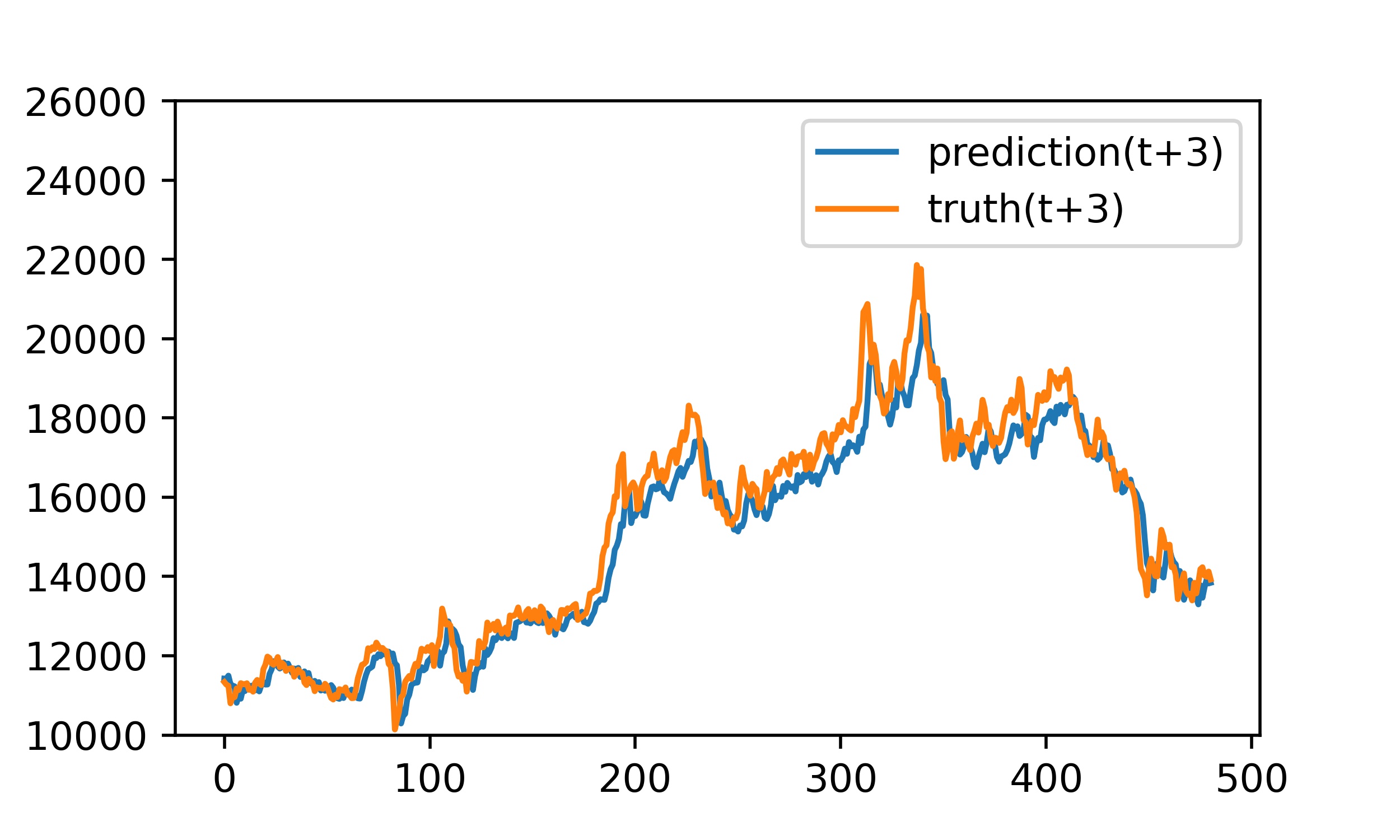}     
    }   
    \caption{ The prediction effect of different stocks with $\alpha=1.5$. (left:SSE Energy Index, middle:SSE50 Index, right:SSE Consumer Index) }     
    \label{fig:B1}     
    \end{figure}
    
    \newpage
    \subsection{Full Results of \ref{exp2} of the main paper:SSE 50 Index}
    \begin{figure}[htbp]
        \centering
        \begin{minipage}[t]{0.45\textwidth}
            \centering \includegraphics[width=5cm]{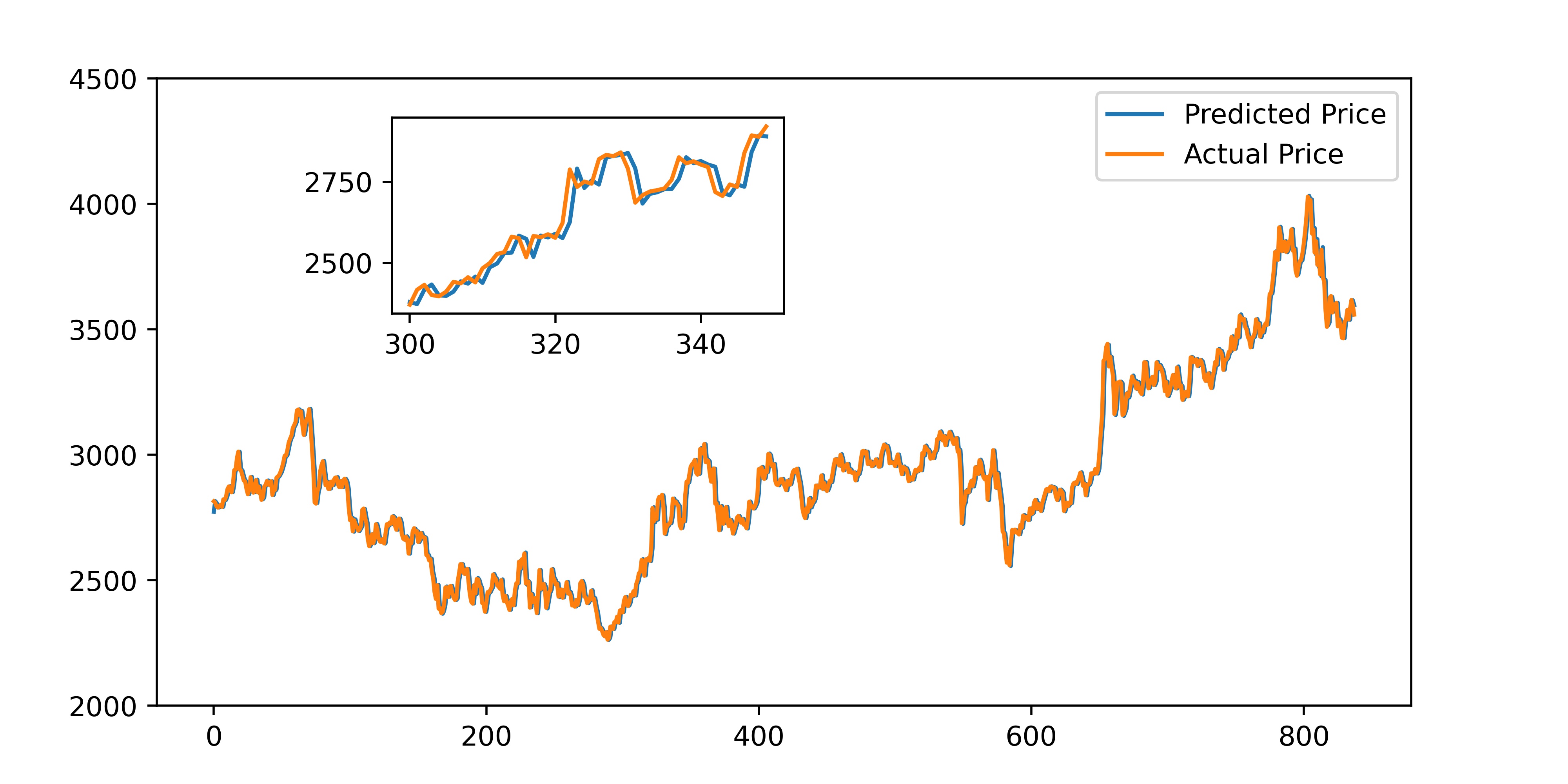}
            \label{fig:SZ50ARIMA}
        \end{minipage}
        \begin{minipage}[t]{0.45\textwidth}
            \centering         \includegraphics[width=5cm]{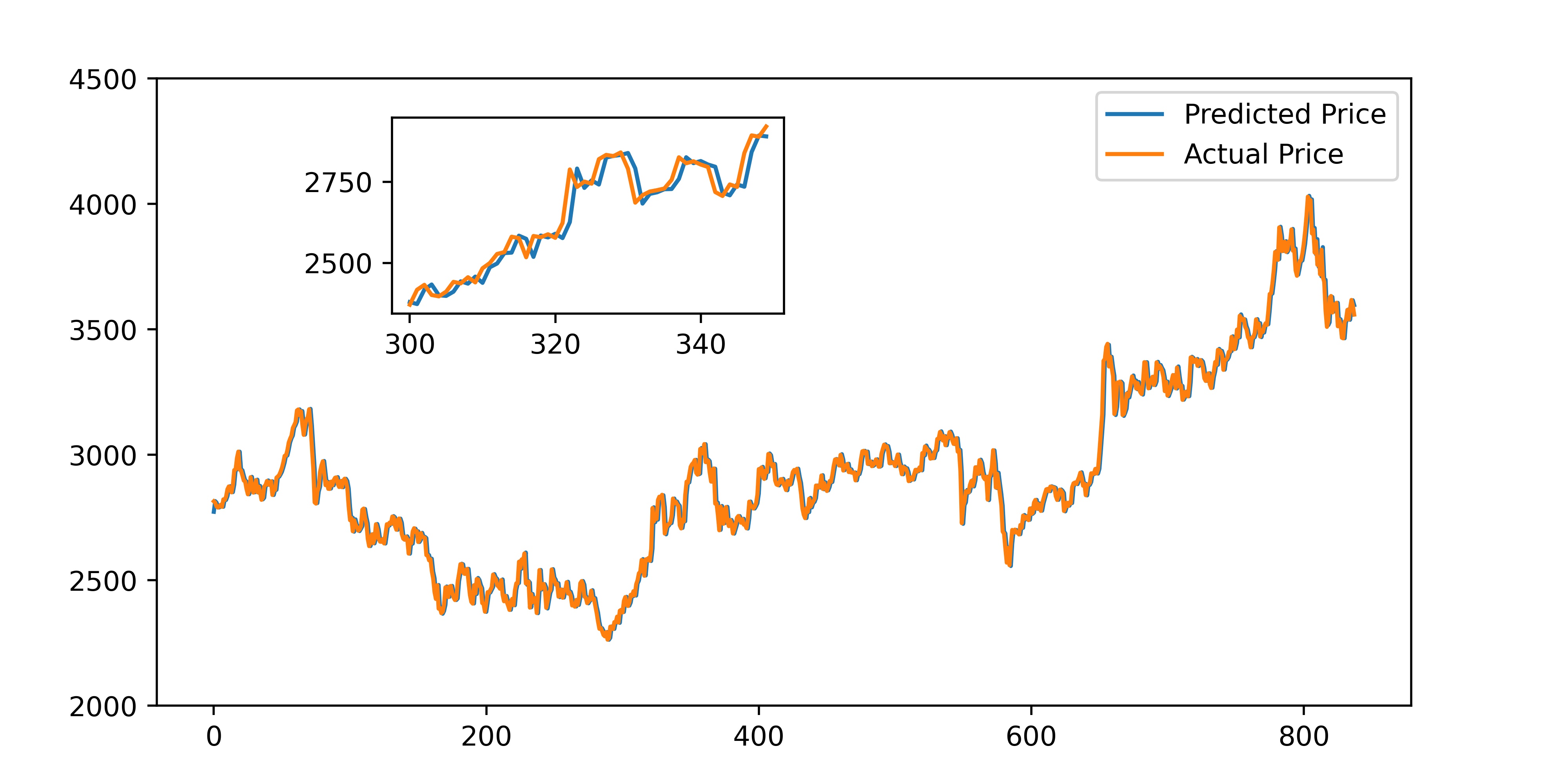}
            \label{fig:SZ50LSTM}
        \end{minipage}
        \begin{minipage}[t]{0.45\textwidth}
            \centering         \includegraphics[width=5cm]{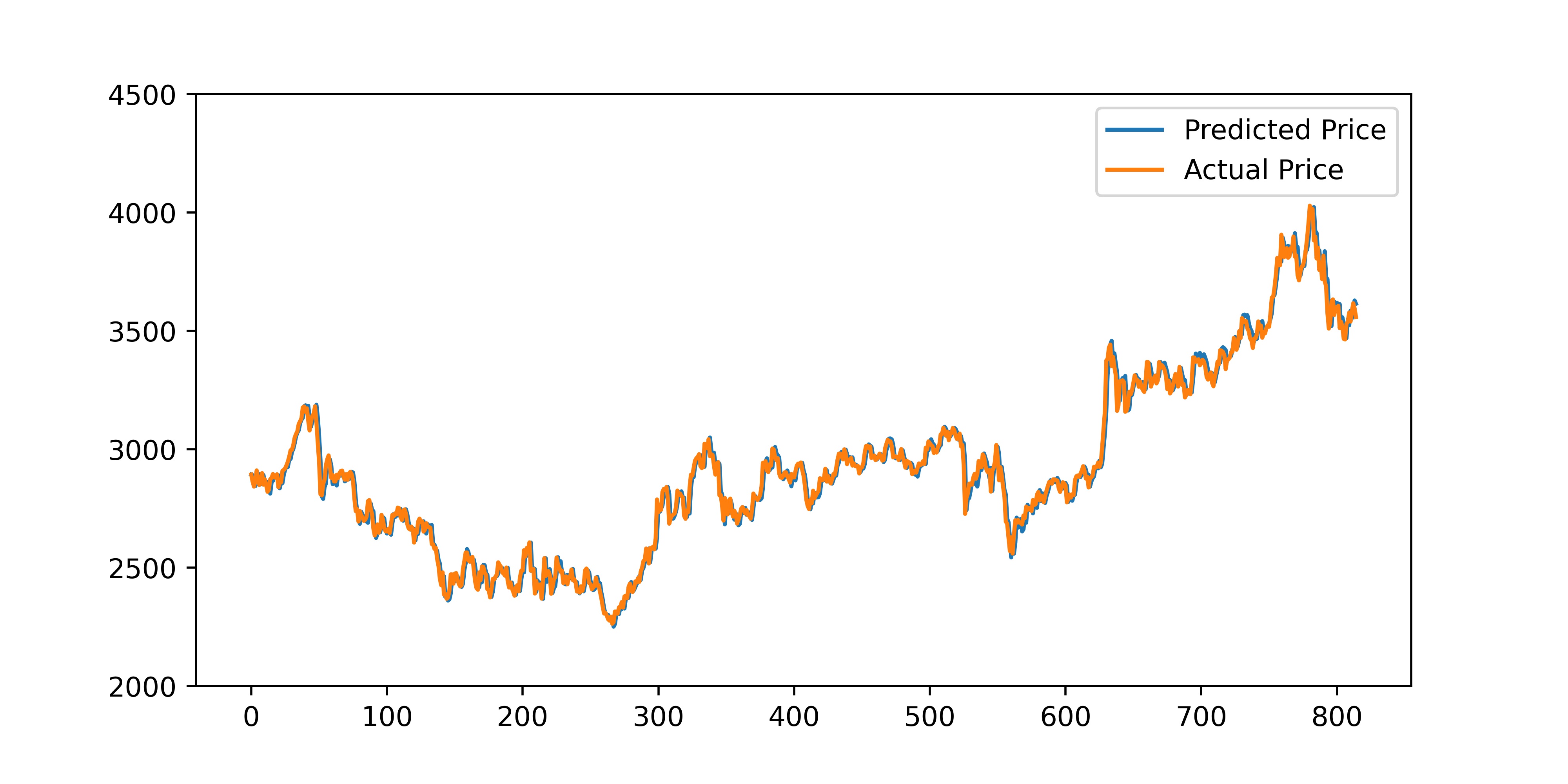}
            \label{fig:SZ50OLD}
        \end{minipage}
        \begin{minipage}[t]{0.45\textwidth}
            \centering           \includegraphics[width=5cm]{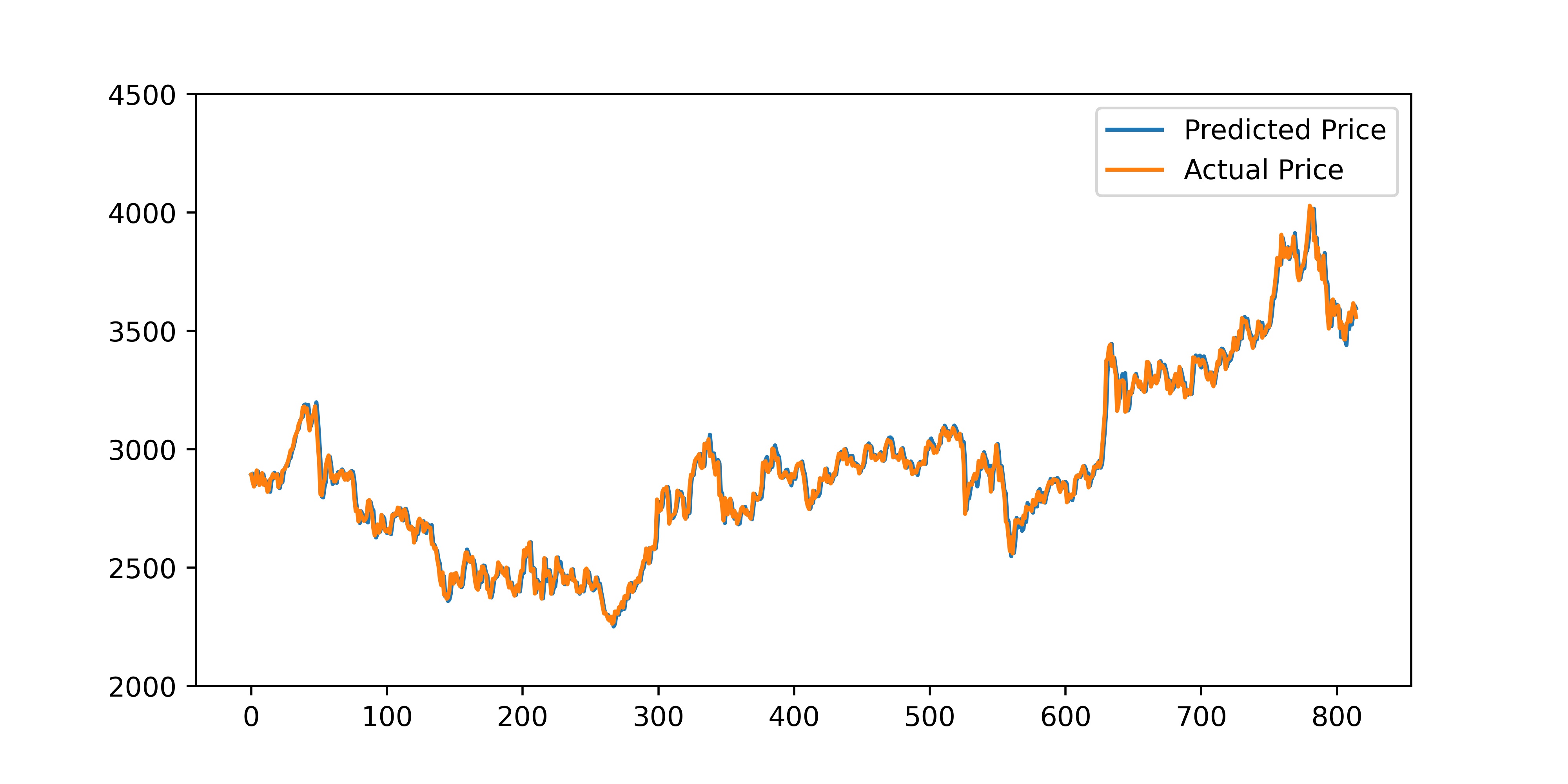}
            \label{fig:SZ50LDE-Net}
        \end{minipage}
        \caption{For SSE 50 Index, the predictive effects of different models. (The upper left is ARIMA, the upper right is LSTM, the lower left is SDE-Net, the lower right is LDE-Net.)}
    \end{figure}
    
    \subsection{Experiment : Safe prediction horizon}
    \label{Safe prediction horizon}
    \begin{table}[H]
    	\caption{Forecast results of SSE Energy Index with different horizon ($\alpha=1.5$)}
    	\centering
    	\setlength{\textwidth}{15mm}{
    	\begin{tabular}{cccccc}
    		\toprule
    		  Embedding Dimension  & MSE(t+1) & MSE(t+2)& MSE(t+3)& MSE(t+4) \\
    		\midrule
            20 &0.0010&0.0022&0.0034&0.0045  
            \\
            43 &0.0013&0.0032&0.0058&0.0075
            \\
            50 &0.0018&0.0026& 0.0048&0.0054
            \\
    	\bottomrule
    	\end{tabular}}
    	\label{tab: SSEenergyhorizon}
    \end{table}  
        
    \begin{table}[H]
    	\caption{Forecast results of SSE50 Index with different horizon ($\alpha=1.5$)}
    	\centering
    	\setlength{\textwidth}{15mm}{
    	\begin{tabular}{cccccc}
    		\toprule
    		 Embedding Dimension  & MSE(t+1) & MSE(t+2)& MSE(t+3)& MSE(t+4) \\
    		\midrule
            23 &0.0029&0.0058&0.0091&0.0125   
            \\
            41 &0.0037&0.0070&0.0100&0.0137
            \\
            50 &0.0036&0.0066&0.0102&0.0138
            \\
    	\bottomrule
    	\end{tabular}}
    	\label{tab: SSE50horizon}
    \end{table}

    \subsection{Experiment : Forecasting with different steps}
    \label{different steps}
    \begin{table}[H]
	\caption{Forecast results of different steps ($\alpha=1.5$)}
	\centering
	\setlength{\textwidth}{15mm}{
	\begin{tabular}{ccccccc}
		\toprule
		Data  & MSE(t+5) & MSE(t+6)& MSE(t+7)& MSE(t+8) & MSE(t+9) & MSE(t+10) \\
		\midrule
        SSE Energy Index &0.0068 &0.0071 &0.0109 &0.0140 &0.0155 &0.0168
        \\
        SSE50 Index &0.0158&0.0188&0.0211&0.0236&0.0257&0.0285
        \\
	\bottomrule
	\end{tabular}}
	\label{tab:stock10step}
    \end{table}
    
    \begin{figure}[htbp]
        \centering    
        \includegraphics[width=5cm,height=4cm]{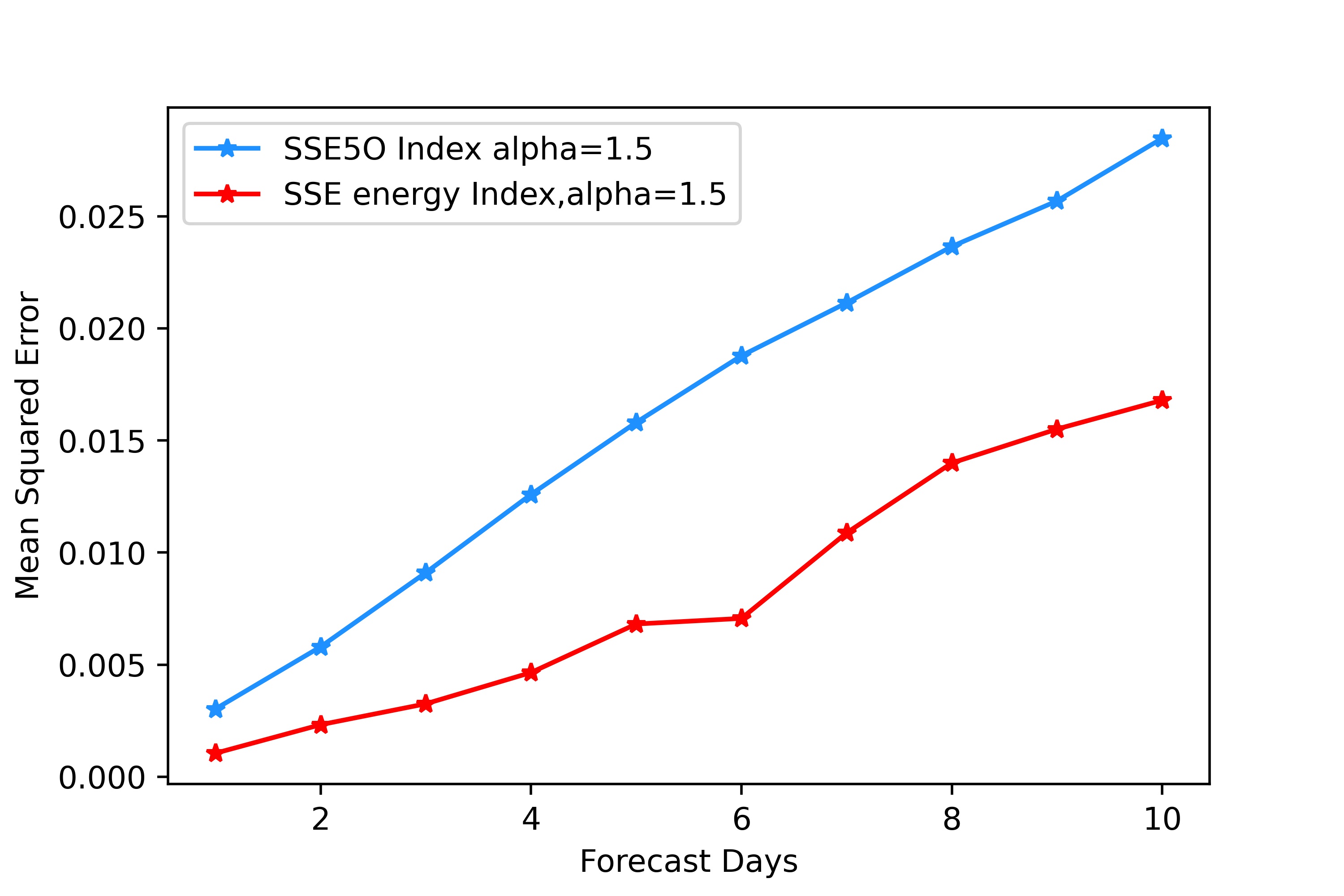}  
        \caption{The prediction loss of SSE Energy Index and SSE50 Index with different steps. ($\alpha=1.5$)}   
        \label{fig:B3}     
    \end{figure}

    \begin{figure}[htbp]
    \centering    
    \subfigure {
        \label{fig:SSEt+5-1.5}     
        \includegraphics[width=4.5cm,height=3.5cm]{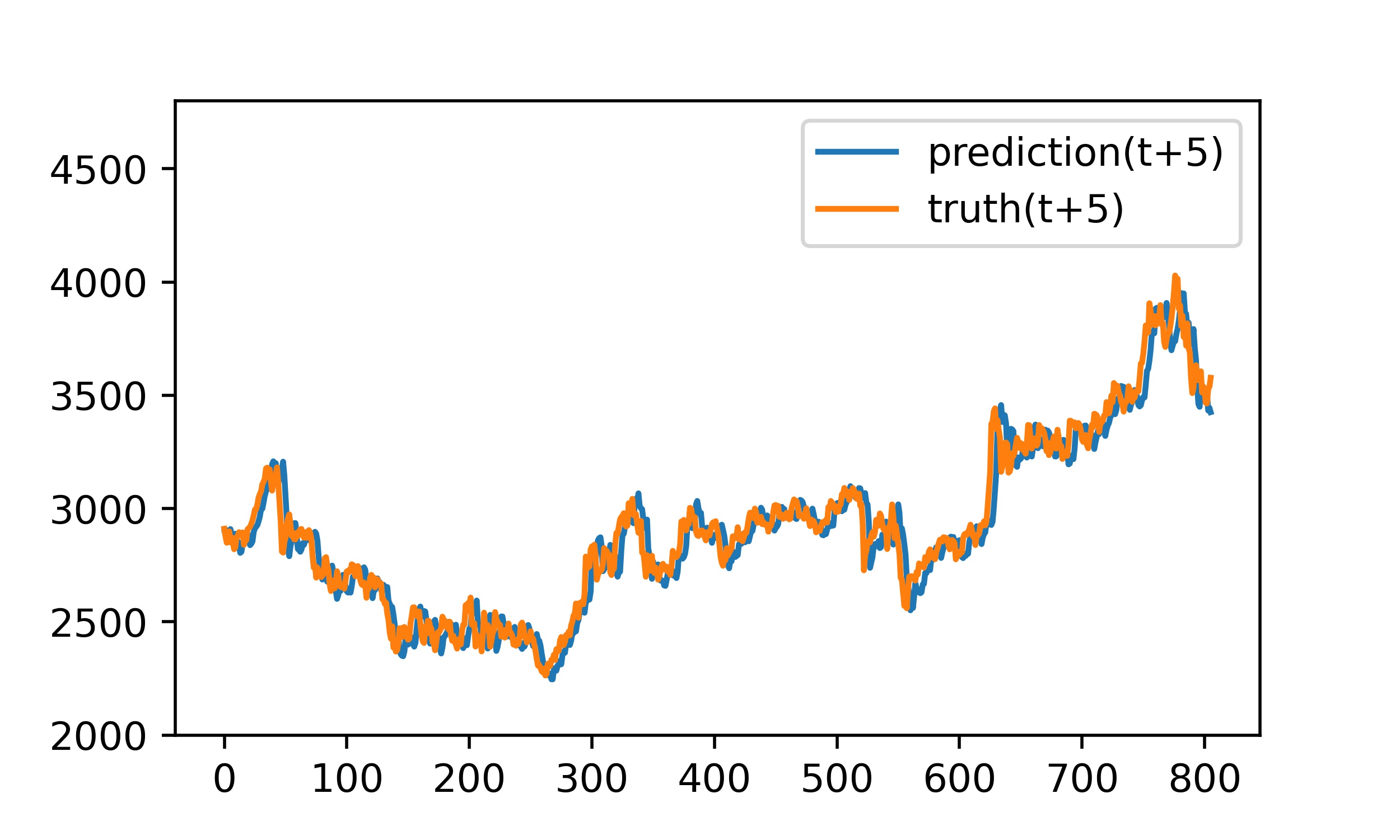}  
    }     
    \subfigure{ 
        \label{fig:SSEt+6-1.5}     
        \includegraphics[width=4.5cm,height=3.5cm]{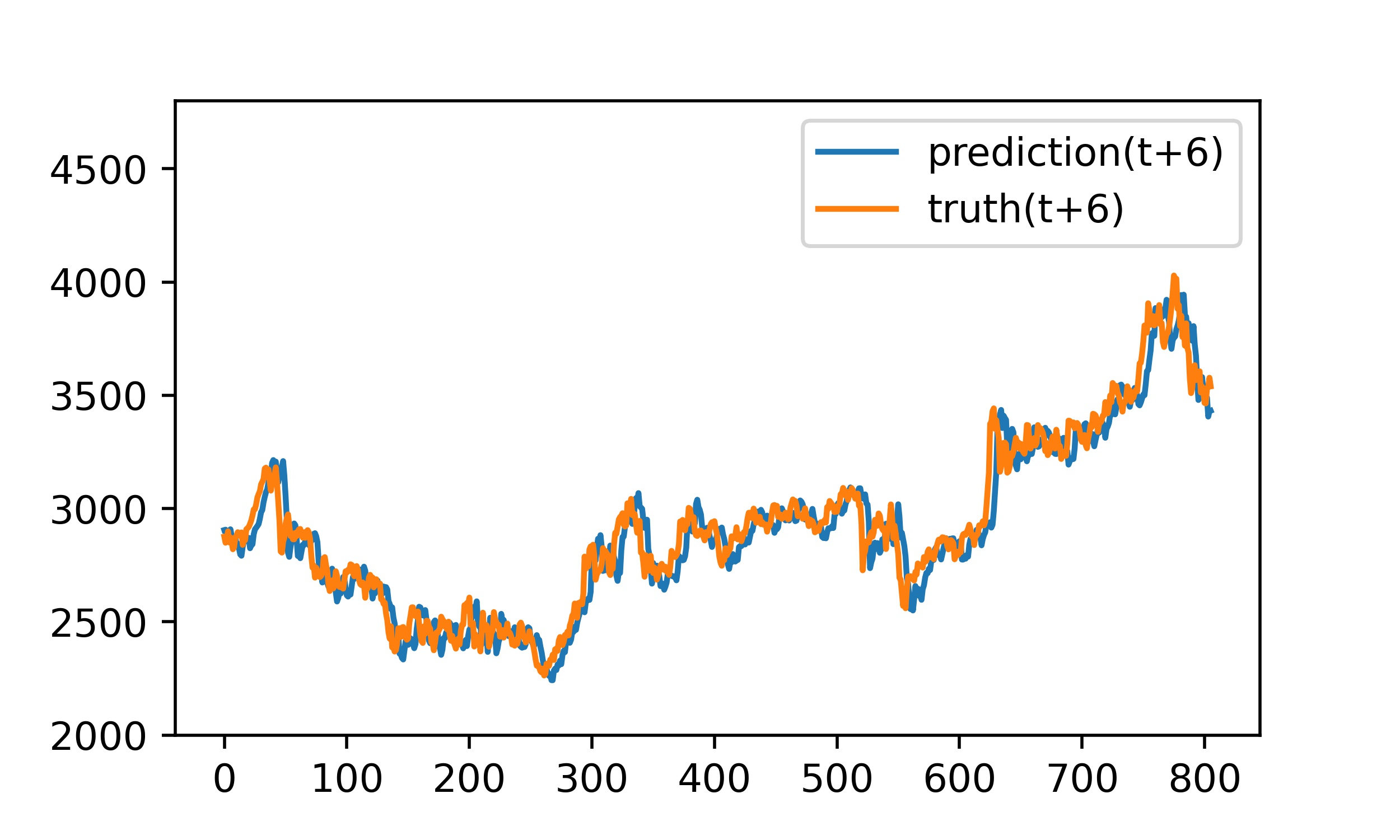}     
    }    
    \subfigure { 
        \label{fig:SSEt+7-1.5}     
        \includegraphics[width=4.5cm,height=3.5cm]{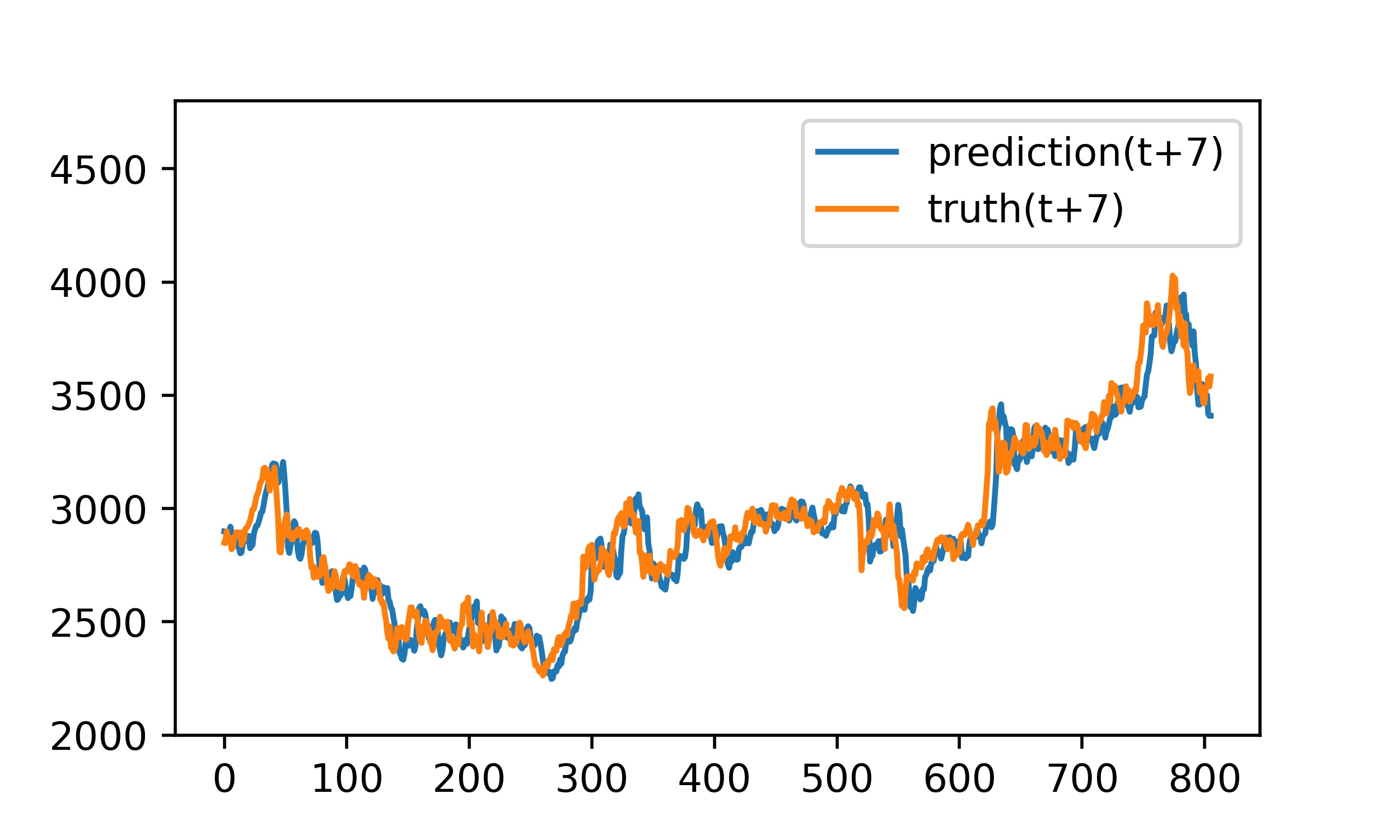}     
    }
     \subfigure {
        \label{fig:SSEt+8-1.5}     
        \includegraphics[width=4.5cm,height=3.5cm]{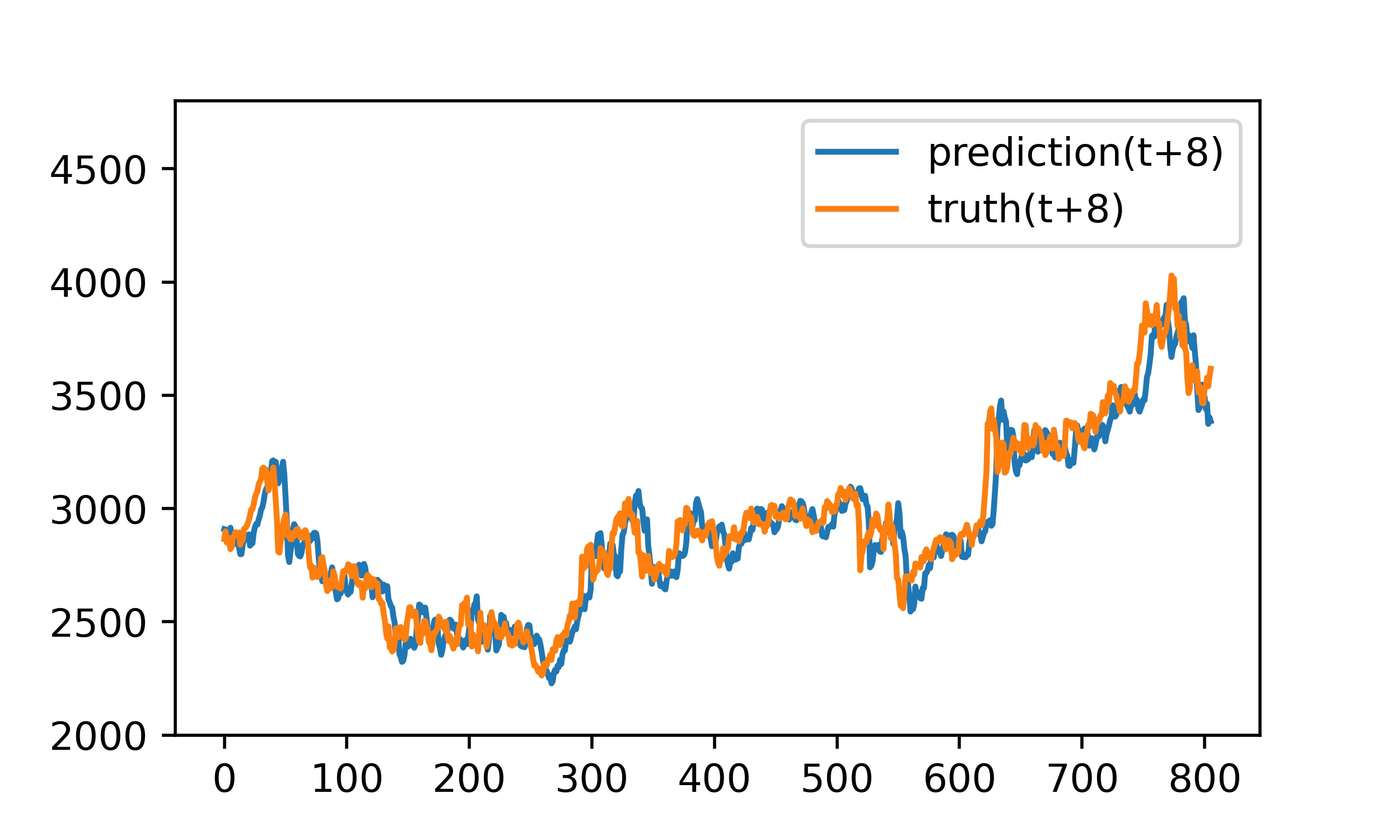}  
    }     
    \subfigure{ 
        \label{fig:SSEt+9-1.5}     
        \includegraphics[width=4.5cm,height=3.5cm]{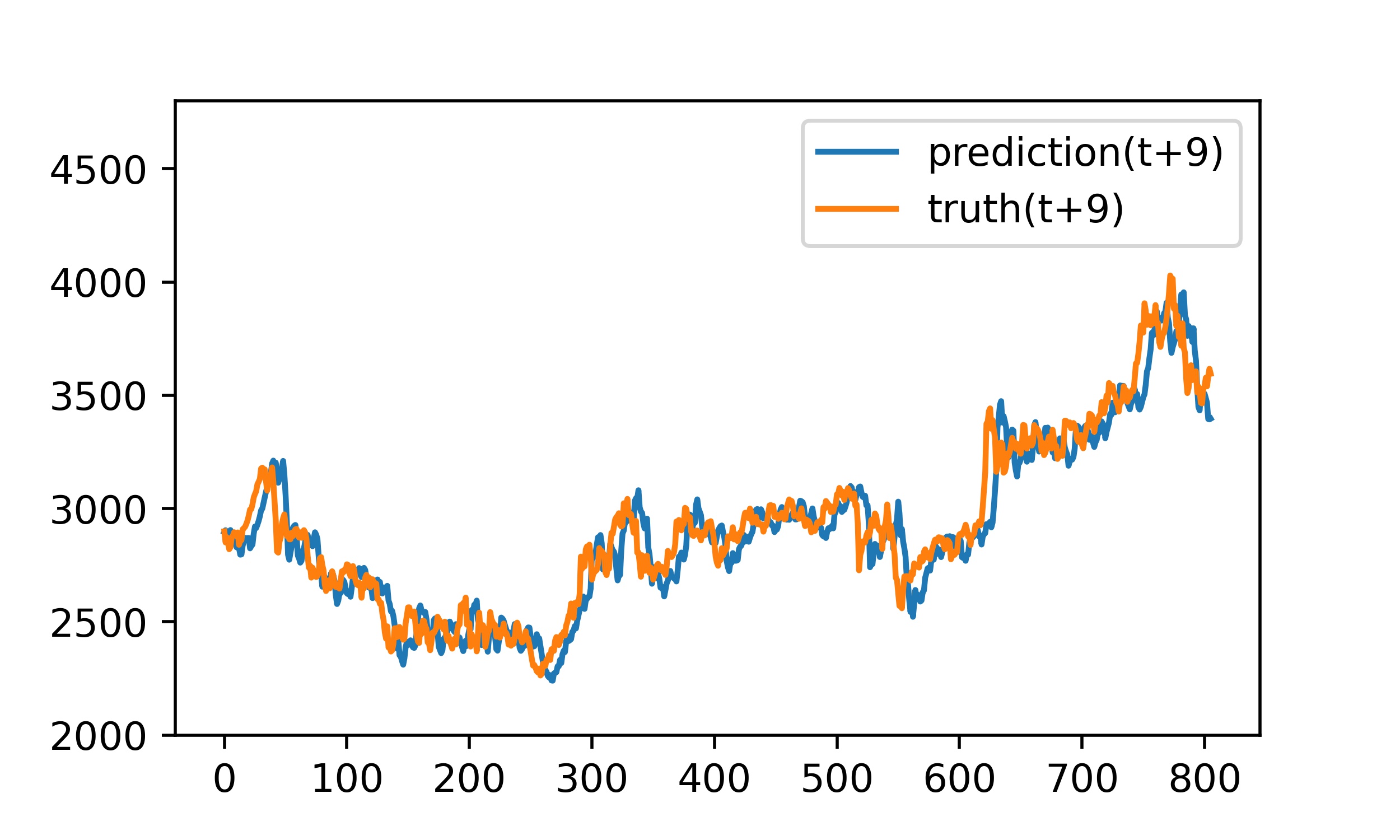}     
    }    
    \subfigure { 
        \label{fig:SSEt+10-1.5}     
        \includegraphics[width=4.5cm,height=3.5cm]{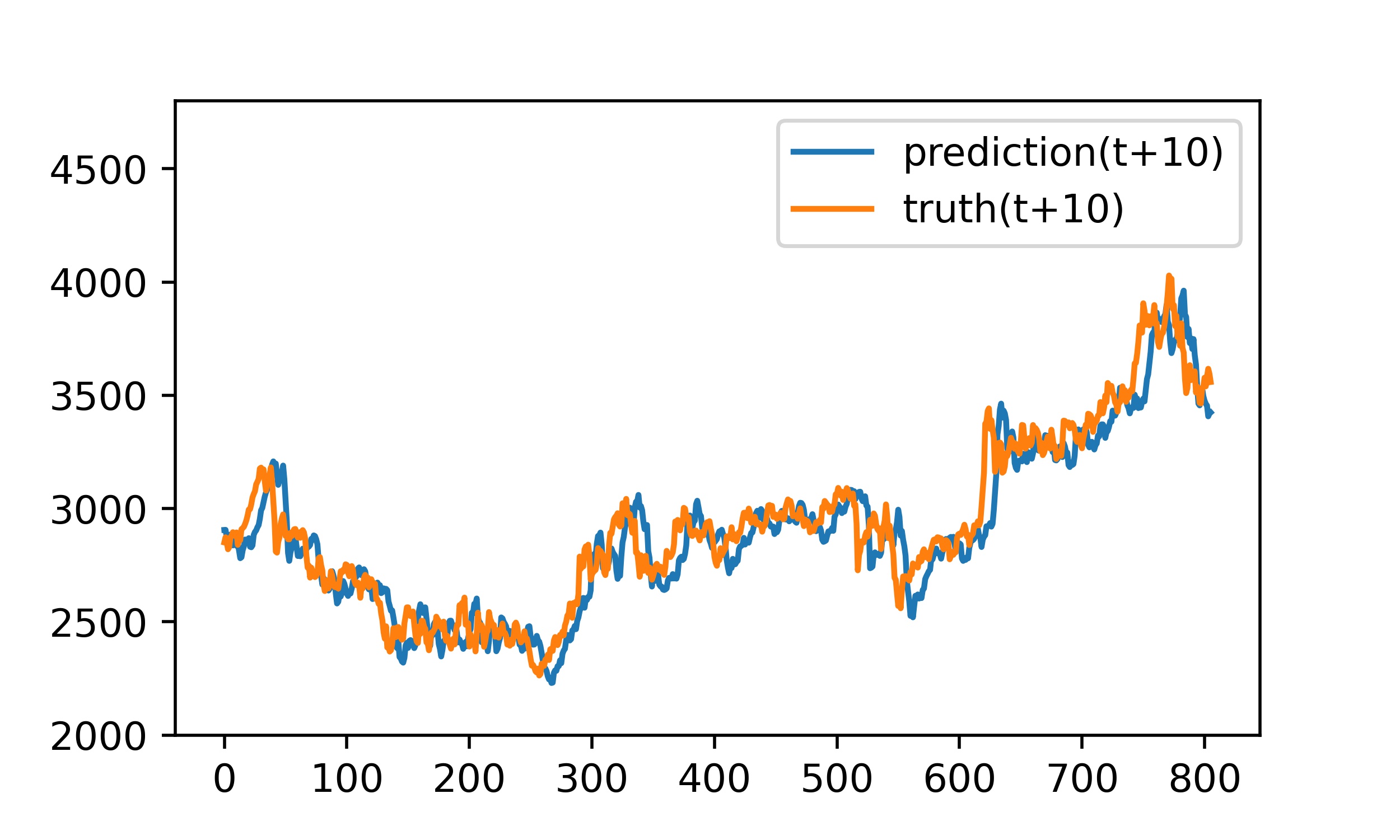}     
    }   
    
   \caption{The prediction effect of SSE50 Index with different steps. ($\alpha=1.5$)}     
    \label{fig:B4}     
    \end{figure}
    
    \begin{figure}[htbp]
    \centering    
    \subfigure {
        \label{fig:SENYt+5-1.5}     
        \includegraphics[width=4.5cm,height=3.5cm]{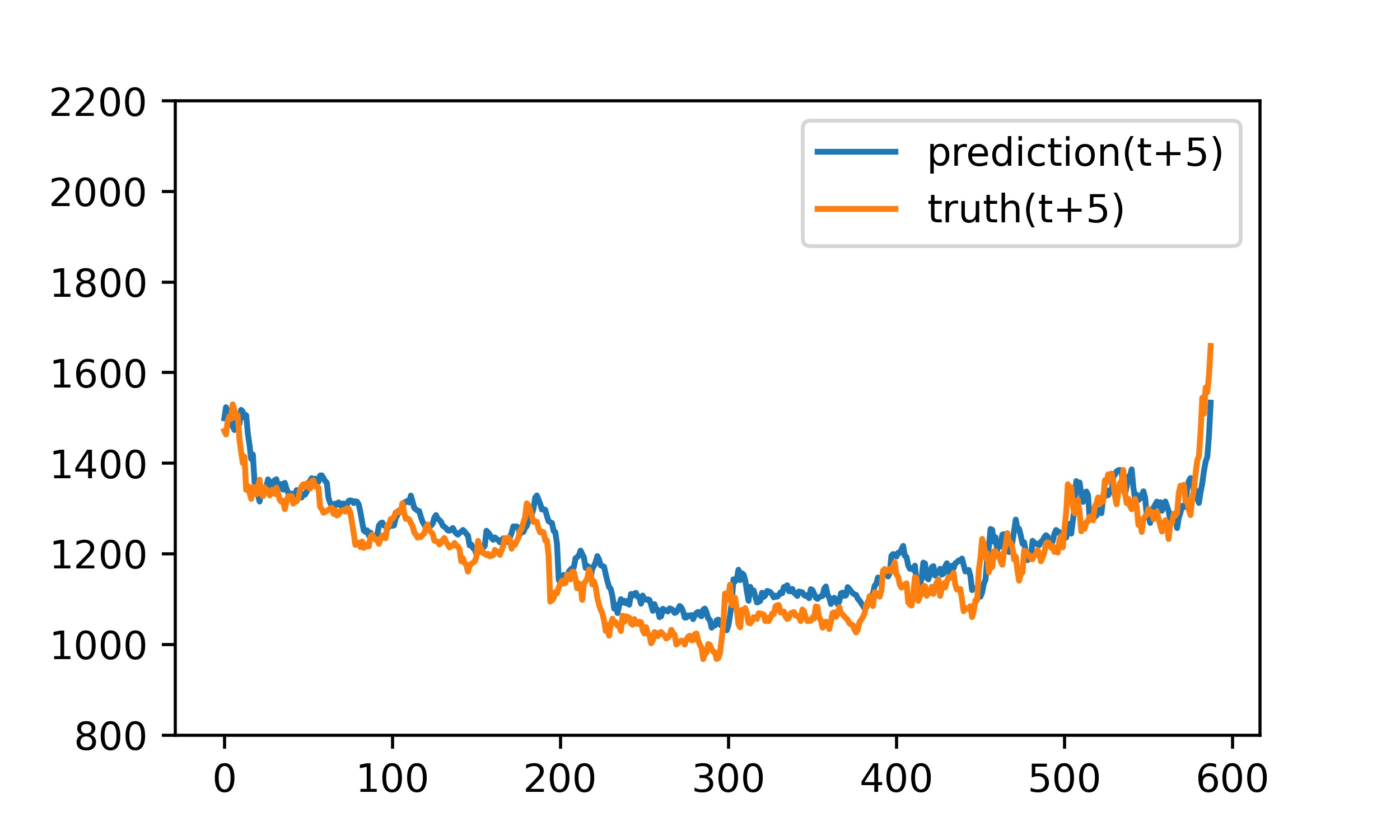}  
    }     
    \subfigure{ 
        \label{fig:SENYt+6-1.5}     
        \includegraphics[width=4.5cm,height=3.5cm]{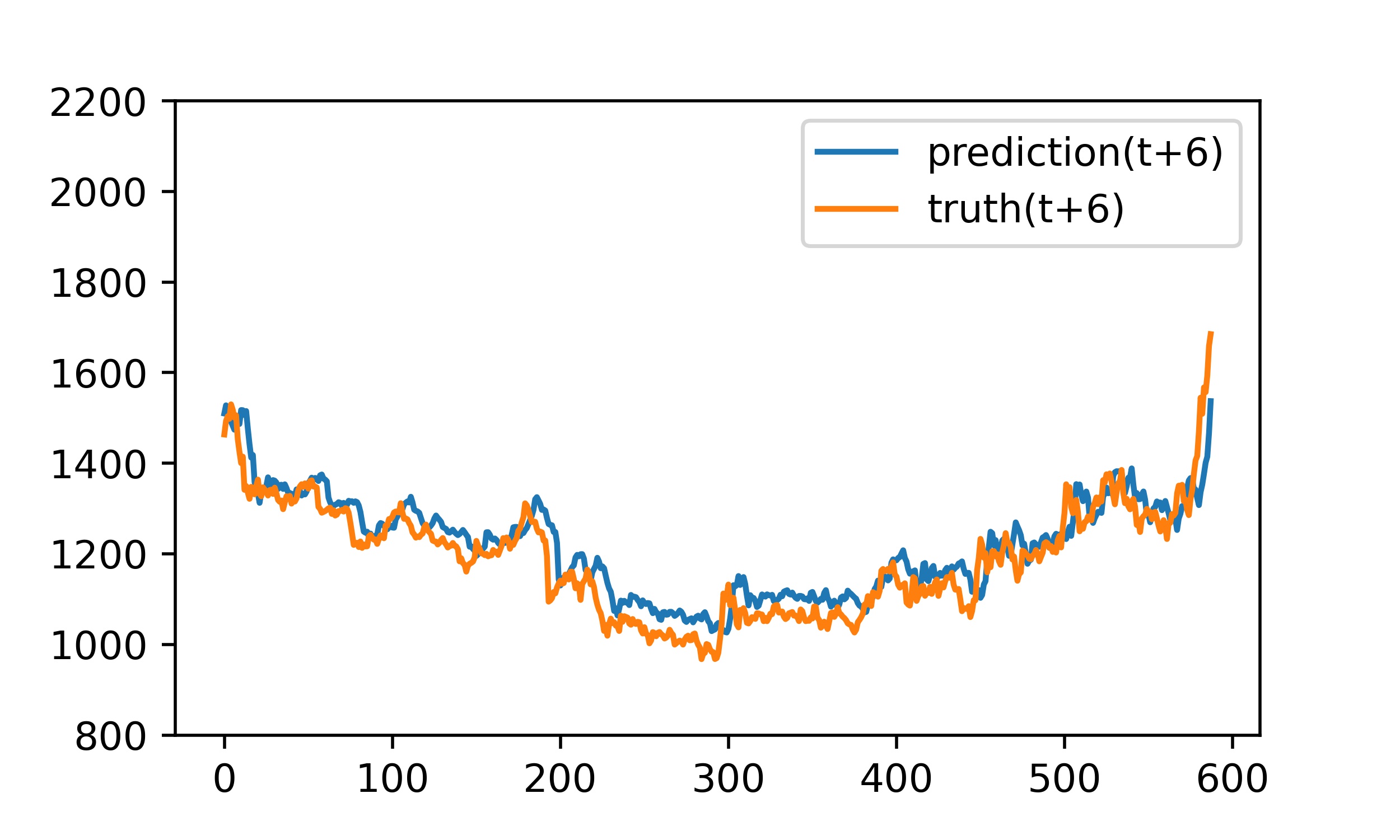}     
    }    
    \subfigure { 
        \label{fig:SENYt+7-1.5}     
        \includegraphics[width=4.5cm,height=3.5cm]{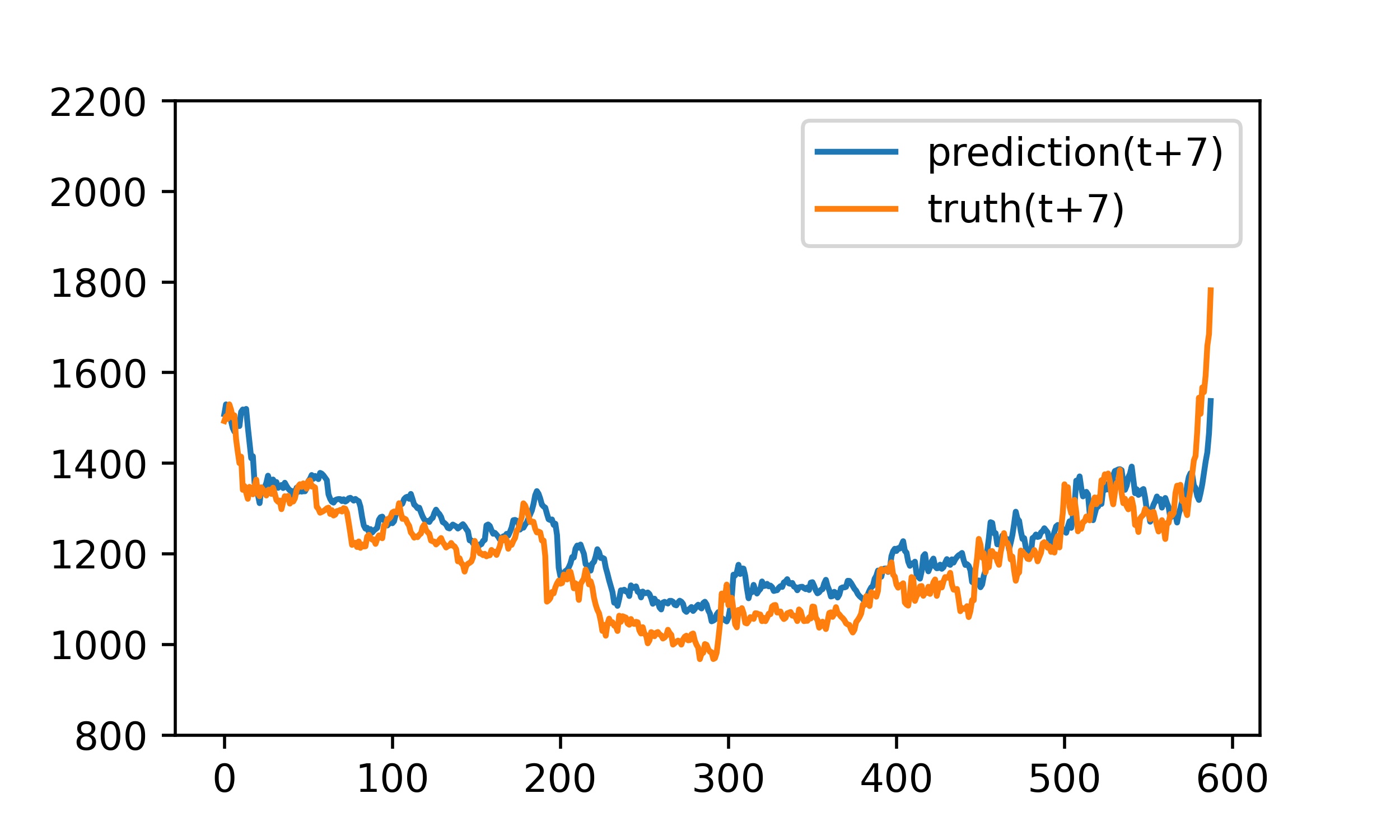}     
    }
     \subfigure {
        \label{fig:SENYt+8-1.5}     
        \includegraphics[width=4.5cm,height=3.5cm]{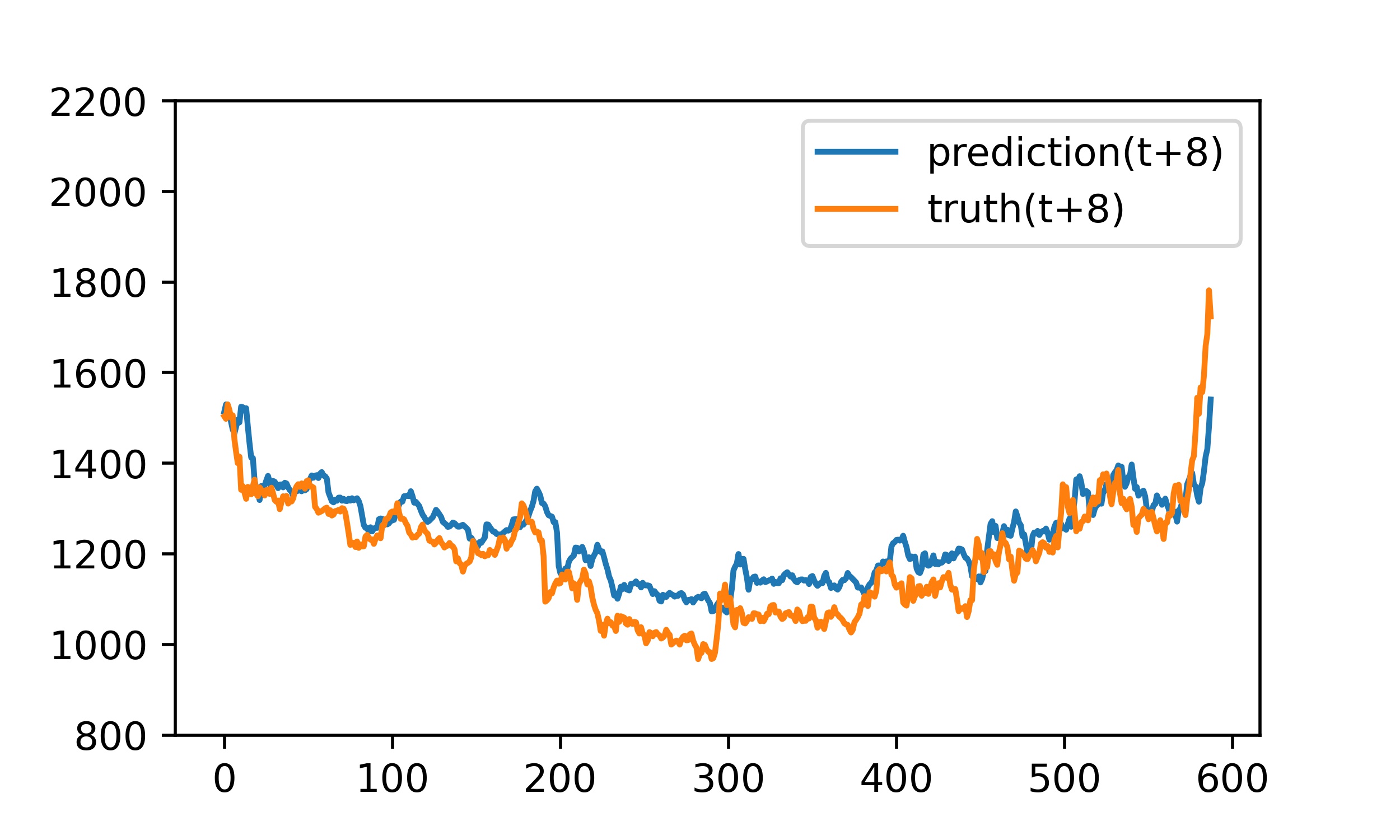}  
    }     
    \subfigure{ 
        \label{fig:SENYt+9-1.5}     
        \includegraphics[width=4.5cm,height=3.5cm]{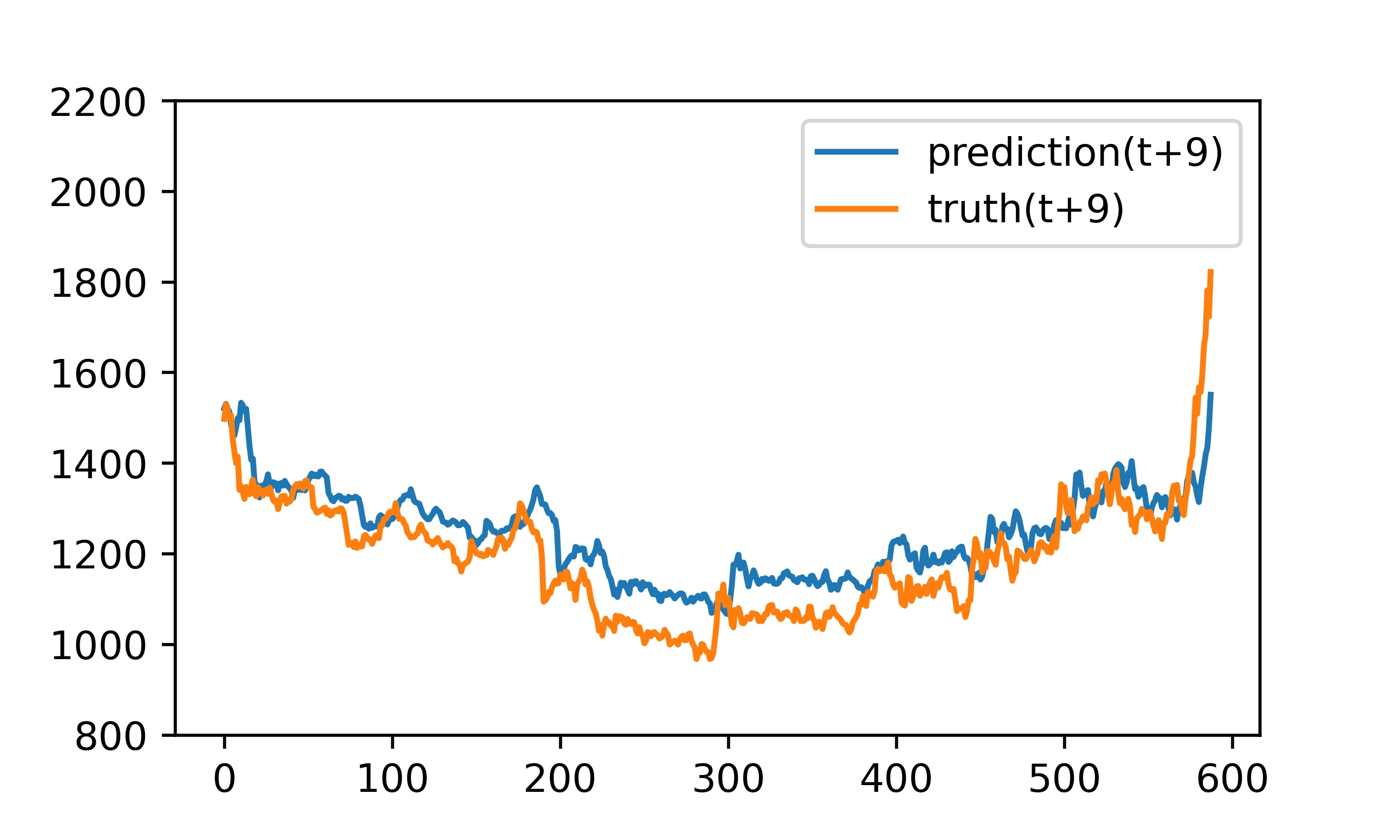}     
    }    
    \subfigure { 
        \label{fig:SENYt+10-1.5}     
        \includegraphics[width=4.5cm,height=3.5cm]{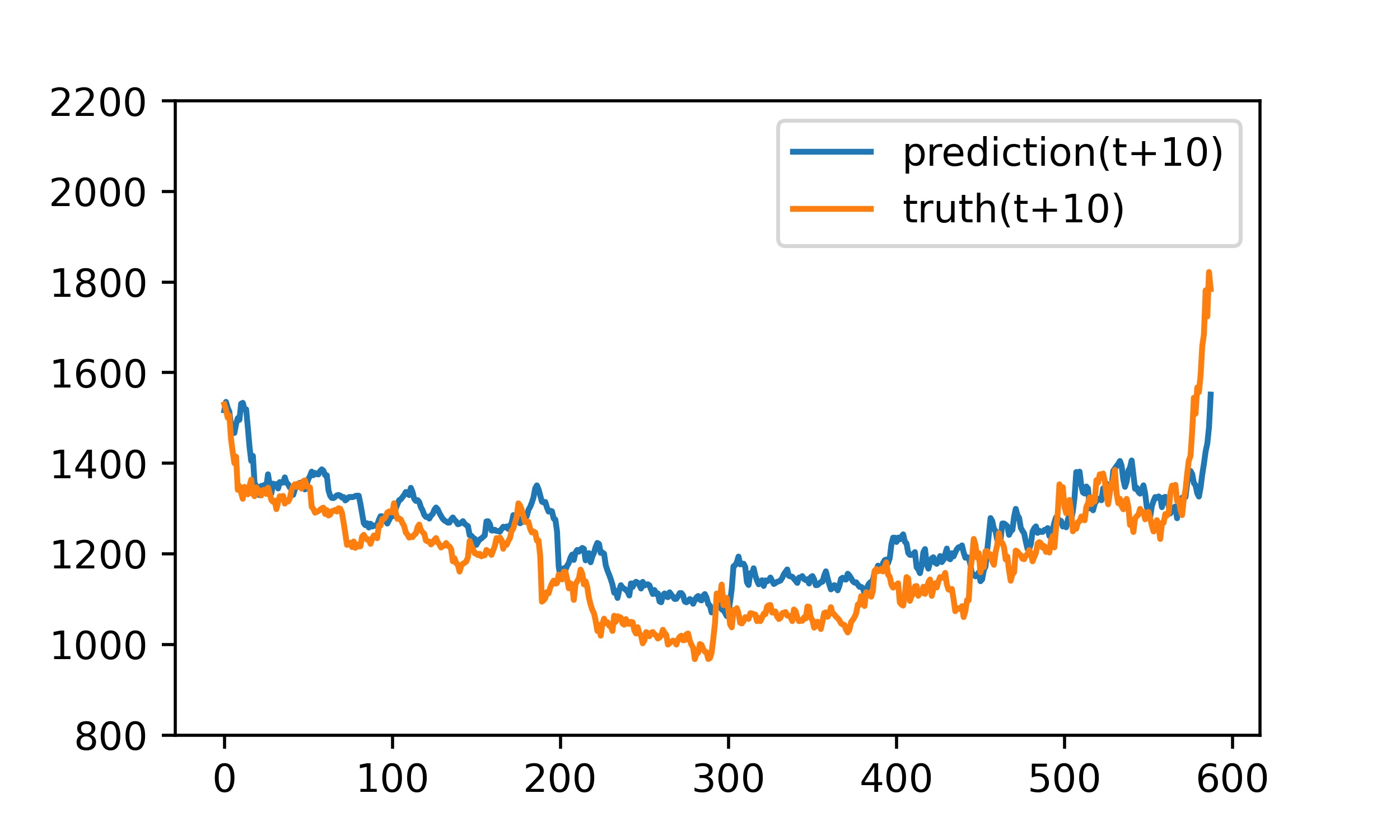}     
    }   
    
    \caption{The prediction effect of SSE Energy Index with different steps. ($\alpha=1.5$)}   
    \label{fig:B5}     
    \end{figure}

  \end{appendices}
\end{document}